\documentclass[11pt]{amsart}
\usepackage{fullpage}
\usepackage[margin=1in]{geometry}
\usepackage{verbatim}
\usepackage{lineno}
\usepackage{amsmath}
\usepackage{bbm}
\usepackage{graphicx}
\usepackage{float}
\usepackage{mathrsfs}
\usepackage[english]{babel}
\usepackage{epstopdf}
\usepackage{amssymb}
\usepackage{listings}
\usepackage{xcolor}
\usepackage{enumerate}
\usepackage{thmtools}
\usepackage{thm-restate}
\usepackage[ruled,vlined]{algorithm2e}
\usepackage{epsfig}
\providecommand{\U}[1]{\protect\rule{.1in}{.1in}}
\usepackage[all]{xy}
\usepackage{caption}
\usepackage{longtable}
\usepackage{mathtools}
\usepackage{subcaption}
\usepackage[pdftex,pdfborder={0 0 0},
colorlinks=true,
linkcolor=blue,
citecolor=red,
pagebackref=false,
]{hyperref}
\usepackage{listings}
\usepackage{xcolor}
\usepackage{enumerate}
\usepackage{thmtools}
\usepackage{thm-restate}
\usepackage{epsfig}
\usepackage[all]{xy}
\usepackage{caption}
\usepackage{longtable}
\usepackage{mathtools}
\usepackage{url}
\usepackage{tcolorbox}
\usepackage{multirow}

\newcommand{\R}{\mathbb{R}}

\def\conv{\textrm{conv}}

\DeclareMathOperator*{\argmin}{argmin}
\DeclareMathOperator*{\argmax}{argmax}

\newtheorem{thm}{Theorem}[section]
\newtheorem{prop}[thm]{Proposition}
\newtheorem{lem}[thm]{Lemma}
\newtheorem{cor}[thm]{Corollary}

\theoremstyle{definition}
\newtheorem{ass}{Assumption}

\newtheorem{exe}[thm]{Example}

\newtheorem{rem}[thm]{Remark}
\numberwithin{equation}{section}

\definecolor{dgreen}{rgb}{0.00,0.49,0.00}
\definecolor{Brown}{rgb}{0.45,0.0,0.05}

\title{Representation and Regression Problems in Neural Networks: Relaxation, Generalization, and Numerics}

\date{\today}

\author[K. Liu and E. Zuazua]{Kang Liu$^1$}
\address{$^1$Institut de Mathematiques de Bourgogne, Universite Bourgogne Europe, CNRS, 21000 Dijon, France.}

\author{Enrique Zuazua$^2$$^3$$^4$}
\address{$^2$Chair for Dynamics, Control, Machine Learning, and Numerics (Alexander von Humboldt Professorship), Department of Mathematics, Friedrich\ -\ Alexander\ -\ Universit\"at Erlangen\ -\ N\"urnberg, 91058 Erlangen, Germany.}
\address{$^3$Departamento de Matem\'aticas, Universidad Aut\'onoma de Madrid, 28049 Madrid, Spain.}
\address{$^4$Chair of Computational Mathematics, Fundaci\'on Deusto, Av.\@ de las Universidades, 24, 48007 Bilbao, Basque Country, Spain.}
\email{kang.liu@u-bourgogne.fr, enrique.zuazua@fau.de}



\thanks{This work was supported by the Alexander von Humboldt-Professorship program, the European Union's Horizon Europe MSCA project ModConFlex, the Transregio 154 Project of the DFG, AFOSR 24IOE027 project, grants PID2020-112617GB-C22, TED2021-131390B-I00 of MINECO and PID2023-146872OB-I00 of MICIU (Spain),  and the Madrid Government - UAM Agreement for the Excellence of the University Research Staff in the context of the V PRICIT (Regional Programme of Research and Technological Innovation).}

\begin{document}

\begin{abstract}
In this work, we address three non-convex optimization problems associated with the training of shallow neural networks (NNs) for exact and approximate representation, as well as for regression tasks. Through a mean-field approach, we convexify these problems and, applying a representer theorem, prove the absence of relaxation gaps. We establish generalization bounds for the resulting NN solutions, assessing their predictive performance on test datasets and, analyzing the impact of key hyperparameters on these bounds, propose optimal choices. 

On the computational side, we examine the discretization of the convexified problems and derive convergence rates. For low-dimensional datasets, these discretized problems are efficiently solvable using the simplex method. For high-dimensional datasets, we propose a sparsification algorithm that, combined with gradient descent for over-parameterized shallow NNs, yields effective solutions to the primal problems.
\end{abstract}

\maketitle

\section{Introduction}

Machine learning has made significant progress and expanded across various fields over the last decade \cite{lecun2015deep}. The shallow neural networks (NNs), as the earliest prototypical NN architecture, have garnered great interest in the mathematical study of NNs (see \cite{cybenko1989approximation,pinkus1999approximation,bach2017breaking,chizat2018global} for instance) due to their simple formulation and nonlinear nature. Shallow NNs take the form:
\begin{equation}\label{eq:shallow}
f_{\text{shallow}} (x, \Theta) = \sum_{j=1}^P \omega_j \sigma(\langle a_j, x \rangle + b_j),
\end{equation}
where $x\in \R^d$ is the input, \( P \in \mathbb{N}_{+} \) is the number of neurons, and \( \Theta = \{(\omega_j, a_j, b_j)\}_{j=1}^P \) represent the parameters, with \(\omega_j \in \mathbb{R}\), \(a_j \in \mathbb{R}^d\), and \(b_j \in \mathbb{R}\). The function \(\sigma\) is a Lipschitz nonlinear activation function.

The supervised learning task is as follows:
Given a set of features \( X = \{x_i \in \mathbb{R}^d\}_{i=1}^N \) and corresponding labels \( Y = \{y_i \in \mathbb{R}\}_{i=1}^N \), the goal is to find \( P\) and \( \Theta \) such that
\begin{equation}\label{pb:ml}
    f_{\text{shallow}}(x_i, \Theta) \approx y_i, \quad \text{for } i = 1, \ldots, N.
\end{equation}
Problem \eqref{pb:ml} is tied to the universal approximation property (UAP) of shallow NNs, a topic extensively explored in the literature. The UAP asserts that functions obeying the ansatz \eqref{eq:shallow}  have the capacity to approximate any continuous function with arbitrary precision (in compact sets) as the number of neurons $P$ approaches infinity (with a suitable choice of $\Theta$), see \cite{cybenko1989approximation,pinkus1999approximation}.

Beyond the UAP, with a finite sample set, we can often achieve an exact mapping from features to their corresponding labels, i.e.,
\begin{equation}\label{eq:representation}
 f_{\text{shallow}}(x_i, \Theta) = y_i, \quad \text{for } i = 1, \ldots, N.
\end{equation}
If such a parameter \( \Theta \) exists for any distinct dataset of size \( N \), we say that this NN architecture has the \( N \)-sample representation property (also referred to as finite-sample expressivity in some references). This property is well-known to hold for non-polynomial activation functions in \cite[Thm.\@ 5.1]{pinkus1999approximation}.

\subsection{Motivation and problems statement} The choice of \(\Theta\) satisfying \eqref{pb:ml} (or \eqref{eq:representation}) is typically non-unique and may even be an infinite set.
However, in practice, most of these parameters $\Theta$ suffer from the so-called ``overfitting" issue,  as discussed in \cite{hoefler2021sparsity} and reference therein. Such $\Theta$ result in poor predictive performance on test datasets distinct from the training one, indicating the inadequate generalization of the trained model.  
A practical approach to mitigate overfitting is to pursue ``sparse" solutions, as suggested in \cite{srivastava2014dropout,mocanu2018scalable}. 
The sparsity of shallow NNs is related to the \(\|\cdot\|_{\ell^0}\)-pseudo-norm of the vector \(\omega = (\omega_1, \ldots, \omega_P)\), since the \(j\)-th neuron is activated if and only if \(\omega_j \neq 0\). However, the nonconvexity of \(\|\omega\|_{\ell^0}\) makes it very difficult to handle in optimization problems. A natural idea from compressed sensing is to replace \(\|\omega\|_{\ell^0}\) with \(\|\omega\|_{\ell^1}\), see \cite[Eq.\@ 1.9 and Eq.\@ 1.12]{candes2006quantitative} for instance. This approach is used in the following three optimization problems addressed in this article. 
Another motivation for the use of \( \|\omega\|_{\ell^1} \) from a generalization perspective of Shallow NNs is discussed in Remark \ref{rem:motivation}.

Fix a training dataset $(X,Y)=\{ (x_i, y_i)\in \R^{d+1} \}_{i=1}^N$.

First, we consider the \textit{sparse exact representation} problem corresponding to \eqref{eq:representation}:
\begin{equation}\label{intro_pb:NN_exact}\tag{\text{P$_0$}}
    \inf_{\{(\omega_j,a_j,b_j)\in \R\times \Omega\}_{j=1}^P } \sum_{j=1}^P |\omega_j|, \qquad \text{s.t. }  \sum_{j=1}^P \omega_j \sigma(\langle a_j, x_i \rangle + b_j ) = y_i  , \quad \text{for }i=1,\ldots,N,
\end{equation}
where $\Omega$ is a compact subset of $\R^{d+1}$ for $(a_j,b_j)$. The necessity of the compactness of \(\Omega\) and its suitable choices are discussed in Remarks \ref{rem:compact} and \ref{rem:omega}.

When the observed labels \(y_i\) are subject to noise, problem \eqref{pb:ml} is more appropriate. Introducing a tolerance parameter \(\epsilon \geq 0\) in \eqref{intro_pb:NN_exact}, this leads to the following \textit{sparse approximate representation} problem:
\begin{equation}\label{intro_pb:NN_epsilon}\tag{\text{P$_\epsilon$}}
    \inf_{\{(\omega_j,a_j,b_j)\in \R\times \Omega\}_{j=1}^P } \sum_{j=1}^P |\omega_j|, \qquad \text{s.t. } \left| \sum_{j=1}^P \omega_j \sigma(\langle a_j, x_i \rangle + b_j ) - y_i \right| \leq \epsilon , \quad \text{for }i=1,\ldots,N.
\end{equation}
 
Another common approach in machine learning to handle noisy data is to consider the following \textit{sparse regression} problem:
\begin{equation}\label{intro_pb:NN_reg}\tag{P$^{\text{reg}}_{\lambda}$}
   \inf_{\{(\omega_j,a_j,b_j)\in \R\times \Omega\}_{j=1}^P } \sum_{j=1}^P |\omega_j| + \frac{\lambda }{N}\sum_{i=1}^N \ell \left( \sum_{j=1}^P \omega_j \sigma(\langle a_j, x_i \rangle + b_j ) - y_i  \right),
\end{equation}
where $\lambda> 0$ and $\ell\colon \R\to \R_{+}$ represents a general fidelity error function. As is common in machine learning, parameters \(P \), \(\epsilon\), and \(\lambda\) in the above problems are referred to as ``hyperparameters". 

These sparse learning problems exhibit non-convexity properties (either in the feasible set or the objective function), making them challenging to handle. This article is devoted to the theoretical and numerical analysis of these sparse learning problems.

\subsection{Methodology and main results}
The main results and techniques employed in this article are presented in the following.

\medskip
\textit{Existence.}
We begin by presenting the \( N \)-sample representation property of \eqref{eq:shallow} in Theorem \ref{thm:NN_exists}, under the condition \( P \geq N \). Subsequently, the existence of solutions of \eqref{intro_pb:NN_exact}, \eqref{intro_pb:NN_epsilon}, and \eqref{intro_pb:NN_reg} follows as a corollary of this property, as stated in Corollary \ref{lm:NN_existence}.
We emphasize that Theorem \ref{thm:NN_exists} generalizes the classical result \cite[Thm.\@ 5.1]{pinkus1999approximation} in two ways: (1) the labels \( y_i \) can lie in a multidimensional space (see Remark \ref{rem:high-d}), and (2) the representation property remains valid even when the bias terms \( b_j \) are bounded. Furthermore, we demonstrate the link between \eqref{intro_pb:NN_exact} and \eqref{intro_pb:NN_epsilon} and \eqref{intro_pb:NN_reg}, by letting the hyperparameters \(\epsilon \to 0\) and \(\lambda \to \infty\), respectively, as discussed in Remarks \ref{rem:exact-regression} and \ref{rem:exact-regression2}.

\medskip
\textit{Relaxation.} The condition \(P \geq N\) renders \eqref{intro_pb:NN_exact}, \eqref{intro_pb:NN_epsilon}, and \eqref{intro_pb:NN_reg} high-dimensional non-convex optimization problems, since \(N\) is assumed to be large and \(\sigma\) is nonlinear. A popular approach to address this nonlinearity in shallow NNs is to consider their ``mean-field" relaxation (or infinitely wide counterpart), as discussed in \cite{mei2018mean,chizat2018global,sirignano2020mean}. This leads to convex relaxations of primal problems, formulated in \eqref{pb:NN_exact_rel}, \eqref{pb:NN_epsilon_rel}, and \eqref{pb:NN_reg_rel}.
To show the general idea of this relaxation, we present \eqref{pb:NN_epsilon_rel}, the mean-field relaxation of \eqref{intro_pb:NN_epsilon}, as an example:
\begin{equation*}
    \inf_{\mu\in \mathcal{M}(\Omega)} \|\mu\|_{\text{TV}}, \qquad \text{s.t. }  \left|\int_{\Omega} \sigma( \langle a,x_i \rangle  + b)  d\mu(a,b) - y_i\right| \leq \epsilon, \quad \text{for }i=1,\ldots,N,
\end{equation*}
where $\mathcal{M}(\Omega)$ is the Radon measure space supported in $\Omega$ and  $\|\cdot\|_{\text{TV}}$ is the total variation norm. It is worth noting that the previous relaxed problem is convex, as the objective function is convex and the constraint functions are linear with respect to $\mu$. By leveraging the ``Representer Theorem" from \cite{fisher1975spline},  in Theorem \ref{thm:NN_exists_0} we show that there is no relaxation gap between sparse learning problems and their relaxations: if $P\geq N$, then
\begin{equation}\label{intro_eq:equivalence}
\text{val}\eqref{intro_pb:NN_exact} = \text{val}\eqref{pb:NN_exact_rel},\quad \text{val}\eqref{intro_pb:NN_epsilon} = \text{val}\eqref{pb:NN_epsilon_rel}, \quad \text{val}\eqref{intro_pb:NN_reg} = \text{val}\eqref{pb:NN_reg_rel}.
\end{equation}
Moreover,  all extreme points of the solution sets of these relaxed problems are empirical measures associated with solutions of their primal problems with $P=N$. The equivalence \eqref{intro_eq:equivalence} implies that the value of the primal problems remains unchanged when \(P\) exceeds \(N\). Consequently, the optimal choice for \(P\) is \(N\), as outlined in Remark \ref{rem:P1}.

\medskip
\textit{Generalization bounds.} 
We examine the generalization properties of the shallow NN \eqref{eq:shallow} with parameters $\Theta$ on some testing dataset $(X', Y') = \{(x_i', y_i') \in \mathbb{R}^{d+1}\}_{i=1}^{N'}$. Specifically, we analyze the discrepancies between the distribution of the actual data and the predicted outputs on the test set. Consider the following empirical distributions:
\begin{equation}\label{eq:distributions}
\begin{split}
    m_{\textnormal{train}} = \frac{1}{N}\sum_{i=1}^N \delta_{(x_i, y_i)},& \quad  m_{\textnormal{test}} = \frac{1}{N'}\sum_{i=1}^{N'} \delta_{(x'_i, y'_i)},  \quad m_{\textnormal{pred}} (\Theta) = \frac{1}{N'}\sum_{i=1}^{N'} \delta_{(x'_i, f_{\textnormal{shallow}}(x_i',\Theta))},\\
  & m_{X} = \frac{1}{N}\sum_{i=1}^N \delta_{x_i}, \quad m_{X'} = \frac{1}{N'} \sum_{i=1}^{N'} \delta_{x_i'}.
\end{split}
\end{equation}
Let $d_{\text{KR}}$ denote the Kantorovich\ -\ Rubinstein distance \cite[Rem.\@ 6.5]{villani2009} for probability measures. In Theorem \ref{thm:generalization}, we prove that 
   \begin{equation}\label{intro_eq:generalization}
        d_{\textnormal{KR}}(m_{\textnormal{test}}, m_{\textnormal{pred}} (\Theta)) \leq \underbrace{d_{\textnormal{KR}}(m_{\textnormal{train}}, m_{\textnormal{test}}) + d_{\textnormal{KR}}(m_X, m_{X'})}_{\textnormal{Irreducible error from datasets}} + r(\Theta),
    \end{equation}
    where $L$ is the Lipschitz constant of $\sigma$, and the residual term $r$ is as follows:
    \begin{equation}\label{intro_eq:r}
        r(\Theta) = \underbrace{\frac{1}{N} \sum_{i=1}^N \left|f_{\textnormal{shallow}} (x_i, \Theta) - y_i\right|}_{\textnormal{Bias from training}} + \underbrace{d_{\textnormal{KR}}(m_X, m_{X'}) L  \sum_{j=1}^P |\omega_j|\|a_j\|}_{\textnormal{Standard deviation}}.
    \end{equation}
The irreducible error in \eqref{intro_eq:generalization} is independent of the parameters \(\Theta\), determined by the training and testing data sets. 
In contrast, the residual term \(r(\Theta)\) consists of two components depending \( \Theta \): (1) the fidelity error on the training set, referred to as the training bias, and (2),  related to the sensitivity of the trained NN, which plays the role of the standard deviation. A similar generalization bound result for mean-$l^p$ error is presented in Remark \ref{rem:mean-lp} for the particular case $N'=N$.

\medskip
\textit{Optimal hyperparameters.}
In Proposition \ref{prop:lambda}, we establish upper bounds on the residual term \(r(\Theta)\) (see \eqref{intro_eq:r}). These upper bounds are determined by $d_{\text{KR}}(X, X')$ and the hyperparameters \(\lambda\) and \(\epsilon\). Minimizing these bounds provides robust and optimal choices for the hyperparameters (see Remark \ref{rem:robust} and Subsection \ref{sec:parameters}). The optimal hyperparameter choices are summarized in Table \ref{table:parameters}.
Specifically, the optimal number of neurons is $N$, and the optimal value for $\lambda$ is given by $d_{\text{KR}}(X,X')^{-1}$. Determining the optimal value of $\epsilon$ is more technical. When $d_{\text{KR}}(X,X')^{-1}$ is less than some threshold $c_0$, the optimal $\epsilon^*$ is $0$, suggesting that \eqref{intro_pb:NN_exact} is preferable to \eqref{intro_pb:NN_epsilon} for any $\epsilon > 0$.
The threshold $c_0$, defined in \eqref{c0}, is the minimum $\ell^1$-norm of the solutions to the dual problem associated with the relaxation of \eqref{intro_pb:NN_exact}. On the other hand, when $d_{\text{KR}}(X,X') > c_0^{-1}$, there is no explicit expression for the optimal $\epsilon^*$. From the first-order optimality condition \eqref{first-order} we can deduce  that $\epsilon^*$ is increasing with respect to $d_{\text{KR}}(X,X')$ and asymptotically approach $\|Y\|_{\ell^{\infty}}$, where $Y$ is the vector of training labels, as illustrated in Figure \ref{fig:optimal_epsilon}. 

\begin{table}[h]
\begin{tabular}{|c|c|c|cc|}
\hline
\multirow{2}{*}{Hyperparameter} & \multirow{2}{*}{$P$} & \multirow{2}{*}{$\lambda$} & \multicolumn{2}{c|}{$\epsilon$}       \\ \cline{4-5} 
        &           &                    & \multicolumn{1}{c|}{ $d_{\text{KR}}(X,X') \leq c_0^{-1}$ } & $d_{\text{KR}}(X,X') >  c_0^{-1} $ \\  \hline
Optimal choice      & 
N                  & $d_{\text{KR}}(X,X')^{-1}$                 & \multicolumn{1}{c|}{0}  & increasing w.r.t.\@ $d_{\text{KR}}(X,X')$  \\ \hline
\end{tabular}
\smallskip
\caption{Optimal hyperparameters derived from Remarks \ref{rem:P1}, \ref{rem:P2}, \ref{rem:lam}, and \ref{rem:eps}.}
\label{table:parameters}
\end{table}

\begin{figure}[htp]
\centering
\includegraphics[width=0.6\linewidth]{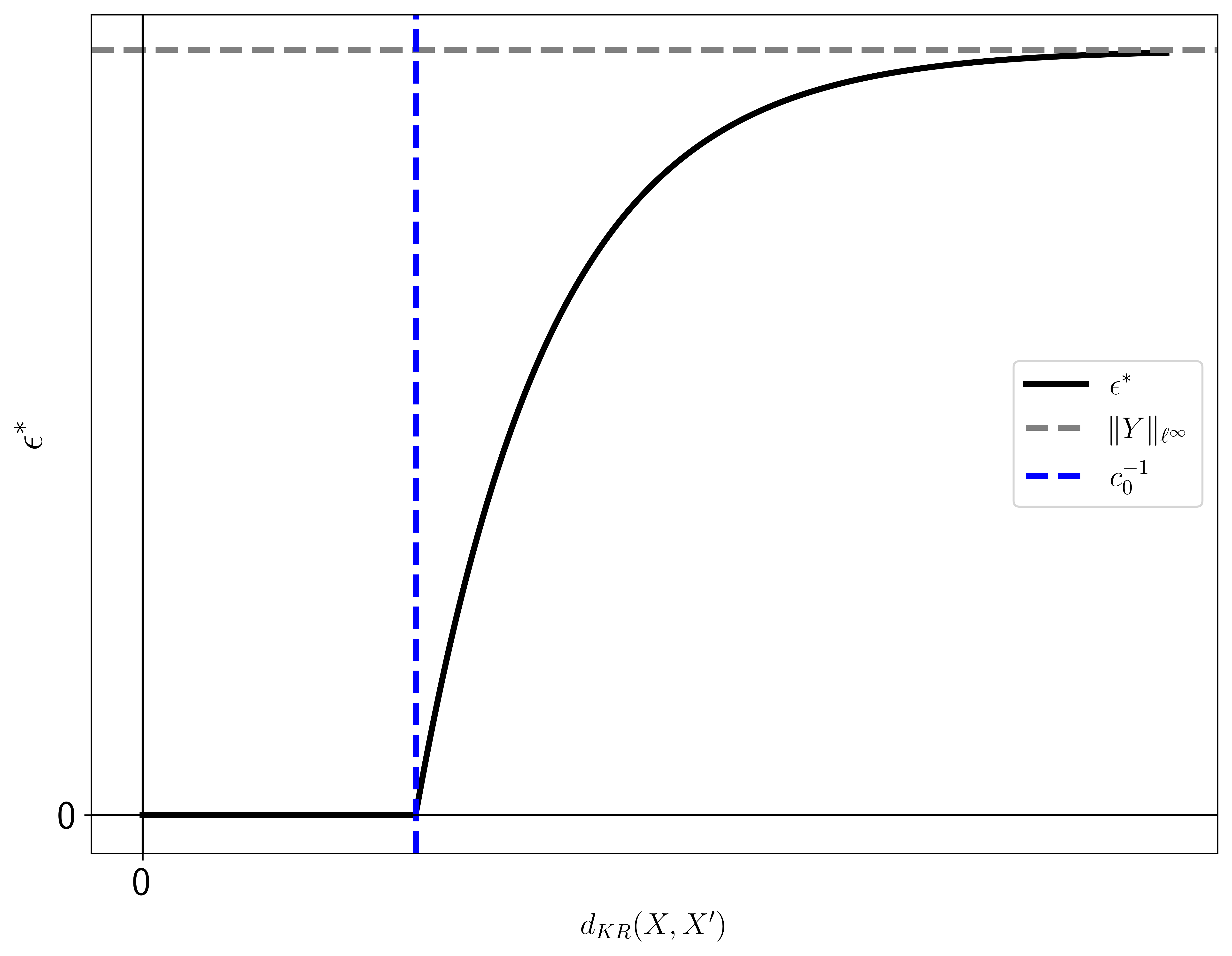}
\caption{Qualitative curve for optimal choice of $\epsilon$ with respect to $d_{KR}(X,X')$.}
\label{fig:optimal_epsilon}
\end{figure}

\textit{Discretization and algorithms.}
Relaxed problems \eqref{pb:NN_epsilon_rel} and \eqref{pb:NN_reg_rel} constitute convex optimization problems in an infinite-dimensional space $\mathcal{M}(\Omega)$. When the dimension of $\Omega$ is small (see Remark \ref{rem:memory}), we discretize $\mathcal{M}(\Omega)$ into $\mathcal{M}(\Omega_{\text{dis}})$, where $\Omega_{\text{dis}}$ represents a set of discrete points within $\Omega$. The related discretized problems are given by \eqref{pb:NN_M_eq} and \eqref{pb:NN_reg_M_eq}, with error estimates established in Theorem \ref{thm:discrete}: 
\begin{equation*}
    |\text{val}\eqref{pb:NN_M_eq}-\text{val}\eqref{pb:NN_epsilon_rel}|,\,  |\text{val}\eqref{pb:NN_reg_M_eq}-\text{val}\eqref{pb:NN_reg_rel}| = \mathcal{O} (d_{\text{Hausdorff}}(\Omega,\Omega_{\text{dis}})),
\end{equation*}
where $d_{\text{Hausdorff}}(\cdot,\cdot)$ represents 
the Hausdorff distance. The discretized problems \eqref{pb:NN_M_eq} and \eqref{pb:NN_reg_M_eq} are equivalent to the linear programming problems \eqref{pb:lp} and \eqref{pb:lp2}, which can be efficiently solved using the simplex algorithm \cite[Sec.\@ 3]{bertsimas1997introduction}. More importantly, the simplex method will return a solution in the set of extreme points of the feasible set. These extreme solutions serve directly as good approximate solutions for the primal problems \eqref{intro_pb:NN_epsilon} and \eqref{intro_pb:NN_reg}, as mentioned in Theorem \ref{thm:simplex}.

The previous discretization method suffers from the curse of dimensionality. Consequently, when $d$ is large, we propose to directly apply the stochastic (sub)gradient descent (SGD) algorithm to address the regression problem \eqref{intro_pb:NN_reg}, utilizing an overparameterized shallow NNs framework with \( P \) significantly exceeding \( N \). The SGD in such overparameterized settings can be seen as a discretization of the gradient flow of relaxed problems, see Remark \ref{rem:gradient flow} for more details.
Its effectiveness has been extensively studied in recent years, see \cite[Sec.\@ 12]{bach2024learning} for instance.
Our contribution is the development of a sparsification method, outlined in Algorithm \ref{alg1}, which filters the solution obtained via SGD into one with fewer than \( N \) activated neurons, see Proposition \ref{prop:sparse}. From the numerical experiments we examine in Section \ref{sec:numerical}, this simplification reduces the storage requirements of the trained model while maintaining prediction accuracy.

\medskip
\textit{$L^2$-regularization and double descent.} In Section \ref{sec:discussion}, we compare $L^2$-regularization (see problem \eqref{pb:L2}) with the total variation regularization used in our study. The solution of problem \eqref{pb:L2} is explicit (see Proposition \ref{prop:L2}) and resides in the reproducing kernel Hilbert space (RKHS) \cite[Def.\@ 2.9]{scholkopf2002learning} associated with the kernel generated by the activation function $\sigma$. Consequently, it cannot be represented in the form of a finite shallow NN as in \eqref{eq:shallow}, which is fundamentally different from the sparse solution provided by Theorem \ref{thm:NN_exists_0}. For further details, we refer to Remark \ref{rem:L1-L2}.

The double descent phenomenon \cite{belkin2019reconciling} is an intriguing behavior observed in neural network training, where the testing error initially decreases with increasing model complexity, then rises, and finally decreases again. In Proposition \ref{prop:double_descent} and Remark \ref{rem:double_descent}, we interpret this phenomenon by using a similar generalization error formula of \eqref{intro_eq:generalization}-\eqref{intro_eq:r} and analyzing the monotonicity of both the training error and total variation with respect to $P$. The numerical results presented in Figure \ref{fig:Double_descent} support our findings.

\subsection{Main contributions}
We provide a rigorous and unified analysis of the mean-field relaxations for approximate (exact) representation and regression problems of shallow NNs, see Theorem \ref{thm:NN_exists_0}. While representer theorems have been applied to exact representation and regression problems in specific scenarios (see \cite{de2020sparsity,bach2017breaking} and \cite{unser2019representer}), we are the first to extend this technique to approximate representation problems, which are of particular interest in studying generalization properties.

The generalization bound established in Theorem \ref{thm:generalization} is novel compared to classical bounds, such as the Rademacher complexity \cite{bach2017breaking} and the Vapnik-Chervonenkis dimension \cite{opper1994learning} in the literature. By minimizing the generalization bounds associated with solutions of \eqref{intro_pb:NN_epsilon} and \eqref{intro_pb:NN_reg}, we obtain optimal choices of hyperparameters in learning problems, as detailed in Table \ref{table:parameters}. This generalization bound also helps us to investigate the double descent phenomena in neural networks.

On the numerical side, for the low-dimensional case, we bridge the discretized problems with the simplex method from operations research, which provides a global minimizer in a finite number of iterations, see Theorem \ref{thm:simplex}. For the high-dimensional case, we develop a sparsification algorithm. When combined with SGD for overparameterized shallow NNs \cite{chizat2018global}, this approach yields a simplified solution without compromising prediction accuracy, as verified by our numerical experiments.

\subsection{Related works} 
\textit{The UAP and the finite-sample representation property of NNs}. 
The first UAP result can be traced back to the Wiener Tauberian Theorem \cite[Thm.\@ II]{wiener1932tauberian} in 1932, which encompasses a large class of activation functions. The UAP for Sigmoidal shallow NNs was demonstrated in the seminal work of Cybenko in 1989 \cite{cybenko1989approximation}. Extensions to multilayer perceptrons were explored in \cite{hornik1991approximation}. Quantitative results based on Barron space are discussed in \cite{barron1993universal} and \cite{ma2022barron}. For comprehensive overviews on this topic, we refer to \cite{pinkus1999approximation} and \cite{devore2021neural}.
The finite-sample representation property (also referred to as finite-sample expressivity in the literature) has recently garnered significant attention. This property has been proven for shallow NNs with ReLU and smooth activation functions in \cite{zhang2021understanding} and \cite{zhang2021expressivity}, and for 2-hidden-layer ReLU networks in \cite{yun2019small}. A general result for non-polynomial activation functions is provided in \cite[Thm.\@ 5.1]{pinkus1999approximation}.  In \cite{ruiz2023neural}, the authors prove this property for Neural ODEs and bridge it to simultaneous controllability in control theory.
This work is extended in \cite{alvarez2024interplay} to Neural ODEs with a width greater than one, examining the relationship between width and the complexity of controls.

\medskip
\textit{Mean-field relaxation}. Training shallow NN presents theoretical challenges due to the non-convexity of their objective functions, a consequence of the non-linearity of $f_{\text{shallow}}$. To address this issue, the authors in \cite{mei2018mean, chizat2018global,sirignano2020mean} consider an infinitely wide shallow NN, also known as the mean-field relaxation of shallow NN. In this mean-field relaxation, the output is defined as the integral of activation functions over a distribution of parameters, rather than a sum of $P$ neurons. This can be expressed mathematically as $\int_{\Omega} \sigma(\langle a, x\rangle +b) d\mu$, which is linear w.r.t.\@ $\mu$.
An upper bound on the relaxation gap is shown to be \(\mathcal{O}(1/P)\) in \cite[Prop.\@ 1]{mei2018mean}. In our setting, using techniques from representer theorems, we show that there is no relaxation gap when $P\geq N$ and the objective function is related to the total variation (see Theorem \ref{thm:NN_exists_0}). 
For a good summary of the mean-field relaxation technique in shallow NNs, we refer to \cite{fernandez2022continuous}.

\medskip
\textit{Representer theorems}. 
Representer theorems characterize the structure of solutions to problems in a form similar to the relaxed problem \eqref{pb:NN_exact_rel}, where the constraint is replaced by general linear equations on $\mu$. The first result on representer theorems can be found in \cite{zuhovickit1962approximation}, where the author demonstrates that there exists a solution as the sum of $m$ Dirac measures, with $m$ being the dimension of the linear constraint. A more precise analysis of the extreme points of the solution set is found in \cite{fisher1975spline}, which we refer to in Theorem \ref{thm:representation}.
The dual version of representer theorems are investigated in \cite{duval2015exact}. 
 Recent works in \cite{boyer2019representer,bredies2020sparsity} develop results of \cite{fisher1975spline} for unbounded domains
$\Omega$.
We refer to \cite{unser2017splines} and \cite{duval2022faces} for a rich source of representer theorems in general settings. In the context of deep learning, \cite{unser2019representer} derives a representer theorem for deep NNs by adding a second-order total variation regularization to the fidelity error. It is worth mentioning that a key assumption for achieving representer theorems is that the feasible set is non-empty (e.g.\@ \cite[Thm.\@ 1]{fisher1975spline}, \cite[Thm.\@ 3.1]{boyer2019representer} and \cite[Ass.\@ H0]{bredies2020sparsity}). This assumption is easy to verify for regression problems (e.g.\@ \cite[Eq.\@ 15]{unser2019representer} and \eqref{intro_pb:NN_reg}), but is typically less obvious for constrained optimization problems like \eqref{intro_pb:NN_exact} and \eqref{intro_pb:NN_epsilon}.
In our context, this non-emptiness assumption is guaranteed by the $N$-sample representation property that we establish in Theorem \ref{thm:NN_exists}. We mention that the value of the relaxed exact representation problem \eqref{pb:NN_exact_rel} is referred to as the variation norm of shallow NNs \cite{bach2017breaking}. Similar works on the representer theorem for \eqref{pb:NN_exact_rel} are conducted in \cite{de2020sparsity,bach2017breaking} with the ReLU activation function. In comparison to these works, our work offers greater flexibility in the choice of activation functions and addresses the gap in the study of approximate representation problems \eqref{intro_pb:NN_epsilon}.

\medskip
\textit{Algorithms}. In this article, we employ the simplex method to solve discretized problems \eqref{pb:NN_M_eq} and \eqref{pb:NN_reg_M_eq} when $d$ is small. It is worth noting that when \(\epsilon = 0\), problem \eqref{pb:NN_M_eq} reduces to the classical basis pursuit problem, which attracted significant attention in the field of compressed sensing in the first decade of this century (see \cite{candes2006quantitative,donoho2006compressed} for instance). Numerous algorithms can be applied to solve the basis pursuit, including the Bregman algorithm \cite{osher2005iterative} and its variants \cite{yin2008bregman,goldstein2009split}, the ADMM algorithm \cite{glowinski1975approximation}, and the primal-dual algorithm \cite{chambolle2011first}, among others. We choose the simplex method because it can provide an exact solution located at the extreme points, directly giving a solution of the primal problems (see Theorem \ref{thm:simplex}). 

When the data dimensionality $d$ is high, we combine the SGD method for overparameterized models with our sparsification technique (Algorithm \ref{alg1}) to obtain solutions with at most \(N\) activated neurons. The convergence properties of SGD for the training of overparameterized NNs have been extensively studied in the literature, including \cite{chizat2018global,allen2019convergence,mei2018mean, bach2024learning}, etc. Additionally, the convergence of the Wasserstein gradient flow, which serves as a continuous counterpart to SGD, has been explored in \cite{chizat2018global} and \cite{sirignano2020mean} in the context of the mean-field relaxation mentioned earlier.

\subsection{Organization} 
In Section \ref{sec:notation}, we introduce some notations, definitions, and the representer theorem used in this article. Section \ref{sec:shallow} focuses on the introduction of relaxed problems and the proof of zero relaxation gaps. In Section \ref{sec:generalization}, we study the generalization property of the solutions of the primal problems.
The discretization and resolution of relaxed problems are presented in Section \ref{sec:algo}, along with a sparsification algorithm. Section \ref{sec:numerical} provides a numerical experiment to illustrate the performance of our numerical schemes and the relationship between hyperparameters and generalization performance.
Some technical proofs are moved to Section \ref{sec:proof}. The last Section \ref{sec:conclusion} summarizes our main findings and some relevant open problems and perspectives for future research.
Appendix \ref{app:dual} contains the dual analysis of the approximate representation problem, applied along the article.

\subsection{Notations}\label{sec:notation}
Let $q$ be a positive integer. For any $x\in \R^q$, we denote by $\|x\|_{\ell^p}$ the $\ell^p$-norm of $x$ ($p\geq 1$ or $p=\infty$), by $\|x\|$ the Euclidean norm ($\ell^2$-norm) of $x$, and by $\|x\|_{\ell^0}$ the number of non-zero elements of $x$. Let $A$ be an $ m \times n $ matrix, where $ m $ and $ n $ are positive integers. We denote by $\|A\|$ the operator norm of $A$, defined as $\|A\| = \sup_{\|u\| \leq 1} \|A u\|$. For any $I\subseteq \{1,2,\ldots, n\}$, the notation $A_{I}$ is the sub-matrix of $A$ consisting of the columns indexed by the elements in $I$. 

Let $(\mathcal{P})$ denote an abstract optimization problem (minimization or maximization) with an objective function $J$. The value (resp.\@ the set of solutions) of the $(\mathcal{P})$ is denoted by $\text{val}(\mathcal{P})$ (resp.\@ $S(\mathcal{P})$)  for convenience.  We say that a point $x_{\epsilon}$ (depending on $\epsilon > 0$) is an $\mathcal{O}(\epsilon)$-optimal solution of $(\mathcal{P})$ if it is feasible and 
$|J(x_{\epsilon}) - \text{val}(\mathcal{P})|\leq C \epsilon$, for 
some constant $C$ independent of $\epsilon$.

Let $\mathcal{B}$ be a compact subset of $\mathbb{R}^q$. Denote by $C(\mathcal{B})$ the Banach space of all real-valued continuous functions on $\mathcal{B}$, equipped with the supremum norm. Let $\mathcal{M}(\mathcal{B})$ be the dual space of $C(\mathcal{B})$, which is complete by equipping the total variation norm:  for all $ \mu \in \mathcal{M}(\mathcal{B})$,
\begin{equation*}
    \| \mu \|_{\text{TV}} \coloneqq \sup_{ \|v\|_{\mathcal{C}(\mathcal{B})} \leq 1} \int_{\mathcal{B}} v(\theta) \, d\mu(\theta).
\end{equation*}
Recall the weak-$*$ convergence in $\mathcal{M}(\mathcal{B})$: $\mu_n \overset{\ast}{\rightharpoonup} \mu$ if and only if for any $v\in \mathcal{C}(\mathcal{B})$, we have that $\lim_{n\to \infty}\int_{\mathcal{B}} v \, d\mu_n = \int_{\mathcal{B}} v \, d\mu  $. For any $N\in \mathbb{N}_{+} $, we denote by $\mathcal{M}_N(\mathcal{B})$ the subset of all elements in $\mathcal{M}(\mathcal{B})$ with their support size less than $N$, i.e.,
\begin{equation*}
\mathcal{M}_N(\mathcal{B})\coloneqq \left\{ \mu= \sum_{j=1}^N \lambda_j \delta_{\theta_j}  \, \Big| \, \lambda_j \in \R, \, \theta_j\in \mathcal{B}, \, \text{for } j=1,\ldots, N   \right\},
\end{equation*}
where $\delta_{\theta}$ is the Dirac measure at point $\theta\in \mathcal{B}$.

Here, we present a representer theorem from \cite[Thm.\@ 1]{fisher1975spline}. Before that, we recall the definition of extreme points of a nonempty and convex set $\mathcal{K}$ in a vector space: 
An element $x \in \mathcal{K}$ is called an extreme point of $\mathcal{K}$ if for any $\lambda \in [0, 1]$ and any $x_1, x_2 \in \mathcal{K}$ such that $x = \lambda x_1 + (1 - \lambda) x_2$, it holds that $x_1 = x_2 = x$. We denote the set of all extreme points of $\mathcal{K}$ by $\text{Ext}(\mathcal{K})$.

\begin{thm}[Fisher-Jerome 75]\label{thm:representation}
 Assume that $\mathcal{B}\subset \mathbb{R}^q$ is compact and $L_i \colon \mathcal{B} \to \mathbb{R}$ is continuous for $i=1,\ldots,N$. Consider the following optimization problem:
\begin{equation}\label{pb:total_variation}
    \inf_{\mu\in \mathcal{M}(\mathcal{B})} \|\mu\|_{\textnormal{TV}}, \qquad \text{s.t. } \int_{\mathcal{B}} L_i(\theta) d\mu(\theta) \in I_i, \quad \text{for }i=1,\ldots,N,
\end{equation}
where $I_i$ is a compact interval or a singleton in $\R$ for $i=1,\ldots,N$.
Assume that the feasible set of problem \eqref{pb:total_variation} is non-empty. Then, its solution set $S\eqref{pb:total_variation}$  is non-empty, convex, and compact in the weak-$*$ sense. Moreover,
      \( \textnormal{Ext}(S\eqref{pb:total_variation}) \subseteq \mathcal{M}_{N}(\mathcal{B}) \).
   
\end{thm}

\section{Relaxation of sparse learning problems}\label{sec:shallow}
This section investigates the relaxation of sparse learning problems related to shallow NNs \eqref{eq:shallow}, i.e., \eqref{intro_pb:NN_exact}, \eqref{intro_pb:NN_epsilon}, and \eqref{intro_pb:NN_reg}. 
We make the following assumption on the activation function $\sigma$ and the parameter domain $\Omega$, which is crucial in the existence of solutions of primal problems. 
 \begin{ass}\label{ass:sigma}
 The following holds:
 \begin{enumerate}
     \item The activation function $\sigma\colon \R \to \R$ is continuous, where $\sigma(x)=0$ for $x\leq 0$ and $\sigma(x)>0$ for $x>0$;
     \item The set $\Omega$ is compact and contains a ball centered at 0.
 \end{enumerate}
 \end{ass}

\begin{exe}\label{rem:relu}
    An example of a pair \((\sigma, \Omega)\) satisfying Assumption \ref{ass:sigma} is when \(\sigma\) is the ReLU function, defined as \(\sigma(x) = (x)_{+} \coloneqq \max\{x,0\}\), and \(\Omega\) is the unit ball in \(\mathbb{R}^{d+1}\). This example plays a crucial role in the study of the generalization properties of shallow NNs, as discussed in Section \ref{sec:generalization_relu}.
\end{exe}

\begin{rem}\label{rem:compact}
We emphasize that \(\Omega\) serves as an a priori bound for \((a_j, b_j)\), which is assumed to be compact. This compactness is necessary for the existence of solutions of the relaxed problems introduced later, which are defined on \(\mathcal{M}(\Omega)\). Without this assumption, the set \(\mathcal{M}(\Omega)\) is not weak-\(*\) compact, preventing the application of the representer theorem (Theorem \ref{thm:representation}). The condition that $\Omega$ contains a ball centered at 0 is required for applying the Hahn-Banach theorem in the proof of the finite-sample representation property. \end{rem}

\begin{rem}\label{rem:omega}
In practice, the choice of the compact set \(\Omega\) should be determined on a case-by-case basis. For instance:
\begin{itemize}
    \item It might depend on the available memory range needed to store parameters;
    \item It could also be the unit ball for shallow ReLU networks, as scaling \(\Omega\) does not affect the generalization bound (introduced in Theorem \ref{thm:generalization}) for solutions to the primal problems, as shown in Lemma \ref{lm:relu}. This invariance property is derived from the 1-homogeneity of the ReLU function;
    \item It could be the smallest hypercube  containing all the pre-trained parameters (for example, those obtained via SGD on a regression problem). By solving our primal problems with this choice of \(\Omega\),  new parameters  outperforming the pre-trained ones may be found.
\end{itemize}
\end{rem}

\subsection{Finite-sample representation property and existence results}
Let us first present the finite sample representation property of shallow NNs in \eqref{eq:shallow} under Assumption \ref{ass:sigma}.  \begin{thm}\label{thm:NN_exists}
    Let Assumption \ref{ass:sigma} hold true. Let $N$ be the number of data samples. If $P \geq N$, then for any dataset $\{(x_i, y_i) \in \mathbb{R}^{d+1}\}_{i=1}^N$, where $x_i \neq x_j$ for $i \neq j$, there exists parameters $\Theta = \{(\omega_j, a_j, b_j) \in \mathbb{R} \times \Omega\}_{j=1}^P$ such that
\[
    f_{\textnormal{shallow}}(x_i, \Theta) = y_i \quad \textnormal{for } i = 1, \ldots, N.
\]
\end{thm}
\begin{proof}
   The proof is presented in Section \ref{sec:proof}.
\end{proof}
\begin{rem}\label{rem:high-d}
    The proof of Theorem \ref{thm:NN_exists} is independent of the dimension of the label variable. Therefore, the conclusion holds true for $y_i\in \R^m$ with $m\geq 2$, where \(\omega_j\) should be considered in \(\mathbb{R}^m\).
\end{rem}

In the remainder of this article, we assume that the training dataset $\{(X, Y)\} = \{(x_i, y_i) \in \mathbb{R}^{d+1}\}_{i=1}^N$ has distinct features, i.e., $x_i \neq x_j$ for $i \neq j$. Then, by Theorem \ref{thm:NN_exists}, we deduce the following corollary on the existence of solutions to the primal problems.

\begin{cor}\label{lm:NN_existence}
    Let Assumption \ref{ass:sigma} hold true. Then, the following holds:  \begin{enumerate}
        \item If $P\geq N$,  problems \eqref{intro_pb:NN_exact} and \eqref{intro_pb:NN_epsilon} have solutions.
        \item If $\ell$ is proper and l.s.c.,  problem \eqref{intro_pb:NN_reg} has solutions.
    \end{enumerate}
\end{cor}

\begin{proof}
   In point (1), since $P\geq N$, by Theorem \ref{thm:NN_exists}, the feasible sets of problems \eqref{intro_pb:NN_exact} and \eqref{intro_pb:NN_epsilon} are non-empty. Then, the existence of solutions of \eqref{intro_pb:NN_exact} and \eqref{intro_pb:NN_epsilon} is deduced from the coercivity of their objective functions w.r.t.\@ $(\omega_j)_{j=1}^N$ and the compactness of $\Omega$. Since $\ell$ is proper and l.s.c., we deduce the same properties for the objective function of \eqref{intro_pb:NN_reg}, which, combined with the coercivity on $(\omega_j)_{j=1}^N$ and the compactness of $\Omega$, leads to (2).
\end{proof}

The correspondence between problems \eqref{intro_pb:NN_exact}, \eqref{intro_pb:NN_epsilon}, and \eqref{intro_pb:NN_reg} can be realized by taking extreme values of the hyperparameters \(\epsilon\) and \(\lambda\). As it can be easily seen, as \(\epsilon \to 0\), the approximate representation problem \eqref{intro_pb:NN_epsilon} converges to the exact one \eqref{intro_pb:NN_exact}. In the following remarks, we explain the convergences of the regression problem \eqref{intro_pb:NN_reg} to \eqref{intro_pb:NN_exact} as \(\lambda \to \infty\)  in different scenarios. Here, as it is commonly done in regression problems, we assume that \(\ell(0) = 0\) and that \(0\) is the unique minimizer of \(\ell\). 
 
\begin{rem}\label{rem:exact-regression}
Let $P\geq N$ and let \(\Theta_{\lambda}^{\text{reg}}\) be a solution of \eqref{intro_pb:NN_reg}. Then the sequence \(\{\Theta_{\lambda}^{\text{reg}} \}_{\lambda\in \mathbb{N}_+}\) is precompact, and any cluster point of \(\{\Theta_{\lambda}^{\text{reg}} \}\) as $\lambda \to \infty$ is a solution of \eqref{intro_pb:NN_exact}. As a consequence, \(\text{val}\eqref{intro_pb:NN_exact} = \lim_{\lambda\to +\infty} \text{val}\eqref{intro_pb:NN_reg}\).
\end{rem}

The previous remark is based on the existence of solutions of \eqref{intro_pb:NN_exact}, which is guaranteed by the setting $P\geq N$.  Theorem \ref{thm:NN_exists} does not apply when \(P < N\). As a consequence, the value of the general regression problem \eqref{intro_pb:NN_reg} may approach infinity as \(\lambda \to \infty\).
However, by Corollary \ref{lm:NN_existence}(2), the solution set of the regression problem \eqref{intro_pb:NN_reg} is non-empty even when \(P < N\). The convergence of the solutions of \eqref{intro_pb:NN_reg} when $P<N$ is discussed in the following remark.

\begin{rem}\label{rem:exact-regression2}
Let $P<N$ and let \(\Theta_{\lambda}^{\text{reg}}\) be a solution of \eqref{intro_pb:NN_reg}. 
Assume that \(\ell\) is coercive. Then the sequence \(\{\Theta_{\lambda}^{\text{reg}}\}_{\lambda \in \mathbb{N}_+}\) is precompact, and any cluster point of \(\{\Theta_{\lambda}^{\text{reg}}\}\) as $\lambda \to \infty$ is a solution of the following bi-level optimization problem:
\begin{equation}\label{pb:bilevel}
      \inf_{(\omega_j,a_j,b_j)_{j=1}^P } \sum_{j=1}^P |\omega_j|, \quad \text{s.t. }  (\omega_j,a_j,b_j)_{j=1}^P\in \argmin_{(\R\times\Omega)^P} \sum_{i=1}^N \ell \left( \sum_{j=1}^P \omega_j \sigma(\langle a_j, x_i \rangle + b_j ) - y_i  \right).
\end{equation}
\end{rem}

\subsection{Relaxed problems} 
Problems \eqref{intro_pb:NN_exact}, \eqref{intro_pb:NN_epsilon}, and \eqref{intro_pb:NN_reg} exhibit non-convexity properties, since the feasible sets of \eqref{intro_pb:NN_exact} and \eqref{intro_pb:NN_epsilon} and the fidelity error in \eqref{intro_pb:NN_reg} are non-convex. This is due to the  nonlinearity of $f_{\text{shallow}}(x,\Theta)$ on $\Theta$. To address this non-linearity, we replace $f_{\text{shallow}}(x,\Theta)$ by $\int_{\Omega} \sigma(\langle a, x\rangle +b) d\mu$ with $\mu \in\mathcal{M}(\Omega)$, which is the so-called mean-field relaxation. For convenience, define the following linear mapping:
\begin{equation}\label{eq:phi}
    \phi \colon \mathcal{M}(\Omega) \to \R^N,\, \mu \mapsto ( \phi_i\,\mu )_{i=1}^N, \text{ where }\phi_i\, \mu \coloneqq \int_{\Omega} \sigma( \langle a,x_i \rangle  + b)  d\mu(a,b), \text{ for }i=1,\ldots, N.
\end{equation}
Recall that $Y=(y_1,\ldots,y_N) \in \R^N$.  The relaxation of \eqref{intro_pb:NN_exact} writes:
\begin{equation}\label{pb:NN_exact_rel}\tag{PR$_0$}
    \inf_{\mu\in \mathcal{M}(\Omega)} \|\mu\|_{\text{TV}}, \quad \text{s.t. }  \phi \, \mu = Y.
\end{equation}
 The relaxation of  \eqref{intro_pb:NN_epsilon} writes:
\begin{equation}\label{pb:NN_epsilon_rel}\tag{\text{PR$_\epsilon$}}
     \inf_{\mu\in \mathcal{M}(\Omega)} \|\mu\|_{\text{TV}}, \quad \text{s.t. }  \|\phi \, \mu - Y\|_{\ell^{\infty}} \leq \epsilon.
\end{equation}
The relaxation of the regression problem \eqref{intro_pb:NN_reg} writes:
\begin{equation}\label{pb:NN_reg_rel}\tag{PR$^{\text{reg}}_{\lambda}$}
  \inf_{\mu\in \mathcal{M}(\Omega)} \|\mu\|_{\text{TV}} + \frac{\lambda }{N}\sum_{i=1}^N \ell \left( \phi_i\, \mu - y_i \right).
\end{equation}
It is easy to check that the previous three relaxed problems are convex. In Theorem \ref{thm:NN_exists_0}, we demonstrate that there is no gap between the primal problems and their associated relaxations in the case $P\geq N$, and we characterize the relationship between their solution sets. 

\begin{thm}\label{thm:NN_exists_0}
     Let Assumption \ref{ass:sigma} hold true. Let $P\geq N$.
     Assume that $\ell$ in problem \eqref{intro_pb:NN_reg} is proper, l.s.c., and convex. Then,
     \begin{equation}\label{eq:value_eq}
         \textnormal{val} \eqref{pb:NN_exact_rel} =  \textnormal{val}\eqref{intro_pb:NN_exact}, \quad \textnormal{val} \eqref{pb:NN_epsilon_rel} =  \textnormal{val}\eqref{intro_pb:NN_epsilon}, \quad \textnormal{val} \eqref{pb:NN_reg_rel} =  \textnormal{val}\eqref{intro_pb:NN_reg}.
     \end{equation}
      Moreover, the solution sets $S\eqref{pb:NN_exact_rel}$, $S\eqref{pb:NN_epsilon_rel}$, and $S\eqref{pb:NN_reg_rel}$ are non-empty, convex, and compact in the weak-$*$ sense. Their extreme points are of the form:
\begin{equation}\label{eq:solution_eq}
    \sum_{j=1}^N \omega_j \delta_{(a_j,b_j)},
\end{equation}
    where $(\omega_j, a_j, b_j)_{j=1}^N$ is a solution of the primal problem \eqref{intro_pb:NN_exact}, \eqref{intro_pb:NN_epsilon}, and \eqref{intro_pb:NN_reg}, respectively, for the case $P = N$.
\end{thm}
Before presenting the proof of Theorem \ref{thm:NN_exists_0}, we provide the following remark on the equivalence of the primal problems \eqref{intro_pb:NN_exact}, \eqref{intro_pb:NN_epsilon}, and \eqref{intro_pb:NN_reg}, for different values of \( P \geq N \), respectively. Based on this remark, we conclude that the optimal number of neurons is \( P = N \).

\begin{rem}[On the optimal choice  $P=N$]\label{rem:P1}
   As a corollary of Theorem \ref{thm:NN_exists_0} the value of the primal problem \eqref{intro_pb:NN_exact} becomes independent of $P$ when $P \geq N$. Specifically, 
\begin{equation*}
    V_N = V_{P}, \, \forall P \ge N.
\end{equation*}
In other words, increasing $P$ beyond $N$ does not yield a better result in terms of the value of the primal optimization problems. At the same time,  the hyperparameter choice  $P=N$ minimizes the complexity of the neural network ansatz.

In addition to the fact that values coincide, we have the following result concerning its solutions: If \( (\omega^*_j, a^*_j, b^*_j)_{j=1}^N \) is a solution of \eqref{intro_pb:NN_exact} for \( P = N \), we obtain a solution to \eqref{intro_pb:NN_exact} for $P>N$  by appending zeros for $j =N+1,...,P$.  Conversely, given any solution \( (\omega^*_j, a^*_j, b^*_j)_{j=1}^P \) of \eqref{intro_pb:NN_exact} for \( P > N \), we can filter it to a solution of \eqref{intro_pb:NN_exact} for $P=N$ by Algorithm \ref{alg1}, as stated in Proposition \ref{prop:sparse}.

The same conclusions apply to problems \eqref{intro_pb:NN_epsilon} and \eqref{intro_pb:NN_reg}.
\end{rem}

\color{black}

\begin{proof}[Proof of Theorem \ref{thm:NN_exists_0}]
First,  from the definition of relaxed problems it follows that
 \begin{equation}\label{eq:proof_0}
         \textnormal{val} \eqref{pb:NN_exact_rel} \leq   \textnormal{val}\eqref{intro_pb:NN_exact}, \quad \textnormal{val} \eqref{pb:NN_epsilon_rel} \leq   \textnormal{val}\eqref{intro_pb:NN_epsilon}, \quad \textnormal{val} \eqref{pb:NN_reg_rel} \leq  \textnormal{val}\eqref{intro_pb:NN_reg}.
     \end{equation}
We then divide the remainder of the proof into the following two steps.

\smallskip
\noindent
\textbf{Step 1} (Cases of \eqref{intro_pb:NN_exact} and \eqref{intro_pb:NN_epsilon}).
By Theorem \ref{thm:NN_exists}, the feasible sets of the relaxed problems \eqref{pb:NN_exact_rel} and \eqref{pb:NN_epsilon_rel} are non-empty. We deduce that $S\eqref{pb:NN_exact_rel}$ and $S\eqref{pb:NN_epsilon_rel}$ are non-empty, convex and compact in the weak-$*$ sense from Theorem \ref{thm:representation}. This implies that $\emptyset\neq  \textnormal{Ext} (S\eqref{pb:NN_exact_rel} )\subseteq S\eqref{pb:NN_exact_rel}$ and $ \emptyset\neq \textnormal{Ext} (S\eqref{pb:NN_epsilon_rel})\subseteq S\eqref{pb:NN_epsilon_rel}$ (see \cite[Thm.\@ 3.24, 3.25]{rudin91functional}). 
The second part of Theorem \ref{thm:representation} states that $ \textnormal{Ext} (S\eqref{pb:NN_exact_rel} )\subseteq \mathcal{M}_N(\Omega)$ and $ \textnormal{Ext} (S\eqref{pb:NN_epsilon_rel})\subseteq \mathcal{M}_N(\Omega)$. As a consequence, $ S\eqref{pb:NN_exact_rel} \cap  \mathcal{M}_N(\Omega) \neq \emptyset$ and $S\eqref{pb:NN_epsilon_rel} \cap  \mathcal{M}_N(\Omega) \neq \emptyset$.

Let $\mu^{*} = \sum_{j=1}^N \omega_j^* \delta_{(a_j^*,b_j^*)} \in S\eqref{pb:NN_exact_rel} \cap  \mathcal{M}_N(\Omega)$.
It follows that
\begin{equation}\label{eq:proof_1}
    \text{val}\eqref{pb:NN_exact_rel} = \sum_{j=1}^N |\omega_j^*| ,\quad \text{and } \sum_{j=1}^N \omega^*_j \sigma(\langle a_j^*, x_i\rangle + b_j^*) = y_i ,\text{ for } i=1,\ldots,N.
\end{equation}
We deduce that \((\omega_j^*, a_j^*, b_j^*)_{j=1}^P\), with \((\omega_j^*, a_j^*, b_j^*) = (0, 0, 0)\) for \(j > N\), belongs to the admissible set of \eqref{intro_pb:NN_exact}. Then, we obtain from \eqref{eq:proof_1} that
\begin{equation}\label{eq:proof_2}
    \text{val}\eqref{pb:NN_exact_rel} \geq \text{val}\eqref{intro_pb:NN_exact}.
\end{equation}
By \eqref{eq:proof_0} and \eqref{eq:proof_2}, we have  \(\text{val}\eqref{pb:NN_exact_rel} = \text{val}\eqref{intro_pb:NN_exact}\). In the specific case where \(P = N\), it follows that \((\omega_j^*, a_j^*, b_j^*)_{j=1}^N\) is a solution of \eqref{intro_pb:NN_exact}, which establishes \eqref{eq:solution_eq} for \eqref{intro_pb:NN_exact}. The case of \eqref{intro_pb:NN_epsilon} is deduced similarly.

\medskip
\noindent\textbf{Step 2} (Case of \eqref{intro_pb:NN_reg}). 
 Let us first prove that $S\eqref{pb:NN_reg_rel}$ is non-empty. Since the objective function of \eqref{pb:NN_epsilon_rel} is proper, we assume that 
    $\{\mu_n\}_{n\geq 1}$ is a minimizing sequence of \eqref{pb:NN_reg_rel}.
    It follows that the sequence $\{\|\mu_n\|_{\text{TV}}\}_{n\geq 1}$ is bounded. Therefore, the set $\{\mu_n \}_{n \geq 1}$ is pre-compact in the weak-$*$ sense by the Banach\ -Alaoglu\ -\ Bourbaki theorem (see \cite[Thm.\@ 3.16]{brezis2011functional}). Without loss of generality, we assume that $\mu_n \overset{\ast}{\rightharpoonup} \tilde{\mu}$ for some $\tilde{\mu}\in \mathcal{M}(\Omega)$.
    Since $\sigma(\langle a, x_i\rangle + b)$ is continuous on $(a,b)$ for all $i$, we have that $  \lim_{n\to \infty} \phi_i\, \mu_n = \phi_i\, \tilde{\mu}$ for all $i$. 
    This, combined with the fact that $\ell$ is l.s.c., leads to
   \begin{equation*}
       \frac{1}{N}\sum_{i=1}^N \ell \left( \phi_i\, \tilde{\mu} -y_i \right)
        \leq \liminf_{n\to \infty}
        \frac{1}{N}\sum_{i=1}^N \ell \left( \phi_i\, \mu_n - y_i \right).
    \end{equation*}
    Moreover,
    $\|\cdot\|_{\text{TV}}$ is l.s.c. in the weak-$*$ topology (see \cite[Prop.\@ 3.13 (iii)]{brezis2011functional}). This implies that $
       \|\tilde{\mu}\|_{\text{TV}} \leq \liminf_{n\to \infty} \|\mu_n\|_{\text{TV}}
    $. Therefore, $\tilde{\mu}$ is a solution of problem \eqref{pb:NN_reg_rel}. 
    Combining with the convexity of $\ell$, we deduce that $S\eqref{pb:NN_reg_rel}$ is non-empty, convex, and compact in the weak-$*$ sense.

    Let $\mu^{\text{e}}\in \textnormal{Ext} (S\eqref{pb:NN_reg_rel} )$ and $Y^{\text{e}} = \phi\, \mu^{\text{e}}$. Assume that $\mu^{\text{e}} \notin \mathcal{M}_N(\Omega)$. Consider the following problem in the form of \eqref{pb:NN_exact_rel}:
    \begin{equation}\label{pb:proof}
    \inf_{\mu\in \mathcal{M}(\Omega)} \|\mu\|_{\text{TV}}, \qquad \text{s.t. }  \phi\, \mu = Y^{\text{e}}.
\end{equation}
Then, we deduce that $\mu^{\text{e}}$ is a solution of \eqref{pb:proof}, and any solution of \eqref{pb:proof} is also a solution of \eqref{pb:NN_reg_rel}.
Since $P\geq N$, by \eqref{eq:solution_eq} for \eqref{intro_pb:NN_exact} (already proved in step 1), $ \mu^{\text{e}}$ can be written as a convex combination of extreme points of $S\eqref{pb:proof}$ and these points belong to $\mathcal{M}_N(\Omega)$. Since $ \mu^{\text{e}} \notin \mathcal{M}_N(\Omega)$, this convex combination is strict. Combining with the fact that these extreme points are also solutions of \eqref{pb:NN_reg_rel}, we obtain a contradiction with the fact that $\mu^{\text{e}}\in \textnormal{Ext} (S\eqref{pb:NN_reg_rel} )$. Therefore, $\mu^{\text{e}}\in \mathcal{M}_N(\Omega)$. Let $ \mu^e = \sum_{\omega_j^*} \delta_{(a_j^*,b_j^*)}$. By a similar argument as in step 1, we obtain $\text{val}\eqref{pb:NN_reg_rel} = \text{val}\eqref{intro_pb:NN_reg}$ and \eqref{eq:solution_eq} for \eqref{intro_pb:NN_reg}.
\end{proof}

\section{A generalization bound of the trained shallow NN}\label{sec:generalization}
\subsection{A generalization bound} 
We study the generalization property of shallow NN \eqref{eq:shallow} with some trained parameters $\Theta = \{(\omega_j,a_j,b_j) \in \R^{d+2}\}_{j=1}^P$. 
Consider some testing dataset $(X',Y') = \{(x_i',y_i') \in \R^{d+1}\}_{i=1}^{N'}$ with $N'\in \mathbb{N}_+$, which differs from the training one $(X,Y) = \{(x_i,y_i) \in \R^{d+1}\}_{i=1}^{N}$.
Ususally, the generalization quality is determined by the performance of this shallow NN on the testing set $(X',Y')$, which is assessed by comparing the actual values $y_i'$ with the predictions $f_{\text{shallow}}(x_i', \Theta)$ for each $i=1,\ldots,N'$. In our context, rather than individually evaluating the differences, we analyze the discrepancies in their empirical distributions to simplify the analysis.
Recall the definition of empirical distributions from \eqref{eq:distributions}.

In Theorem \ref{thm:generalization}, we study the generalization error by the Kantorovich-Rubinstein distance \cite[Eq.\@ 6.3]{villani2009} between the true distribution of the testing dataset $m_{\text{test}}$ and the predicted distribution $ m_{\text{pred}}(\Theta)$. The Kantorovich\ -\ Rubinstein distance is defined as: Let $\mathcal{X}$ be a  Euclidean space and $\mathcal{P}(\mathcal{X}) $ be the set of probability measures in $\mathcal{X}$, for any $\mu, \nu \in \mathcal{P}(\mathcal{X})$,
\begin{equation*}
    d_{\text{KR}}(\mu,\nu) \coloneqq \sup_{F\in \text{Lip}_1(\mathcal{X}) } \int_{\mathcal{X}} F  d (\mu- \nu),
\end{equation*}
where Lip$_1(\mathcal{X})$ is the set of all 1-Lipschitz continuous functions on $\mathcal{X}$. We note that when the support sets of \(\mu\) and \(\nu\) are compact (which include empirical measures), the Kantorovich-Rubinstein distance is equivalent to the Wasserstein-1 distance \cite[Rem.\@ 6.5]{villani2009}.

\begin{thm}\label{thm:generalization}
    Assume that the activation function \(\sigma\) is \(L\)-Lipschitz. Then, for any $\Theta\in \R^{(d+2)P}$, we have
    \begin{equation}\label{eq:generalization}
        d_{\textnormal{KR}}(m_{\textnormal{test}}, m_{\textnormal{pred}} (\Theta)) \leq \underbrace{d_{\textnormal{KR}}(m_{\textnormal{train}}, m_{\textnormal{test}}) + d_{\textnormal{KR}}(m_X, m_{X'})}_{\textnormal{Irreducible error from datasets}} + r(\Theta),
    \end{equation}
    where
    \begin{equation}\label{eq:r}
        r(\Theta) = \underbrace{\frac{1}{N} \sum_{i=1}^N \left|f_{\textnormal{shallow}} (x_i, \Theta) - y_i\right|}_{\textnormal{Bias from training}} + \underbrace{d_{\textnormal{KR}}(m_X, m_{X'}) L  \sum_{j=1}^P |\omega_j|\|a_j\|}_{\textnormal{Standard deviation}}.
    \end{equation}
\end{thm}

\begin{proof}
The proof is presented in Section \ref{sec:proof}.
\end{proof}

From Theorem \ref{thm:generalization}, we observe that the first two terms on the right-hand side of \eqref{eq:generalization} are independent of the parameters \(\Theta\). These terms depend on the difference between the training and testing datasets, which is expected, as a highly accurate model cannot be achieved if these sets differ significantly. Consequently, we refer to these terms as the irreducible error arising from the datasets.

In contrast, the third term, \(r(\Theta)\), depends on the learned parameters \(\Theta\). As seen in \eqref{eq:r}, the residual term \(r(\Theta)\) is composed of two parts: (1) the first term represents the fidelity error on the training set, which is referred to as the bias from training, and (2) the second term is related to the sensitivity of the trained NN, playing the role of the standard deviation.

In the following remark, using an approach similar to the proof of Theorem \ref{thm:generalization}, we provide generalization bounds on the mean-\(\ell^p\) error (see \eqref{eq:mlp1} and \eqref{eq:mlp2}) for the specific case where \( N' = N \).
\begin{rem}[Mean-$\ell^p$ error]\label{rem:mean-lp}
    Assume that $N'=N$. Let $ \text{Sym}(N)$ be the set of all permutations of $\{1,2,\ldots,N\}$. Recall the definitions of the Wasserstein distance \cite[Prop.\@ 2.1, 2.2]{peyre2019computational} and the Bottleneck distance \cite[Sec.\@ 3.1]{cohen2005stability} between \( m_{\text{train}} \) and \( m_{\text{test}} \): 
    \begin{align*}
        W_p(m_{\text{train}},m_{\text{test}}) & \coloneqq \min_{\tau\in \text{Sym}(N)} \left( \frac{1}{N}\sum_{i=1}^N 
 \|(x_i-x'_{\tau(i)},\, y_i - y'_{\tau(i)})\|^p \right)^{1/p}, \quad \text{for }p\geq 1;\\
        W_{\infty}(m_{\text{train}},m_{\text{test}}) & \coloneqq   \min_{\tau\in \text{Sym}(N)} \max_{i}
 \|(x_i-x'_{\tau(i)}, \, y_i - y'_{\tau(i)})\|.
    \end{align*}
Then, using a similar approach as in the proof of Theorem \ref{thm:generalization}, we deduce the following generalization bounds on the mean-\(\ell^p\) error:
\begin{equation}\label{eq:mlp1}
\begin{split}
     \left(\frac{1}{N}\sum_{i=1}^N |y_i' - f_{\text{shallow}}(x_i',\Theta)|^p\right)^{1/p} \leq & W_p(m_{\text{train}},m_{\text{test}}) 
     + \left(\frac{1}{N} \sum_{i=1}^N \left|f_{\textnormal{shallow}} (x_i, \Theta) - y_i\right|^p \right)^{1/p} \\ & + W_p(m_{\text{train}},m_{\text{test}})  L  \sum_{j=1}^P |\omega_j|\|a_j\|, \quad \text{for } p\geq 1,
\end{split}
\end{equation}
and 
\begin{equation}\label{eq:mlp2}
\begin{split}
    \max_{i} |y_i' - f_{\text{shallow}}(x_i',\Theta)| \leq & W_{\infty}(m_{\text{train}},m_{\text{test}}) 
     + \max_i\left|f_{\textnormal{shallow}} (x_i, \Theta) - y_i\right| \\
     & + W_{\infty}(m_{\text{train}},m_{\text{test}})  L  \sum_{j=1}^P |\omega_j|\|a_j\|.
\end{split}
\end{equation}
\end{rem}

\subsection{Generalization bounds by primal solutions of ReLU NNs}\label{sec:generalization_relu}
In this subsection, we examine the remainder term \(r(\Theta)\) provided in \eqref{eq:r}, where \(\Theta\) corresponds to the solutions of \eqref{intro_pb:NN_epsilon} and \eqref{intro_pb:NN_reg}. For simplicity, we fix \(\sigma\) as the ReLU function. We first use a scaling argument to show that, in this context, the value of \(r(\Theta)\) is independent of the a priori choice of the domain \(\Omega\).

\begin{lem}\label{lm:relu}
 Let $\sigma$ be the ReLU function, and $\Omega_0$ be a subset of $\mathbb{R}^{d+1}$ satisfying Assumption \ref{ass:sigma}(2). Then, for any $k > 0$, the following holds:
\begin{equation*}
    \inf_{\Theta \in S\eqref{intro_pb:NN_epsilon} \textnormal{ with }\Omega = \Omega_0} r(\Theta) = \inf_{\Theta' \in S\eqref{intro_pb:NN_epsilon} \textnormal{ with }\Omega = k \Omega_0} r(\Theta'),
\end{equation*}
The same conclusion holds for \eqref{intro_pb:NN_reg} when $\ell(\cdot) = |\cdot|$.
\end{lem}
\begin{proof}
Let us denote by \( V_0 \) and \( V_1 \) the values of \eqref{intro_pb:NN_epsilon} with \( \Omega = \Omega_0 \) and \( \Omega = k \Omega_0 \), respectively. Let \( \Theta = (\omega_j, a_j, b_j)_{j=1}^P \) be a solution of \eqref{intro_pb:NN_epsilon} with \( \Omega = \Omega_0 \). By the \(1\)-homogeneity of the ReLU function, for any \( k > 0 \) and \( x \in \mathbb{R}^d \), we have
\begin{equation*}
    \omega_j \sigma(\langle a_j, x \rangle + b_j) = \frac{\omega_j}{k}  \sigma\left(\left\langle k a_j, x \right\rangle + k b_j\right), \quad \text{for } j = 1, \ldots, P.
\end{equation*}
It follows that \( \Theta^k = \left(\frac{\omega_j}{k}, k a_j, k b_j\right)_{j=1}^P \) is in the admissible set of \eqref{intro_pb:NN_epsilon} with \( \Omega = k \Omega_0 \). Therefore,
\begin{equation*}
    V_1 \leq \frac{V_0}{k}.
\end{equation*}
By exchanging the roles of \( \Omega_0 \) and \( k \Omega_0 \), we obtain
\begin{equation*}
    V_0 \leq k V_1.
\end{equation*}
It follows that \( V_0 = k V_1 \) and \( \Theta^k \) is a solution of \eqref{intro_pb:NN_epsilon} with \( \Omega = k \Omega_0 \). Recall that
\begin{equation*}
    r(\Theta) = \frac{1}{N} \sum_{i=1}^N \left| f_{\textnormal{shallow}} (x_i, \Theta) - y_i \right| + d_{\textnormal{KR}}(m_X, m_{X'}) L \sum_{j=1}^P |\omega_j| \|a_j\|.
\end{equation*}
This implies \( r(\Theta) = r(\Theta^k) \). Since \( \Theta \) is an arbitrary solution of \eqref{intro_pb:NN_epsilon} with \( \Omega = \Omega_0 \), we obtain
\begin{equation*}
   \inf_{\Theta' \in S\eqref{intro_pb:NN_epsilon} \textnormal{ with } \Omega = k \Omega_0} r(\Theta') \leq \inf_{\Theta \in S\eqref{intro_pb:NN_epsilon} \textnormal{ with } \Omega = \Omega_0} r(\Theta).
\end{equation*}
By exchanging the roles of \( \Omega_0 \) and \( k \Omega_0 \), we obtain
\begin{equation*}
   \inf_{\Theta \in S\eqref{intro_pb:NN_epsilon} \textnormal{ with } \Omega = \Omega_0} r(\Theta) \leq \inf_{\Theta' \in S\eqref{intro_pb:NN_epsilon} \textnormal{ with } \Omega = k \Omega_0} r(\Theta').
\end{equation*}
The conclusion follows. The result for \eqref{intro_pb:NN_reg} can be deduced similarly.
  \end{proof}

 Thanks to the previous lemma, we fix $\Omega$ as the unit ball in $\mathbb{R}^{d+1}$ in Proposition \ref{prop:lambda} to study upper bounds of $r(\Theta)$ with $\Theta$ solutions of the primal problems.
 
\begin{prop}\label{prop:lambda}
Let \( P \geq N \). Fix \(\sigma\) as the ReLU function and \(\Omega\) as the unit ball in \(\mathbb{R}^{d+1}\). Consider the fidelity error \(\ell(\cdot) = |\cdot|\). For any \(\epsilon \geq 0\) and \(\lambda > 0\), let \(\Theta_{\epsilon}\) and \(\Theta^{\textnormal{reg}}_{\lambda}\) be the solutions of \eqref{intro_pb:NN_epsilon} and \eqref{intro_pb:NN_reg}, respectively. Then, the following inequalities hold:
\begin{align}
       & r(\Theta_{\epsilon}) \leq \mathcal{U}(\epsilon)\coloneqq \epsilon + C_{X,X'} \,\textnormal{val}\eqref{pb:NN_epsilon_rel}; \label{eq:generalization_epsilon}\\[0.6em]
        & r(\Theta^{\textnormal{reg}}_{\lambda}) \leq \mathcal{L}(\lambda)\coloneqq  \max \{\lambda^{-1}, \, C_{X,X'}\} \, \textnormal{val}\eqref{pb:NN_reg_rel}, \label{eq:generalization_lambda}
    \end{align}
    where $C_{X,X'} =d_{\textnormal{KR}} (m_X,m_{X'})$.
\end{prop}
\begin{proof}
Given that $\Omega$ is the unit ball, for any \(\Theta = (\omega_j, a_j, b_j)_{j=1}^P \in (\mathbb{R} \times \Omega)^P\),  we have $\|a_j\|\leq 1$. Since the ReLU function is 1-Lipschitz, we obtain by the formula of $r(\Theta)$ in \eqref{eq:r} that
\begin{equation}\label{eq:motivation}
     r(\Theta) \leq \frac{1}{N} \sum_{i=1}^N \left|f_{\textnormal{shallow}}(x_i, \Theta) - y_i\right| + d_{\textnormal{KR}}(m_X, m_{X'}) \sum_{j=1}^P |\omega_j|.
\end{equation}
The constraints in \eqref{intro_pb:NN_epsilon} imply that \( \sum_{i=1}^N \left|f_{\textnormal{shallow}}(x_i, \Theta_{\epsilon}) - y_i\right| /N \leq \epsilon\).
Since \( P \geq N \), we deduce from Theorem \ref{thm:NN_exists_0} that \(\text{val}\eqref{intro_pb:NN_epsilon} = \text{val}\eqref{pb:NN_epsilon_rel}\). Consequently, the inequality \eqref{eq:generalization_epsilon} follows.

Additionally, by Theorem \ref{thm:NN_exists_0}, for the regression problem, it holds that:
\[
\text{val}\eqref{pb:NN_reg_rel} = \text{val}\eqref{intro_pb:NN_reg} = \frac{\lambda}{N} \sum_{i=1}^N \left|f_{\textnormal{shallow}}(x_i, \Theta^{\text{reg}}_{\lambda}) - y_i\right| + \sum_{j=1}^P |\omega_j|.
\]
This, together with \eqref{eq:motivation}, leads to \eqref{eq:generalization_lambda}.
\end{proof}

\begin{rem}[Motivation of primal problems]\label{rem:motivation}
The upper bound of \( r(\Theta) \) from \eqref{eq:motivation} aligns with the structure of our primal regression problem \eqref{intro_pb:NN_reg}. This provides an additional motivation, beyond sparsity, for using \( \|\omega\|_{\ell^1} \) as the regularization term in \eqref{intro_pb:NN_reg}. Similarly, we adopt this penalty in the representation problems \eqref{intro_pb:NN_exact} and \eqref{intro_pb:NN_epsilon} as part of the objective function.

An improved alternative to the penalty \( \|\omega\|_{\ell^1} \) is \( \sum_{j=1}^P |\omega_j|\|a_j\| \), which originates directly from the formula of \( r(\Theta) \) in \eqref{eq:r}. This penalty is also closely related to the Barron norm of Shallow NNs, as discussed in \cite{ma2022barron}. However, with this penalty, the Representer Theorem (Theorem \ref{thm:representation}) cannot be applied to analyze the relaxed problems. Instead, recent extensions of Theorem \ref{thm:representation} in \cite{boyer2019representer} offer insights to address this case, forming one of the perspectives of this article.
\end{rem}

\begin{rem}[Optimality of $P=N$]\label{rem:P2}
In Remark \ref{rem:P1}, we have shown that $P = N$ is optimal in terms of the optimization value and the complexity of the neural network ansatz. Furthermore, by Proposition \ref{prop:lambda},  the upper bounds $\mathcal{U}(\epsilon)$ and $\mathcal{L}(\lambda)$ are independent of $P$ when $P \geq N$, as \textnormal{val}\eqref{pb:NN_epsilon_rel} and \textnormal{val}\eqref{pb:NN_reg_rel} do not depend on \(P\). Consequently, \(P = N\) is the optimal choice for the generalization bounds as well.
\end{rem}

\begin{rem}[Order of \(d_{\text{KR}}(m_X, m_{X'})\)]\label{rem:ot}
   In general, the distance $ d_{\text{KR}}(m_X, m_{X'})$ is challenging to be computed, as it requires solving an optimal transport problem (it is equivalent to the Wasserstein-1 distance). However, according to \cite[Thm.\@ 6.1]{dobric1995asymptotics}, \(\mathbb{E}_{X,X' \sim m}[d_{\text{KR}}(m_X, m_{X'})] \sim \mathcal{O}(1/N^{1-1/d})\), as $N\to\infty$, in the particular case \(N = N'\) and \(X, X'\) are independently sampled from the same distribution \(m \in \mathcal{P}(\mathbb{R}^d)\), with compact support and \(d \geq 3\). For further relevant results, see \cite[p.\@ 111]{villani2009}.
\end{rem}

\begin{rem}[Robust optimization on generalization bounds]\label{rem:robust}
   The estimates in \eqref{eq:generalization_epsilon}-\eqref{eq:generalization_lambda} provide upper bounds on the remainder term \( r(\cdot) \), for the solutions of the primal problems. These bounds depend on the hyperparameters and the distance between the training and testing feature sets. By minimizing \( \mathcal{U} \) w.r.t.\@ \( \epsilon \) and \( \mathcal{L} \) w.r.t.\@ \( \lambda \), 
   we obtain optimal hyperparameters that robustly control the worst-case behavior of the remainder terms. 
\end{rem}

\subsection{Optimal choices of the hyperparameters}\label{sec:parameters}
As noted in Remark \ref{rem:robust}, we are interested in finding the minimizers of the upper bound functions \(\mathcal{L}(\lambda)\) and \(\mathcal{U}(\epsilon)\), as defined in Proposition \ref{prop:lambda}. In the following two remarks, along with Figure \ref{fig_epsilon}, we present these minimizers and illustrate the qualitative properties of \(\mathcal{L}\) and \(\mathcal{U}\).

\begin{rem}[Minimizer of $\mathcal{L}$]\label{rem:lam}
Since \(\ell\) is positive, it is straightforward to see that \(\text{val}\eqref{intro_pb:NN_reg}\) increases with respect to \(\lambda\), in particular when \(\lambda \geq C_{X,X'}^{-1}\). On the other hand, for \(\lambda < C_{X,X'}^{-1}\), the function \(\text{val}\eqref{pb:NN_reg_rel}/\lambda\) decreases due to the positivity of \(\sum_{j=1}^N|\omega_j|\). Therefore, \(\mathcal{L}\) is minimized at \(\lambda^{*} = C_{X,X'}^{-1}\) and the minimum is val\eqref{pb:NN_reg_rel} for $\lambda=\lambda^*$. On the other hand, the limits  \(\lambda \to 0^{+}\) and \(\lambda \to +\infty\), lead to \(\mathcal{L}(0^{+}) = \|Y\|_{\ell^1}/N\) and \(\mathcal{L}(+\infty) = C_{X,X'} \, \text{val}\eqref{intro_pb:NN_exact}\), where $Y=(y_1,\ldots, y_N)$.  The qualitative behavior of \(\mathcal{L}(\lambda)\) is described the left-hand plot of Figure \ref{fig_epsilon}. In the scenario of Remark \ref{rem:ot}, the optimal $\lambda^*$ is of the order of $\mathcal{O}(N^{1-1/d})$.
\end{rem}

\begin{rem}[Minimizer of $\mathcal{U}$]\label{rem:eps}
    The function \(\mathcal{U}\) is proven to be convex in Theorem \ref{thm:epsilon}. There are two scenarios for the minimizer of \(\mathcal{U}\). In the first case, where \(C_{X,X'} \leq c_0^{-1}\), the optimal hyperparameter is \(\epsilon^{*} = 0\), as shown by the red curve in the right-hand plot of Figure \ref{fig_epsilon}. This implies that it suffices to consider the exact representation problem \eqref{intro_pb:NN_exact}. Here, the threshold $c_0$ is defined by:
    \begin{equation}\label{c0}
        c_0 \coloneqq \min_{p\in \R^N} \left\{\, \|p\|_{\ell^1}\, \Big|\,  p \in \argmax_{\|\phi^{*} \, p\|_{\mathcal{C}(\Omega)} \leq 1}  \langle Y, p \rangle \right\},
    \end{equation}
 where $\phi^{*} \colon \R^N \to \mathcal{C}(\Omega)$, $ p\mapsto \sum_{i=1}^N p_i\sigma(\langle a, x_i \rangle + b)$, is the adjoint operator of $\phi$, and the maximization problem is the dual problem of \eqref{pb:NN_exact_rel}.

    In contrast, when \( C_{X,X'} > c_0^{-1} \), the hyperparameter \(\epsilon^{*}\) is optimal if and only if the following first-order optimality condition holds:
\begin{equation}\label{first-order}
    C_{X,X'}^{-1} \in [c_{\epsilon^*}, C_{\epsilon^*}],
\end{equation}
where \( c_{\epsilon^*} \) and \( C_{\epsilon^*} \) are defined in \eqref{eq:c_epsilon}, which are analogous to the definition of \( c_0 \) by considering the dual problem of \eqref{pb:NN_epsilon_rel}. From \eqref{first-order} and the convexity of \( \mathcal{U} \), we deduce that the set of optimal \( \epsilon^* \) forms a closed interval, which may reduce to a singleton. This is illustrated by the blue curve in the right-hand plot of Figure \ref{fig_epsilon}.

   For a formal description of these two cases, we refer to Theorem \ref{thm:epsilon}. The statement and proof of Theorem \ref{thm:epsilon} is technical and is contained in Appendix \ref{app:dual}.
\end{rem}

\begin{figure}[htp]
\centering
\begin{subfigure}[t]{0.48\textwidth}
\includegraphics[width=\linewidth]{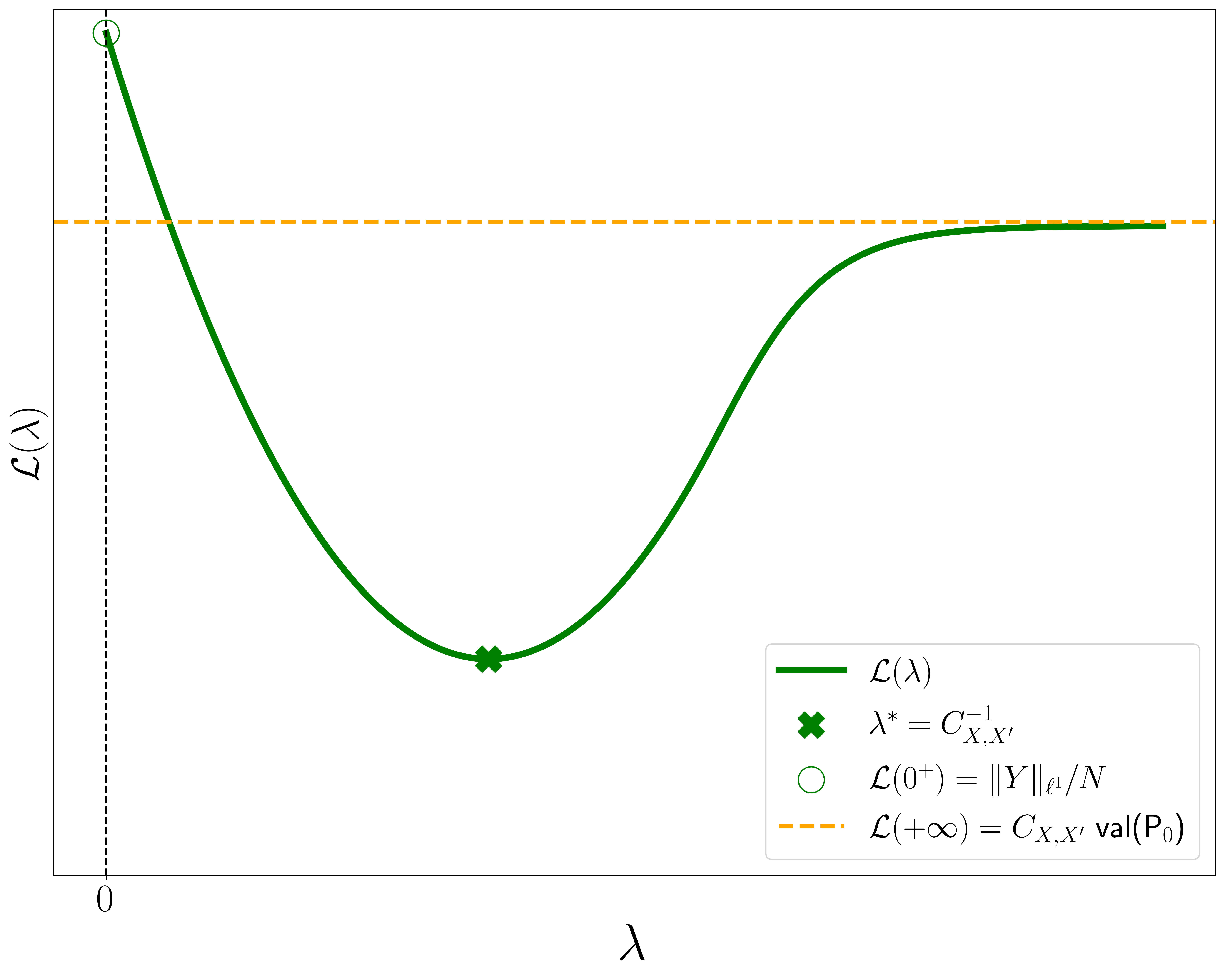}
\caption{Qualitative curve of $\mathcal{L}(\lambda)$.}
\label{fig:lambda}
\end{subfigure}
\begin{subfigure}[t]{0.48\textwidth}
\includegraphics[width=\linewidth]{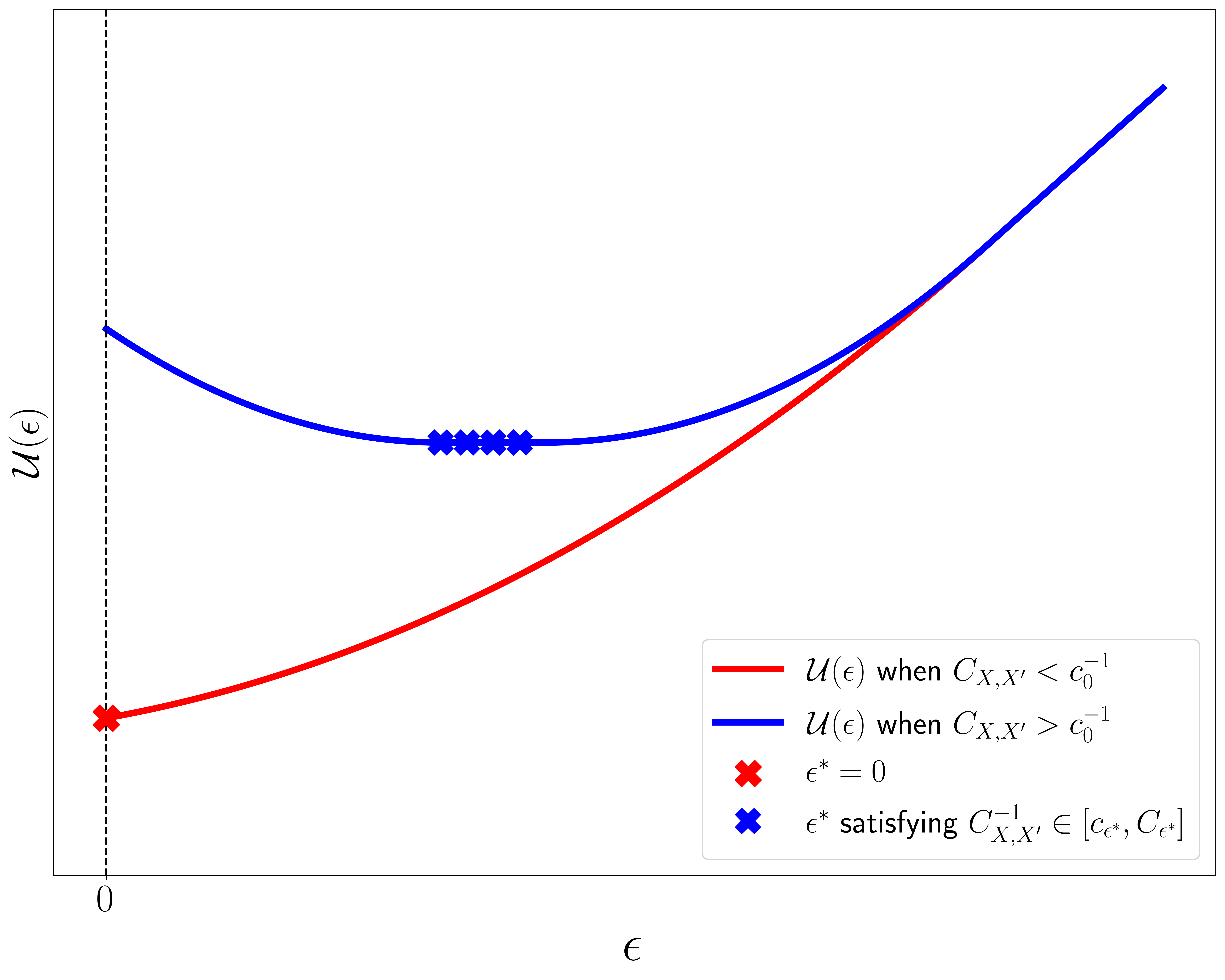}
\caption{Two scenarios of $\mathcal{U}(\epsilon)$ from Theorem \ref{thm:epsilon}.}
\label{fig:epsilon}
\end{subfigure}
\caption{Qualitative curves of $\mathcal{L}$ and $\mathcal{U}$ in Proposition \ref{prop:lambda} and their minimizers.}
\label{fig_epsilon}
\end{figure}

\section{Numerical analysis and algorithms}\label{sec:algo}
In this section, we concentrate on the numerical analysis and algorithms for the primal problems. Specifically, we address \eqref{intro_pb:NN_epsilon} and \eqref{intro_pb:NN_reg}, as the exact representation problem \eqref{intro_pb:NN_exact} can be viewed as a special case of \eqref{intro_pb:NN_epsilon} when $\epsilon=0$. Consequently, the algorithm employed for \eqref{intro_pb:NN_epsilon} also applies to \eqref{intro_pb:NN_exact}. 

By Remarks \ref{rem:P1} and \ref{rem:P2}, it suffices to have $N$ neurons to achieve the minimal optimization value and generalization bounds. Therefore, we assume here the following: (1) $P = N$, and (2) the dataset $\{(X,Y)\} = \{(x_i, y_i) \in \R^{d+1}\}_{i=1}^N$ consists of distinct features.
By Theorem \ref{thm:NN_exists_0}, there is no gap between the primal problems and their convex relaxations. As a consequence, we aim to obtain solutions of \eqref{intro_pb:NN_epsilon} and \eqref{intro_pb:NN_reg} from their relaxed counterparts. However, these relaxed problems lie in the infinite-dimensional space $\mathcal{M}(\Omega)$. To address this challenge, we propose the following methods distinguishing two scenarios:
\begin{itemize}
    \item Discretize, then optimize. When the dimension of \(\Omega\) is small (less than $6$ from Remark \ref{rem:memory}), we discretize \(\mathcal{M}(\Omega)\), leading to the discretized problems \eqref{pb:NN_M_eq} and \eqref{pb:NN_reg_M_eq}. The errors introduced by these discretizations are estimated in Theorem \ref{thm:discrete}. The discretized problems are equivalent to linear programming ones (we consider $\ell(\cdot)=|\cdot|$ for the regression problem) and can be efficiently solved using algorithms from operations research. As shown in Theorem \ref{thm:simplex}, the solution obtained by the simplex algorithm provides solutions to the primal problems.
    \item Optimize, discretize, then sparsify. When the dimension of \(\Omega\) is high, we focus on the regression problem \eqref{intro_pb:NN_reg}. We apply SGD to the overparameterized version of \eqref{intro_pb:NN_reg}.  This method is explained as the discretization of the gradient flow associated with \eqref{pb:NN_reg_rel} in Remark \ref{rem:gradient flow}. We then employ a sparsification algorithm (see Algorithm \ref{alg1}) to refine the solution obtained by SGD, producing a solution of \eqref{intro_pb:NN_reg} with less than $N$ activated neurons.
\end{itemize}

\subsection{The simplex algorithm in the low-dimensional case}\label{sec:discretization}
Let us consider a collection of $M \in \mathbb{N}{+}$ ($M\geq N$) distinct points within the domain $\Omega$, denoted by $\Omega_{M} = \{(a_j,b_j) \in \Omega\}_{j=1}^M$. The set $\mathcal{M}(\Omega_M)$ can be viewed as an $M$-dimensional discretization of the infinite-dimensional space $\mathcal{M}(\Omega)$. 

We define the $N\times M$-matrix $A$ as follows:
\begin{equation}\label{eq:A}
    A_{ij}  = \sigma(\langle a_j , x_i\rangle+ b_j) ,\quad \text{for } i=1,\ldots, N, \text{ and }j=1,\ldots,M.
\end{equation}

The associated discretized problem of \eqref{pb:NN_epsilon_rel} is expressed as follows:
\begin{equation}\label{pb:NN_M_eq}\tag{\text{PD$_\epsilon$}}
     \inf_{\omega \in \R^{M}} \|\omega\|_{\ell^1}, \qquad \text{s.t. }  \left\|A \, \omega - Y\right\|_{\ell^{\infty}} \leq \epsilon.
\end{equation}
On the other hand, since $\ell(\cdot)$ is fixed to be the absolute value function, the discretized version of \eqref{pb:NN_reg_rel} is as follows:
\begin{equation}\label{pb:NN_reg_M_eq}\tag{PD$^{\text{reg}}_{\lambda}$}
     \inf_{\omega \in \R^{M}} \|\omega\|_{\ell^1} + \frac{\lambda}{N} \| A\, \omega -Y \|_{\ell^1}.
\end{equation}

The existence of solutions of problem \eqref{pb:NN_reg_M_eq} is ensured by the coercivity and continuity of the objective function.

We provide a sufficient condition for the existence of solutions of problem \eqref{pb:NN_M_eq} in the following lemma. 
Let \( h(\Omega_M) \) denote the Hausdorff distance between \( \Omega \) and \( \Omega_M \):
\begin{equation}\label{eq:h}
   h(\Omega_M) \coloneqq  d_{\text{Hausdorff}}(\Omega,\Omega_M) = \max\left\{ \sup_{\theta_1 \in \Omega} \inf_{\theta_2 \in \Omega_M} \|\theta_1-\theta_2\| ,\; \sup_{\theta_2 \in \Omega_M} \inf_{\theta_1 \in \Omega} \|\theta_1-\theta_2\|  \right\}.
\end{equation}

\begin{lem}\label{lm:non-empty}
    Let Assumption \ref{ass:sigma} hold true. Assume that $\sigma$ is $L$-Lipschitz. 
Let  
    \begin{equation}\label{eq:V_0_D_x}
        V_0 \coloneqq \textnormal{val}\eqref{pb:NN_exact_rel}, \quad D_X \coloneqq \min_{x_i \in X } \sqrt{\|x_i\|^2+1}.
    \end{equation}
    If 
    \begin{equation}\label{Mecondition}
    h(\Omega_M)\leq \epsilon (V_0 L D_X)^{-1},     \end{equation}  then problem \eqref{pb:NN_M_eq} has solutions.
\end{lem}

\begin{proof}
The proof is presented in Section \ref{sec:proof}.
\end{proof}

\begin{rem}
    From Lemma \ref{lm:non-empty}, condition \eqref{Mecondition} between the inexact tolerance $\epsilon$ and the discretization precision $h(\Omega_M)$  ensures the existence of a solution for \eqref{pb:NN_M_eq}.

In other words, if we first fix the discrete set $\Omega_M$, then existence is guaranteed for $\epsilon \geq V_0 L D_X h(\Omega_M)$. Conversely, if we choose a small $\epsilon$, we need a finer discretization $\Omega_M$.
\end{rem}

The following theorem demonstrates the convergence of the discretized problems \eqref{pb:NN_M_eq} and \eqref{pb:NN_reg_M_eq} to their original counterparts. Recall the definition of $h(\Omega_M)$ from \eqref{eq:h}.

\begin{thm}\label{thm:discrete}
   Let Assumption \ref{ass:sigma} hold true. Assume that $\sigma$ is $L$-Lipschitz. Define $V_0$ and $D_X$ as in \eqref{eq:V_0_D_x}. Then, we have
    \begin{equation}\label{eq:regression_dis}
    \textnormal{val}\eqref{pb:NN_reg_rel} \leq \textnormal{val}\eqref{pb:NN_reg_M_eq} \leq \textnormal{val}\eqref{pb:NN_reg_rel} \, ( 1+ \lambda L D_X \, h (\Omega_M) ).
\end{equation}
Furthermore, if \eqref{Mecondition} holds,  then we have
\begin{equation}\label{eq:epsilon_dis}
    \textnormal{val}\eqref{pb:NN_epsilon_rel} \leq \textnormal{val}\eqref{pb:NN_M_eq} \leq \textnormal{val}\eqref{pb:NN_epsilon_rel} +  c_0 V_0 L D_X \, h (\Omega_M),
\end{equation}
where $c_0$ defined in \eqref{c0} is independent of $\Omega_M$ and $\epsilon$.
\end{thm}

\begin{proof}
    The proof is presented in Section \ref{sec:proof}.
\end{proof}

\begin{rem}
    In Theorem \ref{thm:discrete}, we demonstrate that for any fixed $\lambda, \, \epsilon > 0$, as $ h (\Omega_M) $ approaches $0$, the values of the discretized problems converge to those of their continuous counterparts. The convergence rates are given in \eqref{eq:regression_dis} and \eqref{eq:epsilon_dis}. We do not examine convergence in terms of solution distance, as the continuous problem solutions are non-unique. However, the convergence results in the value sense can be directly applied to \eqref{eq:generalization_epsilon} and \eqref{eq:generalization_lambda}, thereby providing generalization bounds of the solutions of the discretized problems.
\end{rem}

\begin{rem}[Dimension for discretization]\label{rem:memory}
Let us consider the efficiency and limitations of this discretization approach. Assume we have a dataset with \(N = 10^3\) data points and that \(\Omega\) is the unit hypercube in \(\mathbb{R}^{d+1}\). Let \(\Omega_M\) denote the uniform finite-difference mesh of \(\Omega\). The convergence results in Theorem \ref{thm:discrete} hold as the discretization step size tends to zero. If we choose a step size of 0.1, then the number of discretization points is \(M = 10^{d+1}\). Consequently, the total number of elements in the matrix \(A\) is \(10^{d+4}\). Given this, we can perform this discretization up to \(d+1 = 6\) on a computer with 16 GB of RAM without exceeding the available memory.
\end{rem}

For the remainder of this subsection, we consider the loss function \(\ell(\cdot) = |\cdot|\).
 Following \cite[Sec.\@ 3.1]{chen2001atomic}, the discretized problem \eqref{pb:NN_M_eq} can be re-written as the following linear program:
    \begin{equation}\label{pb:lp}
        \inf_{u^{+},u^{-} \in \R^{M}_{+}}\, \langle \mathbf{1}_M, u^{+} + u^{-}\rangle ,\quad \text{s.t. } -\epsilon  \mathbf{1}_N\leq A (u^{+}-u^{-}) - Y \leq \epsilon \mathbf{1}_N,
    \end{equation}
    where $ \mathbf{1}_M\in \R^M$ and $ \mathbf{1}_N\in \R^N$ are the vectors, with each coordinate taking the value $1$. Similarly, we provide the following  equivalent linear program to \eqref{pb:NN_reg_M_eq}:
\begin{equation}\label{pb:lp2}
    \inf_{\substack{v^{+},v^{-} \in \R^{M}_{+} \\ z^{+},z^{-} \in \R^{N}_{+}}}\, \langle \mathbf{1}_M, v^+ + v^- \rangle + \frac{\lambda}{N}\langle \mathbf{1}_N, z^+ + z^- \rangle  ,\quad \text{s.t. }  A(v^{+}-v^-) - (z^+-z^-)  = Y.
\end{equation}
 Linear programming problems \eqref{pb:lp} and \eqref{pb:lp2} can be solved efficiently using the simplex algorithm or the interior-point method (see \cite[Sec.\@ 3, Sec.\@ 9]{bertsimas1997introduction} for instance). 
 The equivalence between the discretized problems and the previous linear programming problems is analyzed in the next lemma. Recall that a basic feasible solution of a linear programming problem is defined as a point that is both a minimizer and an extreme point of the feasible set. 
   \begin{lem}\label{lm:lp}
   The following holds:
   \begin{enumerate}
       \item  If $(u^+,u^-)\in \R_{+}^{2M}$ is a solution of \eqref{pb:lp}, then $u^+-u^-$ is a solution of \eqref{pb:NN_M_eq}.
       \item If $(v^+,v^-,z^+,z^-)\in \R_{+}^{2(M+N)}$ is a solution of \eqref{pb:lp2}, then $v^+-v^-$ is a solution of \eqref{pb:NN_reg_M_eq}.
   \end{enumerate}
   Moreover, if $(u^+,u^-)$ and $(v^+,v^-,z^+,z^-)$ are basic feasible solutions, then
   \begin{equation}\label{eq:extreme_lp}
       \|u^{+}-u^-\|_{\ell^0} \leq N, \quad \|v^{+}- v^-\|_{\ell^0} \leq N.
   \end{equation}
   \end{lem}
 \begin{proof}
 The proof is presented in Section \ref{sec:proof}.
 \end{proof}
 
 Combining Lemma \ref{lm:lp} and the convergence of the simplex algorithm \cite[Thm.\@ 3.3]{bertsimas1997introduction}, we obtain the following result, which implies that  approximate solutions to the primal problems can be found by solving \eqref{pb:lp} and \eqref{pb:lp2}.

 \begin{thm}[Convergence of the simplex algorithm]\label{thm:simplex}
     Under the setting of Theorem \ref{thm:discrete}, let \((u^+, u^-)\) and \((v^+, v^-, z^+, z^-)\) be solutions of \eqref{pb:lp} and \eqref{pb:lp2} obtained from the simplex algorithm, respectively. Define \(\omega^{1} = u^+ - u^-\) and \(\omega^{2} = v^+ - v^-\).
     Let $I^1 $ and $I^2$ be two vectors collecting the index $j$ such that \(\omega^{1}_j \neq 0\) and \(\omega^{2}_j \neq 0\), respectively.
      Then, the following holds:
     \begin{itemize}
         \item The point $(\omega^*_j,a_j^*,b_j^*)_{j=1}^N$ is an $\mathcal{O}(h(\Omega_M))$-optimal solution of the primal problem \eqref{intro_pb:NN_epsilon}, where
         \begin{equation*}
            (\omega^*_j,a_j^*,b_j^*) = \begin{cases}
               (\omega^{1}_{I^{1}_j}, a_{I^{1}_j} , b_{I^{1}_j}) ,\quad & \text{if } j\leq \|\omega^{1}\|_{\ell^0},\\
                (0,0,0), & \text{otherwise}.
            \end{cases}
         \end{equation*}
         \item The point $(\omega^r_j,a_j^r,b_j^r)_{j=1}^N$ is an $\mathcal{O}(h(\Omega_M))$-optimal solution of the primal problem \eqref{intro_pb:NN_reg}, where
         \begin{equation*}
            (\omega^r_j,a_j^r,b_j^r) = \begin{cases}
               (\omega^{2}_{I^{2}_j}, a_{I^{2}_j} , b_{I^{2}_j}) ,\quad & \text{if } j\leq \|\omega^{2}\|_{\ell^0},\\
                (0,0,0), & \text{otherwise}.
            \end{cases}
         \end{equation*}
     \end{itemize}
 \end{thm}
 \begin{proof}
     The simplex algorithm terminates after a finite number of iterations and provides a basic feasible solution, as established in \cite[Thm.\@ 3.3]{bertsimas1997introduction}. By Lemma \ref{lm:lp}, $\omega^1$ and $\omega^2$ are solutions of \eqref{pb:NN_M_eq} and \eqref{pb:NN_reg_M_eq}, respectively. Moreover, $\|\omega^1\|_{\ell^0}$ and $\|\omega^2\|_{\ell^0}$ are both bounded by $N$. Then, the conclusions follow from Theorem \ref{thm:discrete} and the equivalence between the primal and relaxed problems as established in Theorem \ref{thm:NN_exists_0}.
 \end{proof}

\subsection{SGD and  sparsification in  high-dimensions}
The discretization approach presented above is challenged by the curse of dimensionality, as mentioned in Remark \ref{rem:memory}. Therefore, in high-dimensions, their equivalent linear programming problems are intractable. In this context, we propose to apply the SGD  method directly to the overparameterized primal regression problem \eqref{intro_pb:NN_reg}, where \( P > N \). This SGD can be viewed as a discretization of the gradient flow associated with the relaxed problem \eqref{pb:NN_reg_rel}, as mentioned in the following remark.

\begin{rem}[Gradient flow and SGD]\label{rem:gradient flow}
To derive the gradient flow, let us first rewrite \eqref{pb:NN_reg_rel} in an equivalent form within the probability space. Let $\theta = (\omega, a, b)$, then the problem becomes:
\begin{equation}\label{pb:gradient_flow}
    \inf_{\nu \in \mathcal{P}(\mathbb{R} \times \Omega)} F(\nu), \quad \text{where } 
    F(\nu) = \int_{\mathbb{R} \times \Omega} |\omega| \, d\nu + \frac{\lambda}{N} \sum_{i=1}^N \ell\left(\int_{\mathbb{R} \times \Omega} \omega \sigma(\langle a, x_i \rangle + b) \, d\nu(\theta) - y_i\right).
\end{equation}
The gradient flow associated with \eqref{pb:gradient_flow} is governed by the following degenerate parabolic equation:
\begin{equation}\label{eq:gradient_flow}
    \partial_t m(t, \theta) - \text{div}_{\theta}\left(m(t, \theta) \, \nabla_{\theta} \frac{\delta F}{\delta \nu}(m(t))(\theta) \right) = 0, \quad m(0,\cdot)=m_0,
\end{equation}
where $\frac{\delta F}{\delta \nu}$ is the first variation of $F$, given by:
\begin{equation*}
    \frac{\delta F}{\delta \nu}(m(t))(\theta) = |\omega| + \frac{\lambda}{N} \sum_{i=1}^N \ell'\left(\int_{\mathbb{R} \times \Omega} \tilde{\omega} \sigma(\langle \tilde{a}, x_i \rangle + \tilde{b}) \, dm(t, \tilde{\theta}) - y_i\right) \omega \sigma(\langle a, x_i \rangle + b).
\end{equation*}
From \cite[Thm.\@ 3.5]{chizat2018global}, under the assumption that $\text{supp}(m_0)$ is sufficiently dense (i.e., $\{0\} \times \Omega \subseteq \operatorname{supp}(m_0)$) and that $m(t)$ converges to some distribution $m_{\infty}$ as $t \to \infty$, it follows that $m_{\infty}$ is a solution of \eqref{pb:gradient_flow}. We then obtain a solution of \eqref{pb:NN_reg_rel} by their equivalence. 
In general, the convergence of \( m(t) \) to \( m_{\infty} \) is challenging to establish. A common approach to address this issue is to add an entropy penalty to \( F \) in \eqref{pb:gradient_flow}. For instance, introducing a relative entropy with respect to the Lebesgue measure results in an additional \( -\Delta m \) on the left-hand side of \eqref{eq:gradient_flow}, transforming it into a non-degenerate parabolic equation, which is more tractable. However, this modification disrupts the Dirac structure of solutions for the relaxed problems: in this case, \( m_{\infty} \) becomes an absolutely continuous measure, deviating from the sum of Dirac measures.

The gradient descent of an overparameterized Shallow NN with $P$ neurons can be interpreted as a discretization of \eqref{eq:gradient_flow} through the following two steps:
\begin{enumerate}
    \item (Discretization in space) Replace $m_0$ with an empirical measure of $P$ points. With this substitution, \eqref{eq:gradient_flow} is equivalent to an ODE system of $P$ particulars.
    \item (Discretization in time) Apply the forward Euler scheme to the previous ODE system, where the learning rate serves as the time step.
\end{enumerate}
The convergence of step (1) has been established in \cite[Thm.\@ 3.5]{chizat2018global} in the asymptotic sense. However, the convergence rate (with respect to $P$) of step (1) and the overall convergence of steps (1)–(2) remain open problems in the literature \cite{fernandez2022continuous} and are beyond the scope of our article. Numerically, we take $P = 2N$ in our experiments.
\end{rem}

We next present a sparsification method (Algorithm \ref{alg1}) that refines the solution obtained from the SGD algorithm, yielding a solution of the primal problem with at most \(N\) activated neurons (\(\|\omega\|_{\ell^0} \leq N\)). This method holds significant practical value for real-world applications, as it can substantially reduce the size of the final trained model.

Before the presentation of Algorithm \ref{alg1}, let us define the following filtering operator, which aims to eliminate zero weights \(\omega_j\) in \(\Theta\) and their corresponding \((a_j,b_j)\):
\begin{equation}\label{eq:F}
    \mathcal{F}\colon \bigcup_{p=1}^P (\mathbb{R} \times \Omega)^p \to \bigcup_{p=1}^P (\mathbb{R} \times \Omega)^p, \quad (\omega_j, a_j, b_j)_{j=1}^{p_1} \mapsto (\bar{\omega}_j, \bar{a}_j, \bar{b}_j)_{j=1}^{p_2},
\end{equation}
where \(1 \leq p_1 \leq P\), \(p_2\) is the number of non-zero elements in \(\omega = (\omega_j)_{j=1}^{p_1}\), \(\bar{\omega}_j\) is the \(j\)-th nonzero element of \(\omega\), and \((\bar{a}_j, \bar{b}_j) \in \Omega\) is the element \((a_{j'}, b_{j'})\) corresponding to \(\bar{\omega}_j\).

\begin{algorithm}[h]
\SetAlgoLined
\textbf{Input}: $\Theta = (\omega_j,a_j,b_j)_{j=1}^P \in (\R\times \Omega)^P$\;
\textbf{Initialization}: $ \Theta^0=(\omega^0,a^0,b^0) = \mathcal{F}(\Theta)$, where $\mathcal{F}$ is defined in \eqref{eq:F}; $A^0\in \R_{N\times p_0} $ with $A^0_{ij}  = \sigma(\langle a^0_j , x_i\rangle+ b^0_j)$ for $i=1,\ldots,N$ and $j=1,\ldots,p_0$, where $p_0$ is the dimension of $\omega^0$\;
\For{$k= 0,1,\ldots$}{

\medskip

\If{$A^k$ is full column rank}
{Return $\Theta^k$;}

\Else{
   Compute $ 0\neq \tilde{\omega} \in \R^{p_k}$ by solving $A^k \,  \tilde{\omega} = 0$\;
   \medskip 
   Compute $\alpha \in \R^{p_k}$ and $\beta\in \R$ by
   \begin{equation}\label{eq:alg_sparse_1}
       \alpha_{j} = 
           \frac{\tilde{\omega}_j}{ \omega_j^k}, \quad \text{for } j=1,\ldots, p_k ;\quad \beta = -\frac{1}{\alpha_{j^{*}}}, \quad \text{where } j^{*} \in \argmax_{j} |\lambda_{j}|;
   \end{equation}

    Compute $\hat{\omega}^{k}$ and $\bar{\omega}^k$ by 
   \begin{equation}\label{eq:alg_sparse_2}
       \hat{\omega}^{k}_j = 
          (1+ \beta \alpha_j) \omega_{j}^k,
         \qquad \bar{\omega}^{k}_j = 
          (1 - \beta \alpha_j) \omega_{j}^k, \quad \text{for }j=1,\ldots, p_k;
   \end{equation}

   Update $\Theta^{k+1}$ by
   \begin{equation}\label{eq:alg_sparse_3}
        \Theta^{k+1}= (\omega^{k+1},a^{k+1},b^{k+1})= \begin{cases}
          \mathcal{F}(\hat{\omega}^k,a^k,b^k) , \; & \text{if } \|\hat{\omega}^k\|_{\ell^1} \leq \| \bar{\omega}^k\|_{\ell^1}, \\
           \mathcal{F}(\bar{\omega}^k,a^k,b^k) & \text{otherwise},
        \end{cases}
   \end{equation}
    update $A^{k+1}$ and $p_{k+1}$ by the rule in the initialization step with $\Theta^{k+1}$.
}

\medskip

}
\caption{Sparsification algorithm}\label{alg1}
\end{algorithm}

\begin{lem}\label{lem:monotone}
    Let $\Theta^k=(\omega^k,a^k,b^k)$ be a sequence generated by Algorithm \ref{alg1} with any input $\Theta\in (\R\times \Omega)^P$. For any $k \geq 0$, it holds that
    \begin{align*}
     f_{\textnormal{shallow}}(x_i, \Theta^k) = f_{\textnormal{shallow}}(x_i, \Theta),\quad \text{for }i=1,\ldots, N; \quad    \| \omega^{k+1}\|_{\ell^1}  \leq  \| \omega^k\|_{\ell^1}.
    \end{align*}
   Additionally, if $\| \omega^{k+1}\|_{\ell^1} =  \| \omega^k\|_{\ell^1} $, then $\| \omega^{k+1}\|_{\ell^0}  \leq  \| \omega^k\|_{\ell^0} -1$.
\end{lem}
\begin{proof}
The proof is presented in Section \ref{sec:proof}.
\end{proof}

\begin{prop}\label{prop:sparse}
If the input $\Theta$ of Algorithm \ref{alg1} is a solution of \eqref{intro_pb:NN_reg} with $P> N$, then Algorithm \ref{alg1} terminates within a finite number of iterations, and the output is a solution of \eqref{intro_pb:NN_reg} with $P=N$. The same conclusion holds for \eqref{intro_pb:NN_exact} and \eqref{intro_pb:NN_epsilon}.
\end{prop}
 \begin{proof}
 Let $\Theta^k=(\omega^k,a^k,b^k)$ be a sequence generated by Algorithm \ref{alg1}. By Lemma \ref{lem:monotone},  $f_{\textnormal{shallow}}(x_i, \Theta^k)$ remains constant and $ \omega^{k+1}\|_{\ell^1}  \leq  \| \omega^k\|_{\ell^1}$. Combining with the fat that $\Theta$ is a solution of \eqref{intro_pb:NN_reg}, we deduce that $\|\omega^k\|_{\ell^1}$ is invariant. By Lemma \ref{lem:monotone}, $\|\omega^k\|_{\ell^0} \leq \|\omega^0\|_{\ell^0} - k$ for any $k\geq 1$. Thus, Algorithm \ref{alg1} terminates after at most $P$ iterations. Let $\Theta^K$ be the output. Therefore, $A^K$ is full column rank, implying that the column number of $A^K$ is less than $N$. This implies that $\Theta^K$ (up to appending zeros) is a solution of \eqref{intro_pb:NN_reg} with $P=N$.
\end{proof}

\begin{rem}
Proposition \ref{prop:sparse} deals with  the global minimimum of \eqref{intro_pb:NN_reg}. However, the sparsification algorithm \ref{alg1} also provides valuable insights into improving a local minimum. Specifically, if the input corresponds to a local minimizer of \eqref{intro_pb:NN_reg}, then according to Lemma \ref{lem:monotone}, each iteration of Algorithm \ref{alg1} results in either a reduction in the number of non-zero elements of $\omega^k$, thereby simplifying the model, or a decrease in the value of the objective function, which aids in escaping from this local minimum.
\end{rem}

\section{Numerical simulation}\label{sec:numerical}

\subsection{Data setting and pre-trained shallow NN} For the first numerical experiment, we consider 1,000 points in \(\mathbb{R}^2\), distributed across four classes labeled 1, 2, 3, and 4. By adding different scales of white noise to these points, we obtain three scenarios of datasets, as shown in Figure \ref{fig_data}. The standard deviations (std) of these noises are \(0.1\), \(0.22\), and \(0.4\), resulting in clearly separated, moderately separated, and largely overlapping domains for each class of points, respectively.

We apply the shallow NN \eqref{eq:shallow} to perform classification tasks for all scenarios depicted in Figure \ref{fig_data}. To evaluate the generalization performance, we allocate 500 points to the training dataset and the remaining 500 points to the testing dataset. According to Theorem \ref{thm:NN_exists}, the training data can be transformed into their corresponding labels using \eqref{eq:shallow} with 500 neurons. 

Thus, we pre-train a 500-neuron shallow NN on the training set using the SGD algorithm, targeting the mean squared error (MSE) loss. The activation function \(\sigma\) is fixed as the ReLU function, which satisfies Assumption \ref{ass:sigma}(1) and is 1-Lipschitz continuous. We run SGD for \(2 \times 10^4\) epochs with a learning rate of \(10^{-3}\). The testing accuracies according to 
\begin{equation*}
    \text{Accuracy} = \frac{\#\{|y'_{i, \text{pred}} - y_i'|<0.5\}}{N'}, \text{ where }\# \text{ is the cardinality,}
\end{equation*}
 for the three scenarios are 94\%, 84\%, and 70\%.

\begin{figure}[htbp]
    \centering
    \begin{subfigure}[b]{0.31\textwidth}
        \centering
        \includegraphics[width=\textwidth]{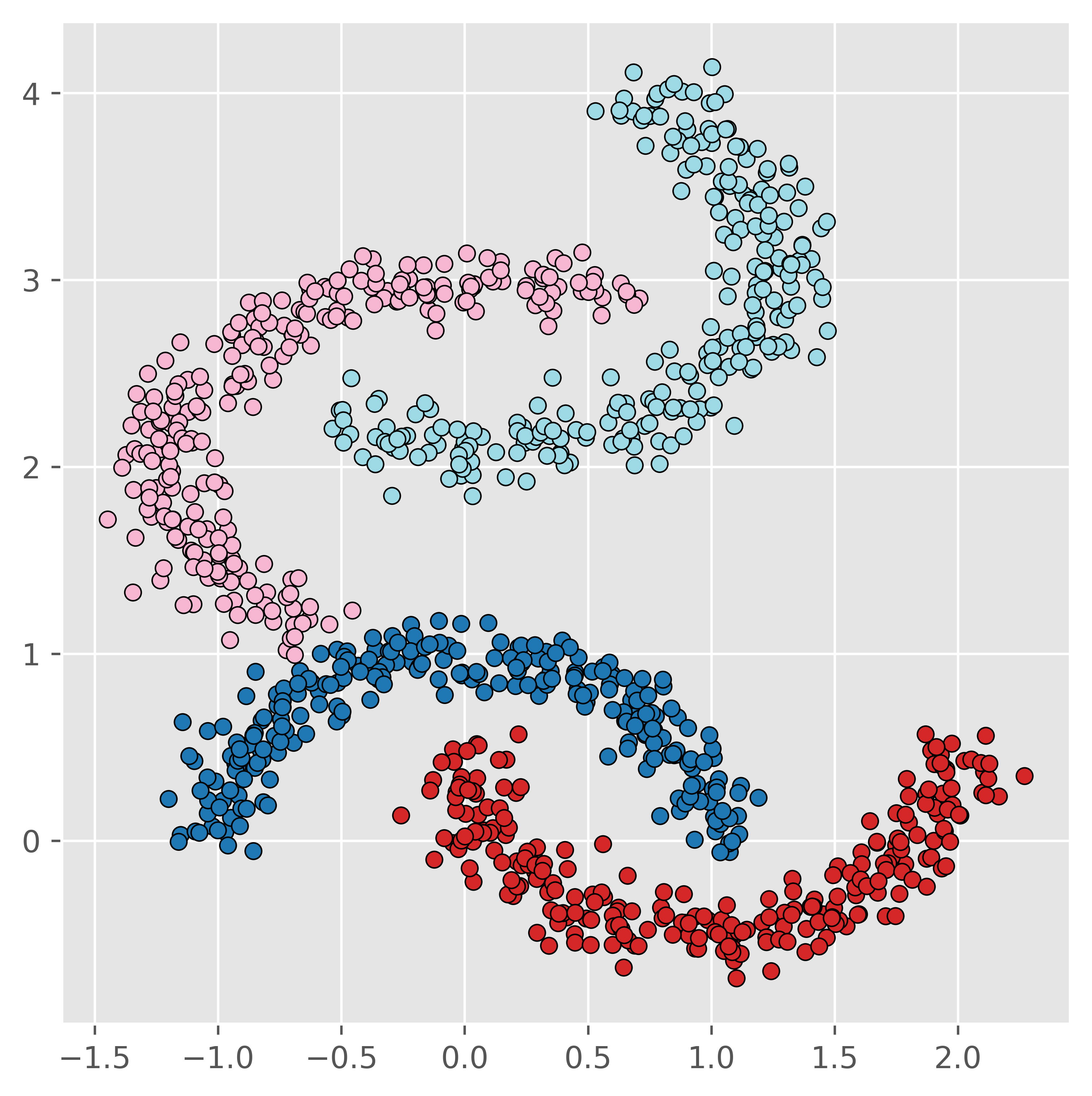}
        \caption{Small noise (std = 0.1).}
        \label{fig:sub1}
    \end{subfigure}
    \hspace{0.01\textwidth} 
    \begin{subfigure}[b]{0.31\textwidth}
        \centering
        \includegraphics[width=\textwidth]{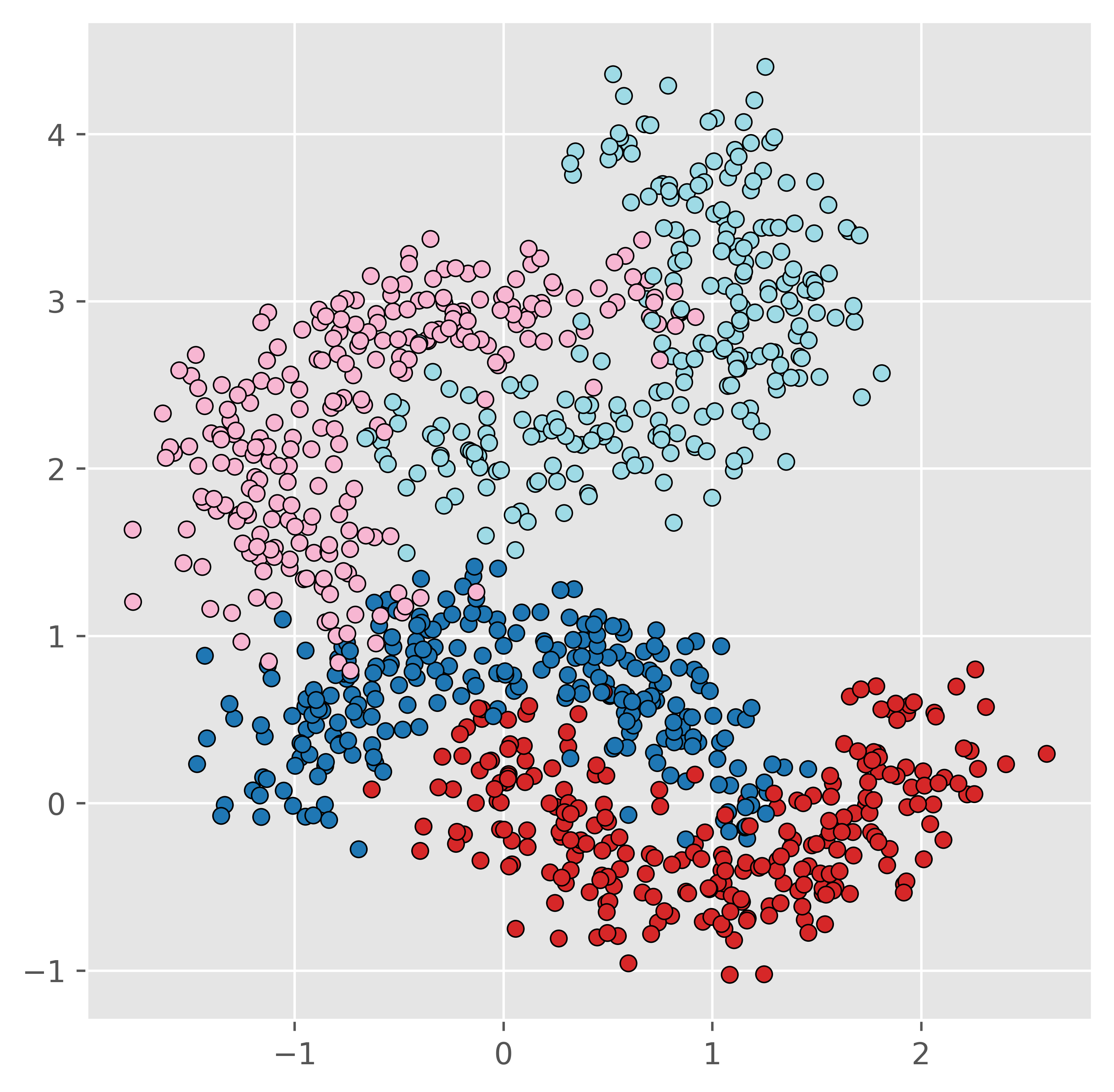}
        \caption{Moderate noise (std = 0.22).}
        \label{fig:sub2}
    \end{subfigure}
    \hspace{0.01\textwidth} 
    \begin{subfigure}[b]{0.31\textwidth}
        \centering
        \includegraphics[width=\textwidth]{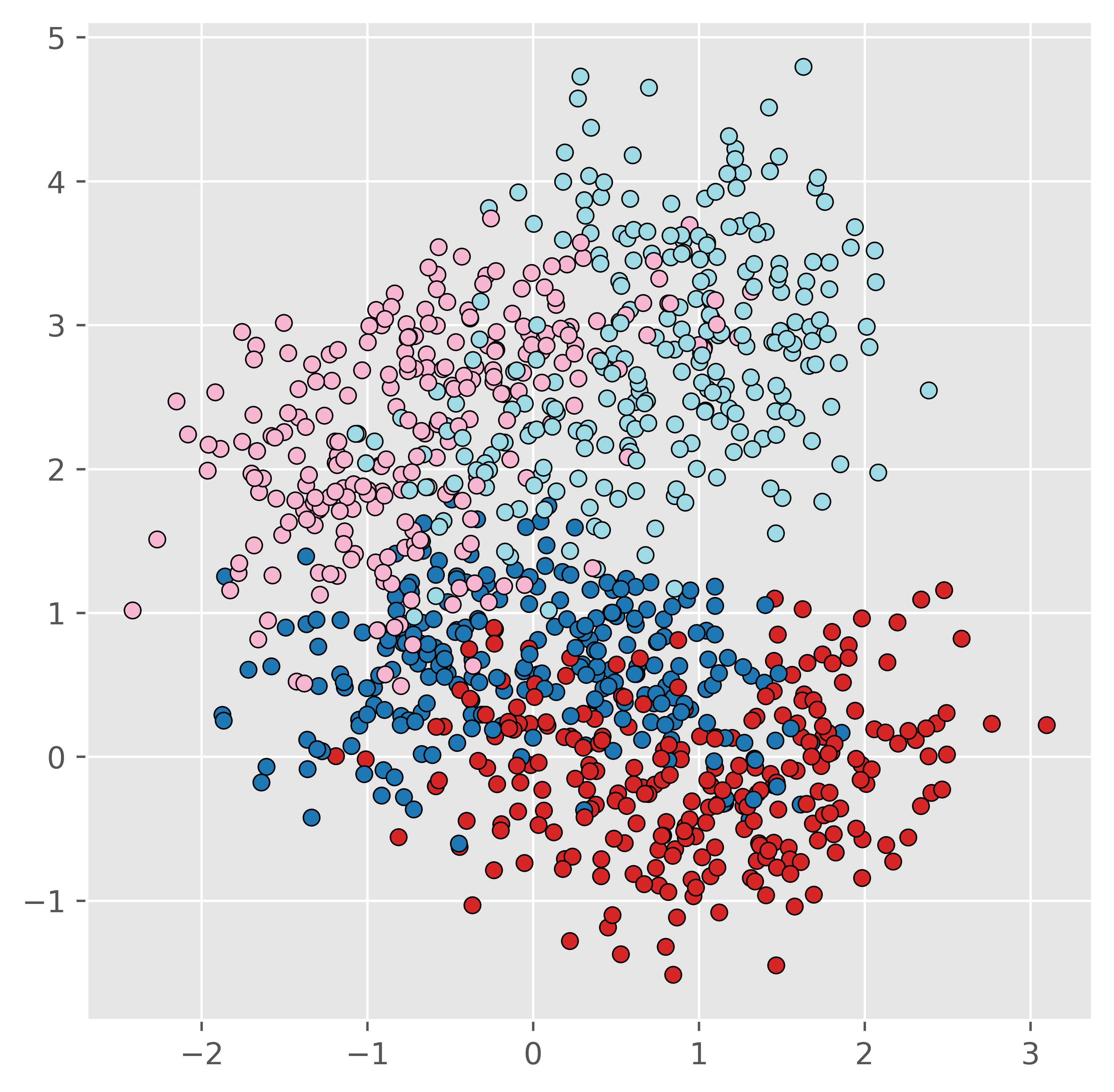}
        \caption{Large noise (std = 0.4).}
        \label{fig:sub3}
    \end{subfigure}
    \caption{Datasets with different standard deviations of noise.}
    \label{fig_data}
\end{figure}

\subsection{Numerical results}
We fix the domain $\Omega$  in our optimization problems \eqref{intro_pb:NN_epsilon} and \eqref{intro_pb:NN_reg} as the smallest cube containing all the parameters $\{(a_j,b_j)\}_{j=1}^{500}$ from the pre-trained models. Here, the pre-trained model is the outcome of the SGD algorithm applied to minimize the MSE loss over \( 2 \times 10^4 \) epochs. The selection of this \(\Omega\) is discussed in Remark \ref{rem:omega}. Subsequently, we discretize this domain $\Omega$ using $M = 10^4$ points, resulting in $\Omega_M$ for the discretized problems \eqref{pb:NN_M_eq} and \eqref{pb:NN_reg_M_eq}. We employ the simplex method to solve equivalent linear quadratic problems and obtain approximate solutions to the primal problems \eqref{intro_pb:NN_epsilon} and \eqref{intro_pb:NN_reg} as described in Theorem \ref{thm:simplex}. Finally, we determine the accuracy rates on the testing datasets using these solutions and plot them in Figure \ref{fig:accuracy}.
In Figure \ref{fig:accuracy}, we examine the testing accuracies for the three dataset scenarios depicted in Figure \ref{fig_data}. The left side of Figure \ref{fig:accuracy} shows the accuracy obtained by \(\Theta_{\epsilon}\), the numerical solution of \eqref{intro_pb:NN_epsilon} with the hyperparameter \(\epsilon \in [0,1]\). The right side displays the accuracy by \(\Theta^{\text{reg}}_\lambda\), the numerical solution of the regression problem \eqref{intro_pb:NN_reg}, with the hyperparameter \(\lambda\) ranging from \(1\) to \(10^5\) (note that the horizontal-axis is in logarithmic scale).

From Figure \ref{fig:accuracy} (A)-(B), we see that when the noise is small, the approximate representation problem \eqref{intro_pb:NN_epsilon} yields excellent results (99.6\%) as \(\epsilon\) approaches 0. This implies that the exact representation problem \eqref{intro_pb:NN_exact} has good generalization properties in this case. This is consistent with the illustrative curve of \(\mathcal{U}\) in Figure \ref{fig_epsilon} for the case where \(C_{X,X',L,D} < c_0^{-1}\), as small noise leads to a small   \(d_{\text{KR}}(m_X,m_{X'})\).
On the other hand, the accuracy of the solution for the regression problem \eqref{intro_pb:NN_reg} increases as \(\lambda\) approaches infinity, corresponding to the convergence of the regression problem stated in Remark \ref{rem:exact-regression}.

For the moderately noisy case, as shown in Figure \ref{fig:accuracy} (C)-(D), the exact representation (\(\epsilon = 0\)) does not yield a  significant improvement  compared to the pre-trained model. As \(\epsilon\) increases, the accuracy initially improves, reaching the maximum value for some \(\epsilon= \epsilon^{*} \in (0.4, 0.5)\), which exceeds the pre-trained accuracy. Beyond this threshold, the accuracy decreases. This is consistent with the illustrative curve of \(\mathcal{U}\) in Figure \ref{fig_epsilon} for the case where \(C_{X,X',L,D} > c_0^{-1}\). For the regression case, the accuracy increases with \(\lambda\) until \(\lambda\) becomes very large (around \(10^4\)), after which the accuracy starts to decrease.

The accuracy curves for the large noise case in Figure \ref{fig:accuracy} (E)-(F) exhibit similar qualitative properties to the previous one. However, in this scenario, the peak testing accuracy achieved by the solutions of \eqref{intro_pb:NN_epsilon} does not surpass that of the pre-trained model. This is reasonable, as the standard deviation of the noise is large, making it difficult to satisfy a uniform constraint on the representation error, which in turn leads to poor generalization properties. Nonetheless, the regression problem shows good performance for \(\lambda \in (10^3, 10^4)\). Additionally, the rate of decrease in accuracy after \(\lambda = 10^4\) is much faster than in Figure \ref{fig:accuracy} (D).

In conclusion, the results of this simulation demonstrate that the exact and approximate representation problems perform well when the datasets exhibit clear or moderately separable boundaries. In contrast, when these datasets feature heavily overlapping areas, it is more reasonable to consider the regression problem.
\begin{figure}[htbp]
\centering
\begin{subfigure}[t]{0.49\textwidth}
\includegraphics[width=\textwidth]{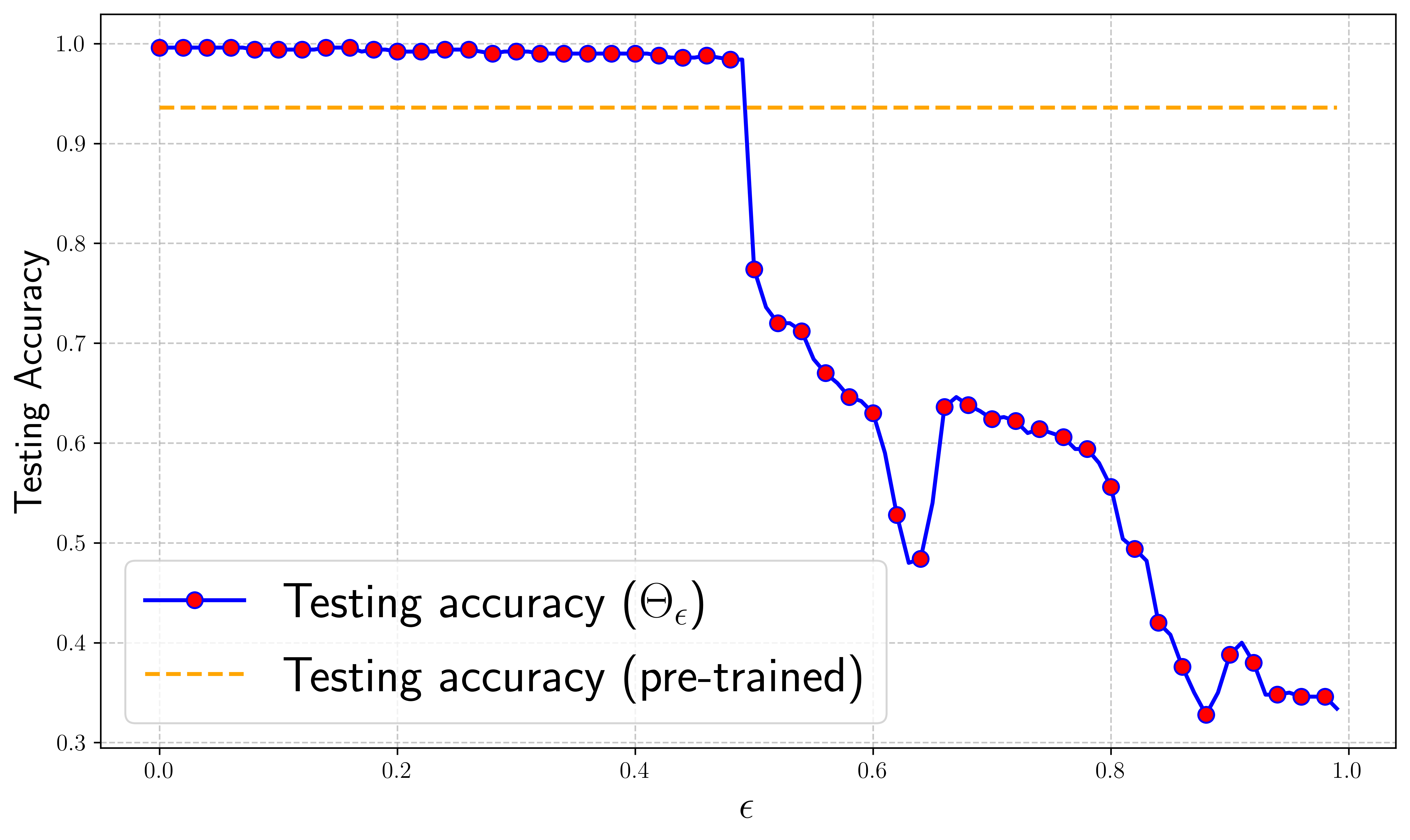}
\caption{Noise std = 0.1} 
\end{subfigure}\hspace*{\fill}
\begin{subfigure}[t]{0.49\textwidth}
\includegraphics[width=\textwidth]{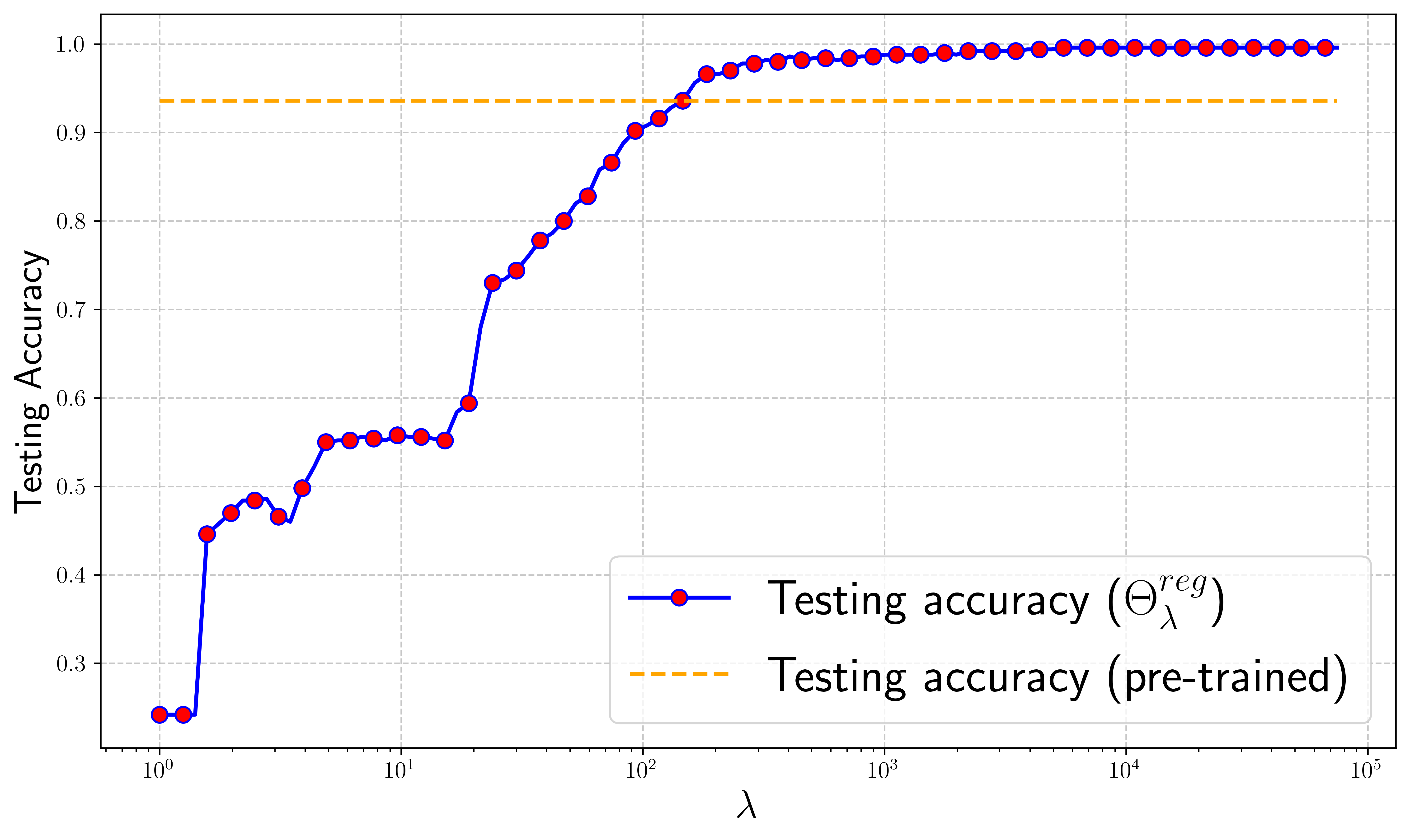}
\caption{Noise std = 0.1} 
\end{subfigure}

\vspace{0.05\textwidth}

\begin{subfigure}[t]{0.49\textwidth}
\includegraphics[width=\textwidth]{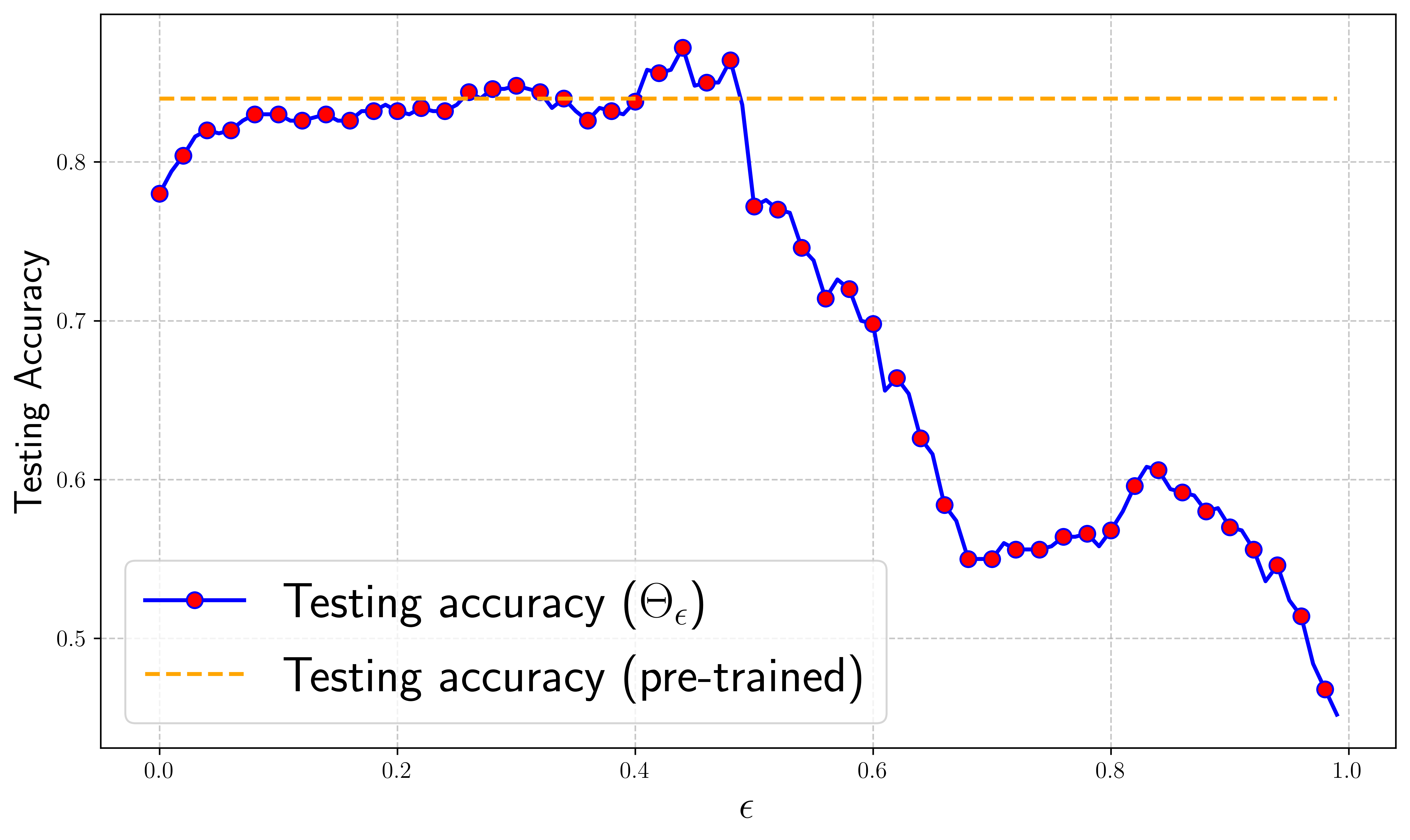}
\caption{Noise std = 0.22} 
\end{subfigure}\hspace*{\fill}
\begin{subfigure}[t]{0.49\textwidth}
\includegraphics[width=\textwidth]{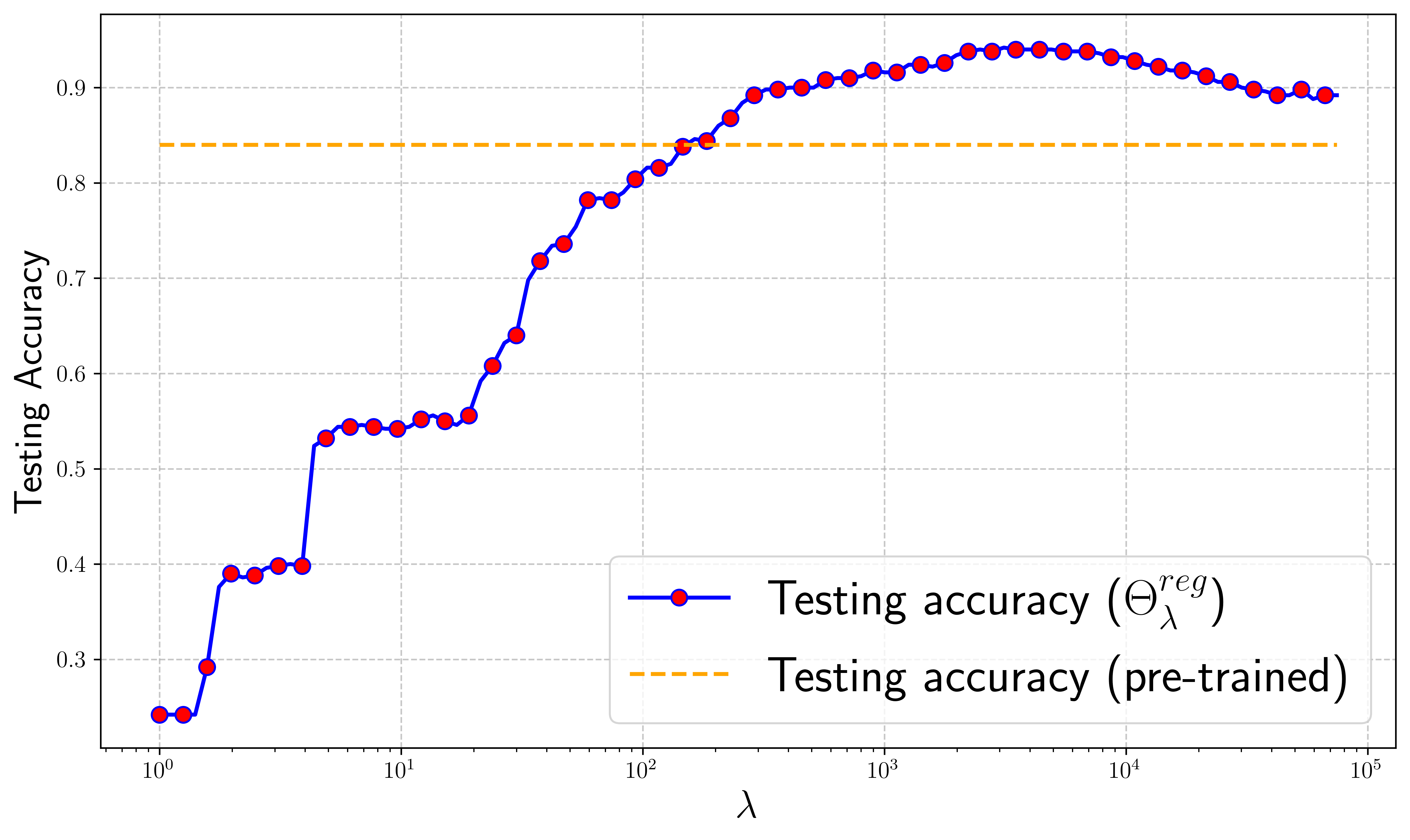}
\caption{Noise std = 0.22} 
\end{subfigure}

\vspace{0.05\textwidth}

\begin{subfigure}[t]{0.49\textwidth}
\includegraphics[width=\textwidth]{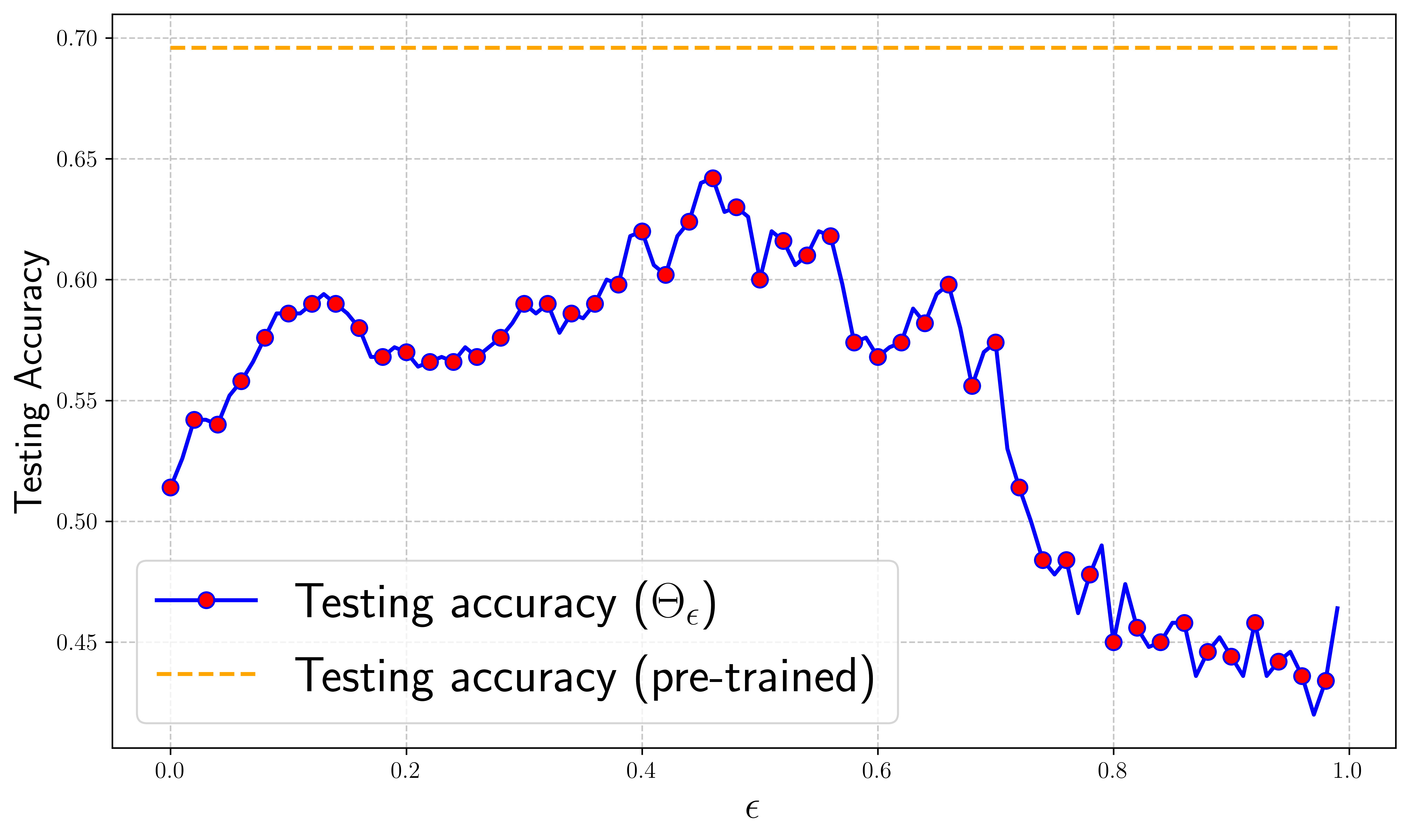}
\caption{Noise std = 0.4} 
\end{subfigure}\hspace*{\fill}
\begin{subfigure}[t]{0.49\textwidth}
\includegraphics[width=\textwidth]{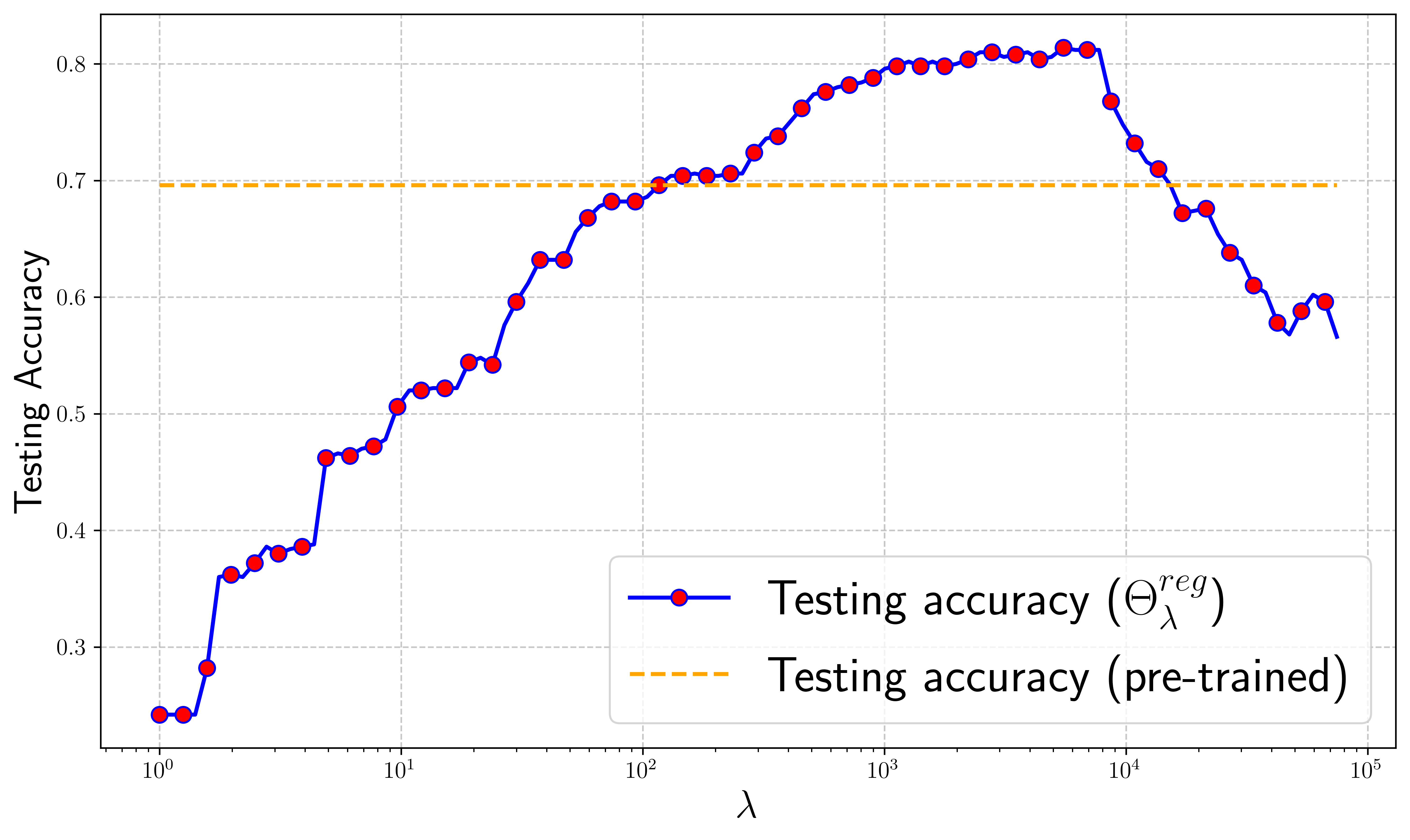}
\caption{Noise std = 0.4} 
\end{subfigure}

\caption{Testing accuracy by solutions of problem \eqref{intro_pb:NN_epsilon} (left) and problem \eqref{intro_pb:NN_reg} (right) with different hyperparameters \(\epsilon\) and \(\lambda\) in three noise scenarios.} \label{fig:accuracy}
\end{figure}

\subsection{Sparsification of solutions of a high-dimensional example}
We examine a second numerical experiment focused on a high-dimensional classification problem using data from the MNIST dataset \cite{lecun1998gradient}. Each data point in this dataset has a feature dimension of $28 \times 28$, with labels ranging from 0 to 9. To simplify the computation and examine the generalization properties, we randomly select 300 data points featuring labels 0, 1, and 2 for training. Additionally, we use 1,000 data points with the same labels for testing.

We vary the number of neurons in \eqref{eq:shallow} from 100 to 600, utilize the MSE loss, and train the models using the Adam algorithm, a commonly used variant of SGD in deep learning. The prediction results on the testing set are depicted in Figure \ref{fig:sparse} (A). Observations from the blue curve in Figure \ref{fig:sparse} (A) indicates 
an increasing trend in prediction accuracy as the number of neurons increases, aligning with the expected convergence behavior of gradient descent in overparameterized settings. We subsequently initialize Algorithm \ref{alg1} with the solution obtained from the 600-neuron shallow neural network. The outcomes, marked by red crosses in Figure \ref{fig:sparse}(A), reveal that a model with 300 activated neurons outperforms the results obtained directly using the Adam algorithm.
Additionally, we present in Figures \ref{fig:sparse} (B) and (C) the evolutions of $\|\omega^k\|_{\ell^1}$ and $\|\omega^k\|_{\ell^0}$ over the iteration number $k$. As indicated by Lemma \ref{lem:monotone}, both metrics decrease throughout the iterations. Ultimately, $\|\omega^k\|_{\ell^0}$, i.e., the number of activated neurons, stabilizes at 300, which corresponds to the number of training data points.

This numerical example, despite the limited dataset size, demonstrates good performance of the combination of a gradient descent-type algorithm for overparameterized models and our sparsification method. This approach maintains the accuracy of the trained model while reducing its size.

\begin{figure}[htbp]
    \centering
    \begin{subfigure}[b]{0.31\textwidth}
        \centering
        \includegraphics[width=\textwidth]{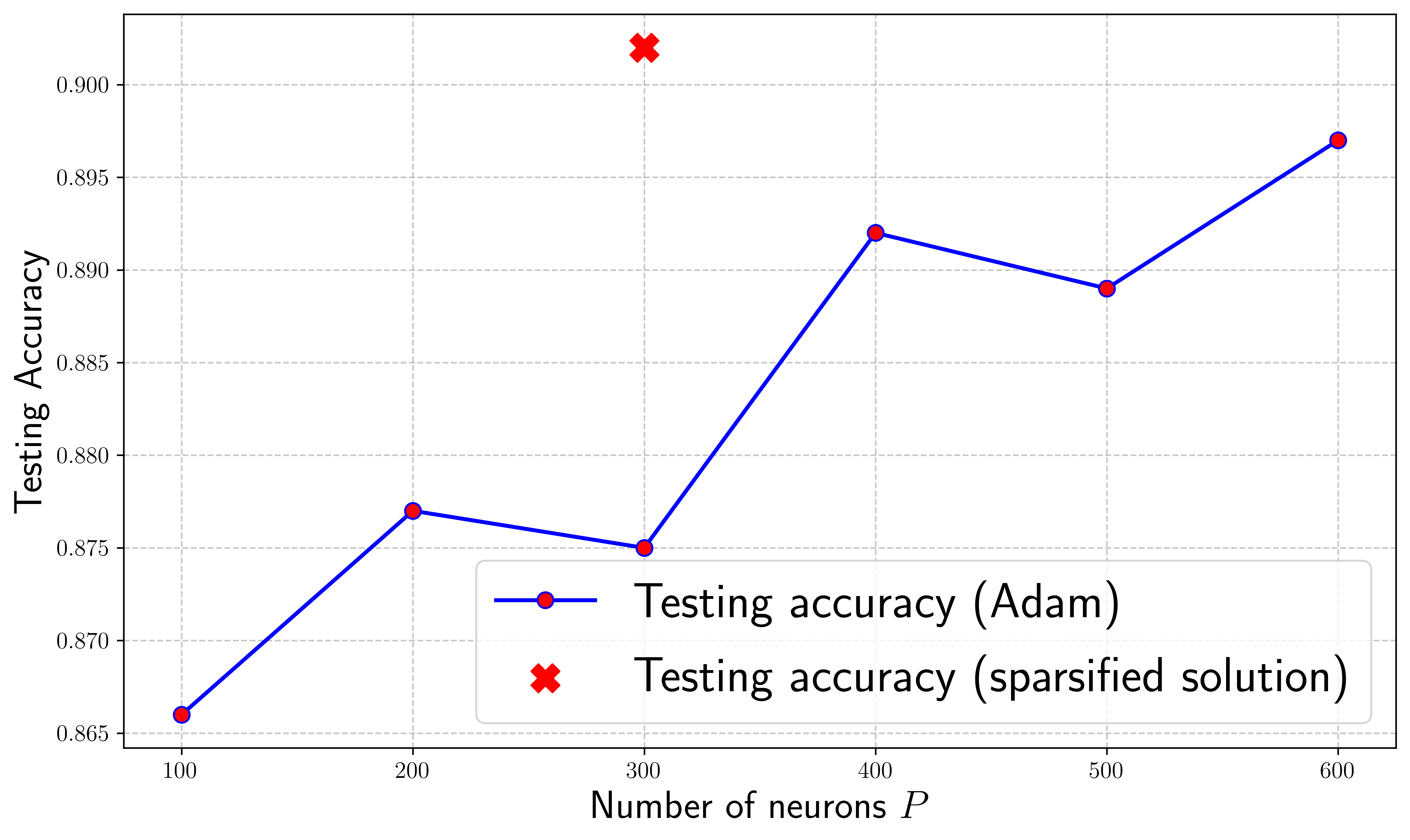}
        \caption{Testing accuracy with different number of neurons $P$.}

    \end{subfigure}
    \hspace{0.01\textwidth} 
    \begin{subfigure}[b]{0.31\textwidth}
        \centering
        \includegraphics[width=\textwidth]{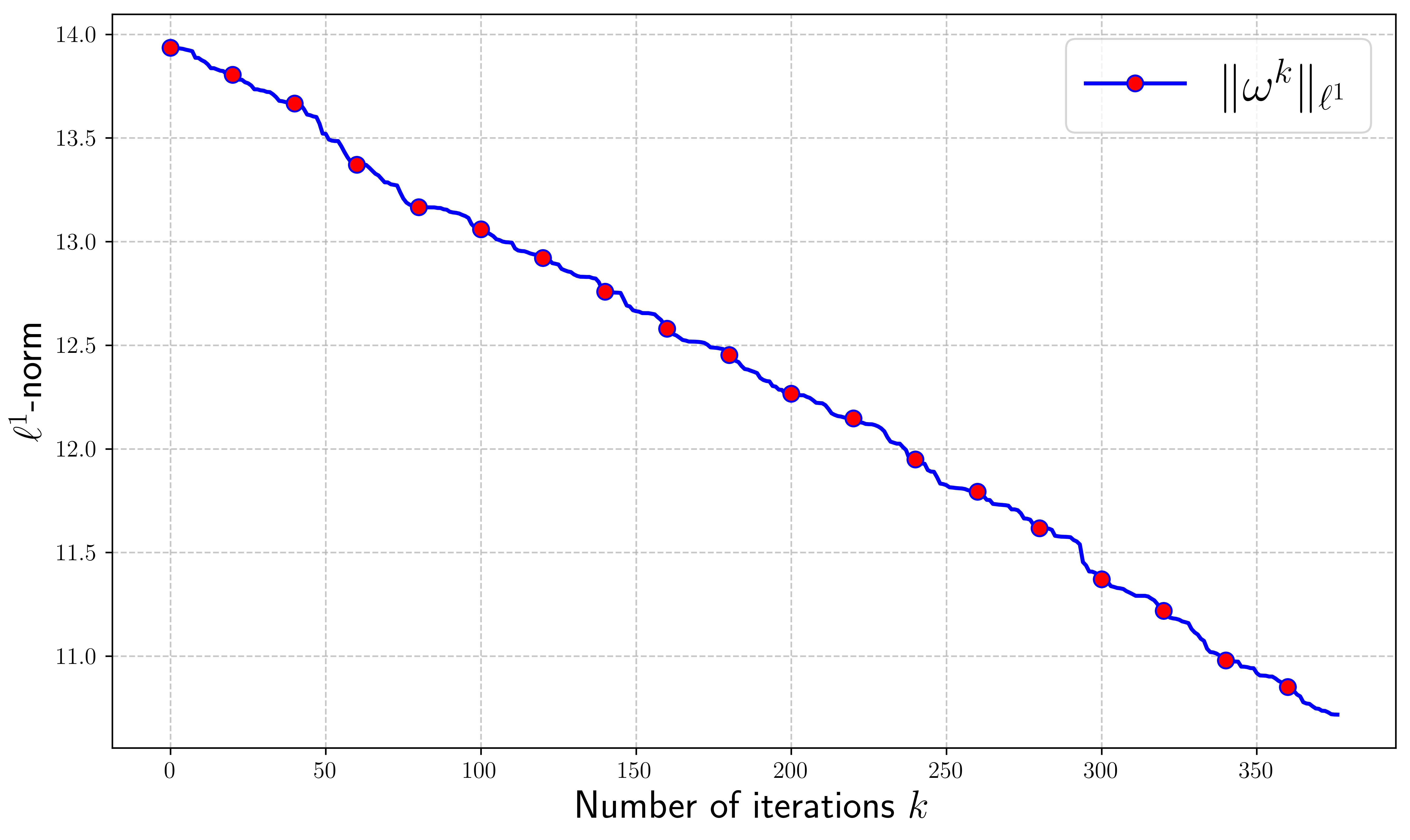}
        \caption{The $\ell^1$-norm of $\omega^k$ during iterations of Algorithm \ref{alg1}.}
    \end{subfigure}
    \hspace{0.01\textwidth} 
    \begin{subfigure}[b]{0.31\textwidth}
        \centering
        \includegraphics[width=\textwidth]{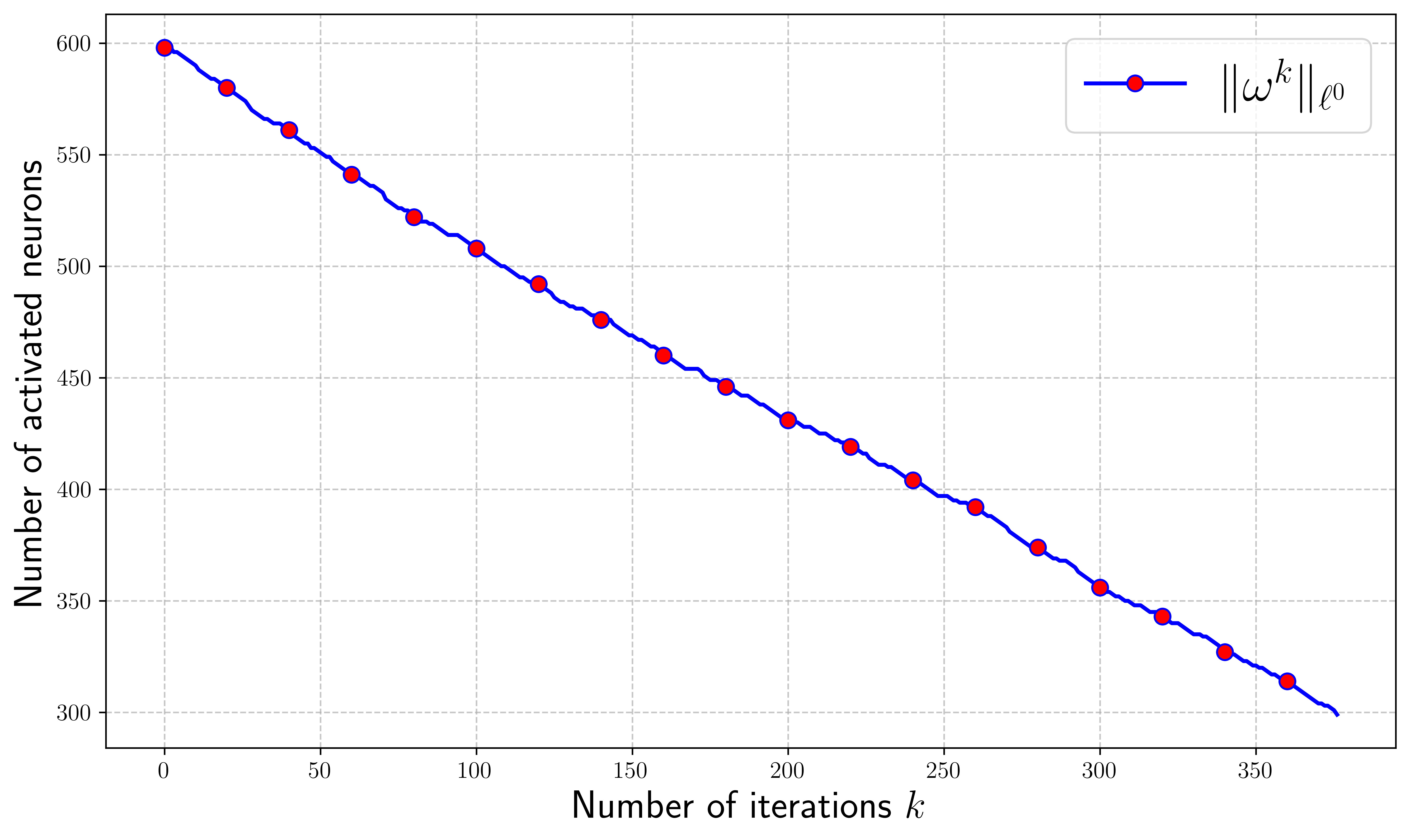}
        \caption{Number of activated neurons during iterations of Algorithm \ref{alg1}.}
    \end{subfigure}
    \caption{Numerical results of training shallow NNs on a subset of the MNIST dataset.}
    \label{fig:sparse}
\end{figure}

\section{Discussion}\label{sec:discussion}
\subsection{$L^2$ minimization problem}
Instead of considering the \(\ell^1\)-norm of \(\omega\), which promotes the sparsity of activated neurons in shallow NNs, we focus here on the \(\ell^2\)-norm. Specifically, we study the following regression problem:
\begin{equation}
    \inf_{\{(\omega_j,a_j,b_j)\in \mathbb{R} \times \Omega\}_{j=1}^P } \sum_{j=1}^P |\omega_j|^2 + \frac{\lambda}{N} \sum_{i=1}^N \left|  \sum_{j=1}^P \omega_j \sigma(\langle a_j, x_i \rangle + b_j ) - y_i \right|^2.
\end{equation}
We then consider its mean-field relaxation in \(L^2(\Omega)\). By the Riesz representation theorem in Hilbert spaces, the linear mapping \(\phi_i\) (defined in \eqref{eq:phi}) can be identified with the following functions mapping \(\Omega\) to \(\mathbb{R}\):
\begin{equation*}
    \phi_i (a,b) = \sigma(\langle a, x_i\rangle + b), \quad \text{for } i=1,\ldots, N.
\end{equation*}
As a result, the mean-field relaxation problem takes the form:
\begin{equation}\label{pb:L2}\tag{L2}
    \inf_{\mu \in L^2(\Omega)} \| \mu \|_{L^2(\Omega)}^2 +  \frac{\lambda}{N} \sum_{i=1}^N \left|\langle \phi_i, \mu\rangle_{L^2(\Omega)} - y_i \right|^2.
\end{equation}
Recall \( Y = (y_1, \ldots, y_N) \in \mathbb{R}^N \), and let \( Q \) denote the Gram matrix associated with the functions \( (\phi_i)_{i=1,\ldots, N} \):
\begin{equation*}
    Q_{i,j} = \langle \phi_i, \phi_j \rangle_{L^2(\Omega)}, \quad \text{for } i, j = 1, \ldots, N.
\end{equation*}

\begin{prop}\label{prop:L2}
    Under Assumption \ref{ass:sigma}, problem \eqref{pb:L2} admits a unique solution:
    \begin{equation}\label{eq:sol_L2_reg}
        \mu^*_{\lambda} = \sum_{i=1}^N (w_{\lambda}^*)_i \, \phi_i, \qquad \text{with} \quad w_{\lambda}^* = \left(Q + \frac{N}{\lambda} \, \textnormal{Id}\right)^{-1} Y.
    \end{equation}
    Moreover, the matrix \( Q \) is invertible. As \( \lambda \) tends to infinity, \( \mu^*_{\lambda} \) converges in the \( L^2 \)-sense to 
    \begin{equation}\label{eq:sol_L2_exact}
        \mu^* = \sum_{i=1}^N w_i^* \, \phi_i, \qquad \text{with} \quad w^* = Q^{-1} Y,
    \end{equation}
    and \( \mu^* \) is the unique solution of the exact representation counterpart of problem \eqref{pb:L2}.
\end{prop}
\begin{proof}
    The proof is presented in Section \ref{sec:proof}.
\end{proof}

The shallow NNs corresponding to the solution of \eqref{pb:L2} lie in the so-called reproducing kernel Hilbert space (RKHS) \cite[Def.\@ 2.9]{scholkopf2002learning}. In particular, the shallow NN of \eqref{eq:sol_L2_reg} can be rewritten as:
\begin{equation}\label{eq:RKHS}
   f_{\textnormal{shallow}} (x,\, \mu^*_{\lambda}) = \int_{\Omega} \sigma(\langle a, x \rangle + b)\, d\mu^*_{\lambda}(a,b) = \sum_{i=1}^N (w^*_{\lambda})_i\, k(x_i, x),
\end{equation}
where the kernel function \( k(x_i, x) \) is given by:
\begin{equation}\label{eq:kernel}
    k(x_i, x) = \int_{\Omega} \sigma(\langle a, x_i \rangle + b)\, \sigma(\langle a, x \rangle + b) \, d(a,b), \quad \text{for } i = 1, \dots, N.
\end{equation}

\begin{rem}[Comparison between TV and $L^2$ regularization]\label{rem:L1-L2}
Here, we compare the TV-problem \eqref{pb:NN_reg_rel} and the $L^2$-problem \eqref{pb:L2} in the following aspects:

\begin{enumerate}
\setlength{\itemsep}{6pt}
    \item (Structure of the solutions) 
    The solutions of \eqref{pb:NN_reg_rel} reside in the space of measures, and Theorem \ref{thm:NN_exists_0} states that their extreme points are supported on at most \( N \) points. Therefore, the associated shallow NN is a finite sum of activation functions \eqref{eq:shallow}. In contrast, the solution of \eqref{pb:L2} is a nonzero continuous function by \eqref{eq:sol_L2_reg}. Consequently, its support set has a non-empty interior, which precludes representing the corresponding shallow NN in the form of \eqref{eq:shallow}. Instead, it takes the form of a finite sum of kernel functions, as given in \eqref{eq:RKHS}.

    \item (Training complexity) The solutions of \eqref{pb:NN_reg_rel} are obtained via iterative algorithms, as described in Section \ref{sec:algo} (e.g., the simplex method and the gradient descent). In contrast, the solution of \eqref{pb:L2} has an exact formula \eqref{eq:sol_L2_reg}, provided one successfully inverts the penalized Gram matrix \( Q + N \text{Id}/\lambda \), which is analogous to solving finite element schemes for linear partial differential equations. However, inverting this matrix (or solving for \( \omega_{\lambda}^* \)) is computationally expensive when \( N \) is large and is typically handled via iterative algorithms such as the conjugate gradient method.

    \item (Prediction and privacy concerns) Once a solution of \eqref{pb:NN_reg_rel} is obtained in the form of \eqref{eq:solution_eq}, making a prediction for a new feature \( x \in \mathbb{R}^d \) is straightforward: one simply evaluates the formula \eqref{eq:shallow}. In contrast, for the solution of \eqref{pb:L2}, the prediction of this new feature \( x \in \mathbb{R}^d \) is given by \eqref{eq:RKHS}, where the kernel value \( k(x_i, x) \) must be computed through the integral \eqref{eq:kernel} for all $i$. When the dimension \( d \) is high, estimating \eqref{eq:kernel} via Monte Carlo methods incurs significant computational cost. Furthermore, since the prediction by \eqref{eq:RKHS} requires knowledge of all training features \( x_i \), this raises privacy concerns regarding the security of training data.
\end{enumerate}
\end{rem}

\subsection{Double descent for random feature problems}
The double descent phenomenon \cite{belkin2019reconciling,mei2022generalization} is studied in the context of non-penalized (or weakly penalized) regression problems, i.e., for very large values of \(\lambda\) in \eqref{intro_pb:NN_reg}. By Remark \ref{rem:exact-regression2}, as \(\lambda\) tends to infinity, the solutions of \eqref{intro_pb:NN_reg} converge to the solution of the bi-level optimization problem \eqref{pb:bilevel}.

In the analysis of the double descent, we vary the number of neurons \(P\). However, problem \eqref{pb:bilevel} is non-convex for finite values of \(P\). To address this non-convexity, we focus on its \textit{random feature} counterpart (see \cite{mei2022generalization} and \cite[Sec.~12]{bach2024learning}), where the pairs \((a_j, b_j)\) are randomly selected prior to the optimization process.

More concretely, let \((\bar{a}_j, \bar{b}_j)_{j \geq 1}\) be a sequence of random points uniformly distributed in \(\Omega\). We define:
\begin{equation*}
    \Omega_P = \{ (\bar{a}_j, \bar{b}_j) \mid 1 \leq j \leq P \}, \quad \text{for } P \geq 1.
\end{equation*}
Rather than considering the entire domain \(\Omega\), we restrict the bi-level optimization problem \eqref{pb:bilevel} to \(\Omega_P\). In this setting, the training results are given by:
\begin{equation*}
    \Theta_{P} = ((\bar{\omega}_{P})_j, \bar{a}_j, \bar{b}_j)_{j=1}^{P},
\end{equation*}
where \(\bar{\omega}_{P}\) is a solution of the following convex optimization problem (since \((\bar{a}_j, \bar{b}_j)\) are fixed):
\begin{equation}\label{pb:random_feature}
    \inf_{\bar{\omega} \in \mathbb{R}^P} \|\bar{\omega}\|_{\ell^1}, \quad \text{s.t. } \bar{\omega}  \in \argmin_{\omega \in \mathbb{R}^P} \sum_{i=1}^N \left| 
 \sum_{j=1}^P \omega_j \sigma(\langle \bar{a}_j, x_i \rangle + \bar{b}_j) - y_i \right|.
\end{equation}

Then, we evaluate the solution \(\Theta_P\) on a testing dataset \((X', Y')\). For simplicity, we assume that the training and testing sets have the same cardinality, \(N\), aligning with the framework of Remark \ref{rem:mean-lp}.
Let us define the total variation and the mean absolute error (MAE) associated with the solution $\Theta_P$: 
\begin{align*}
    \text{TV} (P) &= \|\bar{\omega}_P\|_{\ell^1},\\
    \text{MAE}_{\text{train}} (P) &= \frac{1}{N} \sum_{i=1}^N \left| f_{\textnormal{shallow}} (x_i, \Theta_P) - y_i \right|,\\
    \text{MAE}_{\text{test}} (P) &= \frac{1}{N} \sum_{i=1}^N \left| f_{\textnormal{shallow}} (x'_i, \Theta_P) - y'_i \right|.
\end{align*}
From \eqref{eq:mlp1}, we obtain the following upper bound on the testing MAE:
\begin{equation}\label{eq:MAE}
    \text{MAE}_{\text{test}} (P) \leq \underbrace{W_1(m_{\text{train}}, m_{\text{test}})}_{\text{Irreducible error from datasets}} +  \underbrace{\text{MAE}_{\text{train}} (P)}_{\text{Bias from training}} +  \underbrace{W_1(m_{\text{train}}, m_{\text{test}})\, \text{TV} (P)}_{\text{Standard deviation}}.
\end{equation}
We have the following results on the monotonicity of $\text{MAE}_{\text{train}} (P)$ and $ \text{TV} (P)$.
\begin{prop}\label{prop:double_descent}
   Let Assumption \ref{ass:sigma} hold true. Suppose that \( \Omega \) has strictly positive lower density at each point, i.e.,  
\[
\liminf_{\delta \to 0^+} \frac{\textnormal{Vol}(B(x, \delta) \cap \Omega)}{\delta^{d+1}} >0, \quad \textnormal{for all } x \in \Omega,
\]
where \( B(x, \delta) \) denotes the closed ball centered at \( x \) with radius \( \delta \).
    Consider random points $(\bar{a}_j,\bar{b}_j)_{j\geq 1}$ uniformly distributed in \(\Omega\). 
   Define $P_0$ as the smallest integer $P$ such that
    \begin{equation*}
        \min_{\omega \in \mathbb{R}^P} \sum_{i=1}^N \left| 
 \sum_{j=1}^P \omega_j \sigma(\langle \bar{a}_j, x_i \rangle + \bar{b}_j) - y_i \right| = 0.
    \end{equation*}
    Then, the following statements hold:
    \begin{enumerate}
        \item Almost surely, the number $P_0$ is finite;
        \item For \( P \geq 1 \), the function \( \textnormal{MAE}_{\textnormal{train}}(P) \) is decreasing and vanishes at \( P = P_0 \);
        \item For \( P \geq P_0 \), the function \( \textnormal{TV}(P) \) is decreasing and converges almost surely to \(\textnormal{val} \eqref{pb:NN_exact_rel} \).
    \end{enumerate}
\end{prop}
\begin{proof}
    The proof is presented in Section \ref{sec:proof}.
\end{proof}

\begin{rem}[Double descent]\label{rem:double_descent}
    The value \(P_0\) in Proposition \ref{prop:double_descent} represents the smallest number of neurons (in our sampling sequence) required to achieve an exact representation of the training set. In the literature, \(P_0\) is referred to as the interpolation threshold  \cite{belkin2019reconciling,mei2022generalization}. This threshold plays a crucial role in the double descent phenomena explained below:
    \begin{enumerate}
        \item When \(P \ll P_0\), the training error is high due to the insufficient number of parameters. In this regime, the decrease in training error (as described in point (2) of Proposition \ref{prop:double_descent}) plays a dominant role in the variation of the testing error, resulting in the initial descent of the testing error.
        \item When \(P\) increases (but remains below \(P_0\)), the decrease in training error is not sufficient to counterbalance the increase of the total variation (this increase is observed numerically in Figure \ref{fig:Double_descent}), causing the testing error curve to rise.
        \item When \(P \geq P_0\), the right-hand side of \eqref{eq:MAE} decreases (as indicated by points (2) and (3) of Proposition \ref{prop:double_descent}), which leads to the second descent in the testing error, eventually converging to a constant.
    \end{enumerate}
\end{rem}

\textit{Numerical simulation}. We perform numerical simulations of the random feature problem \eqref{pb:random_feature} with varying $P$, across three classification scenarios associated with the data in Figure \ref{fig_data}. The results are presented in Figure \ref{fig:Double_descent}. In all three simulations, we observe that MAE$_{\text{train}}(P)$ decreases as $P$ increases, consistent with point (2) of Proposition \ref{prop:double_descent}. Additionally, the interpolation threshold $P_0$ is approximately 700. It is natural that $P_0$ exceeds $N = 500$, the theoretical number required for exact representation as established in Theorem \ref{thm:NN_exists}, because the choices of $(a, b)$ are constrained within our random feature domain $\Omega_P$. Regarding the total variation, TV$(P)$, we observe an increasing trend before reaching $P_0$, followed by a decrease, which is consistent with point (3) of proposition \ref{prop:double_descent}. Consequently, as discussed in Remark \ref{rem:double_descent}, we observe the double descent phenomenon for MAE$_{\text{test}}(P)$, illustrated in the last row of Figure \ref{fig:Double_descent}.

\begin{rem}[Comparison of two regimes]
    We emphasize that while the double descent phenomenon appears in all scenarios, the quality of the two valleys (corresponding to the underparameterized and overparameterized regimes) varies significantly. In the low-noise case (i.e., when the difference between the training and testing sets is small), the second valley exhibits superior performance, indicating that overparameterization enhances results. In the moderate-noise scenario, both valleys achieve comparable predictive performance. However, in the high-noise case, the underparameterized regime significantly outperforms the overparameterized regime. This behavior is analogous to what has been recently observed in large language models. For instance, ChatGPT-4o is estimated to have 1.2 trillion parameters, while DeepSeek-V3 has 0.67 trillion parameters (with 0.03 trillion used for predictions, thanks to its \textit{Mixture of Experts} architecture). Despite the size difference, their performances are comparable for many tasks. For tasks that closely resemble existing text, ChatGPT-4o, representing the overparameterized regime, typically performs better. On the other hand, for tasks that significantly deviate from seen data, DeepSeek-V3, representing the underparameterized regime, can deliver superior results.
\end{rem}


\begin{figure}[htbp]
    \centering
    \begin{subfigure}[b]{0.31\textwidth}
        \centering
        \includegraphics[width=\textwidth]{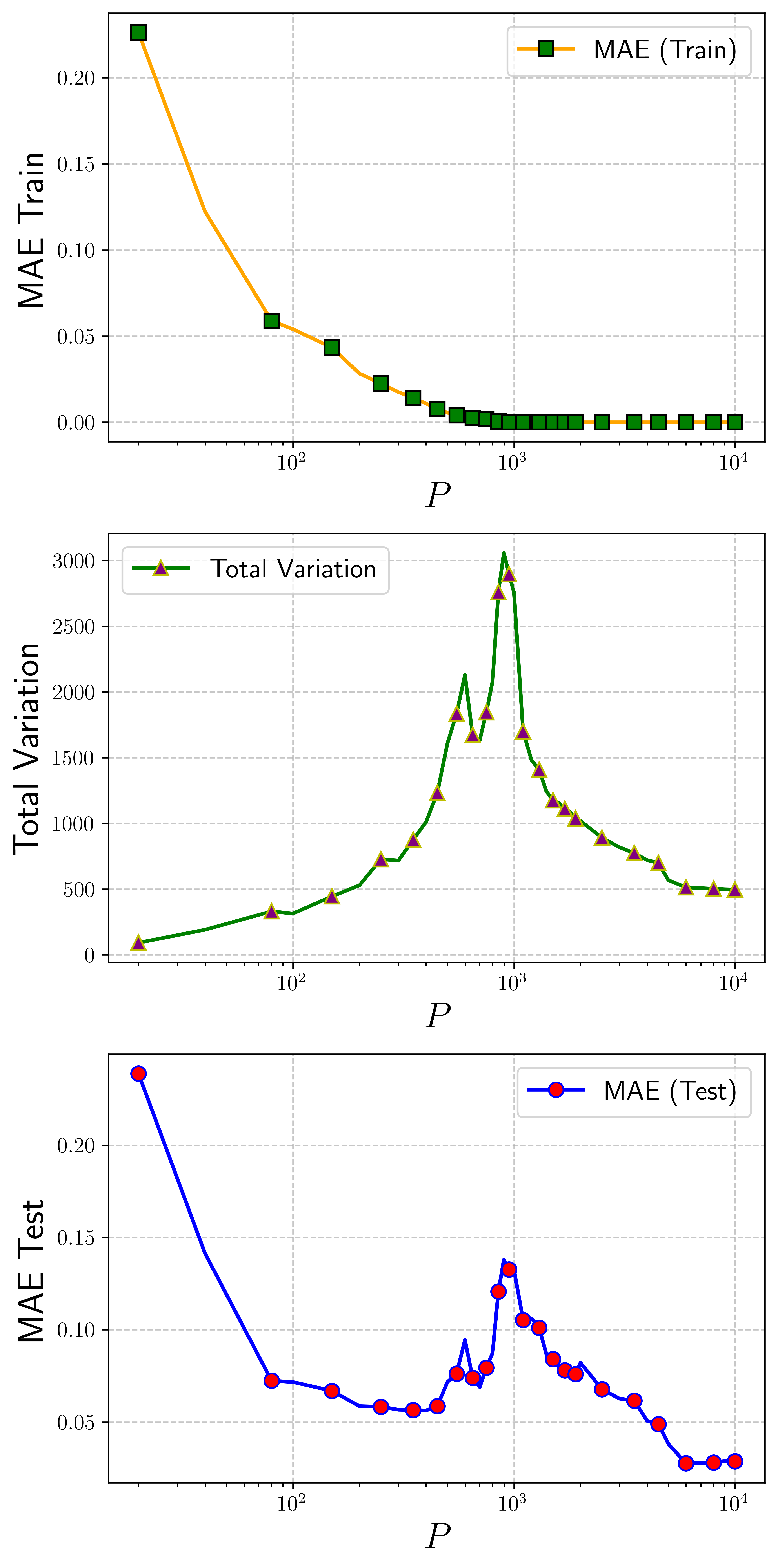}
        \caption{Noise std = 0.1.}

    \end{subfigure}
    \hspace{0.01\textwidth} 
    \begin{subfigure}[b]{0.31\textwidth}
        \centering
        \includegraphics[width=\textwidth]{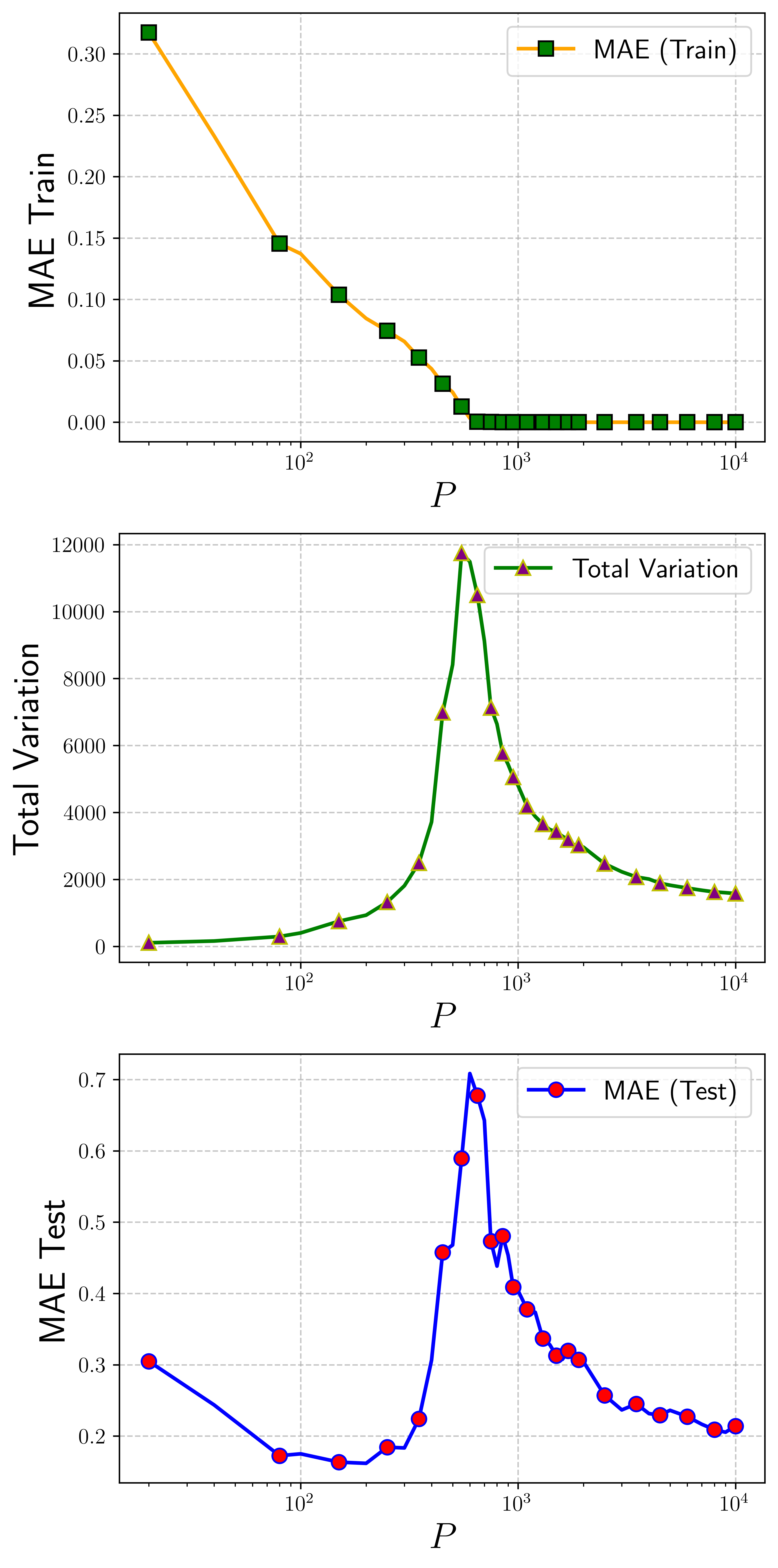}
        \caption{Noise std = 0.22.}
    \end{subfigure}
    \hspace{0.01\textwidth} 
    \begin{subfigure}[b]{0.31\textwidth}
        \centering
        \includegraphics[width=\textwidth]{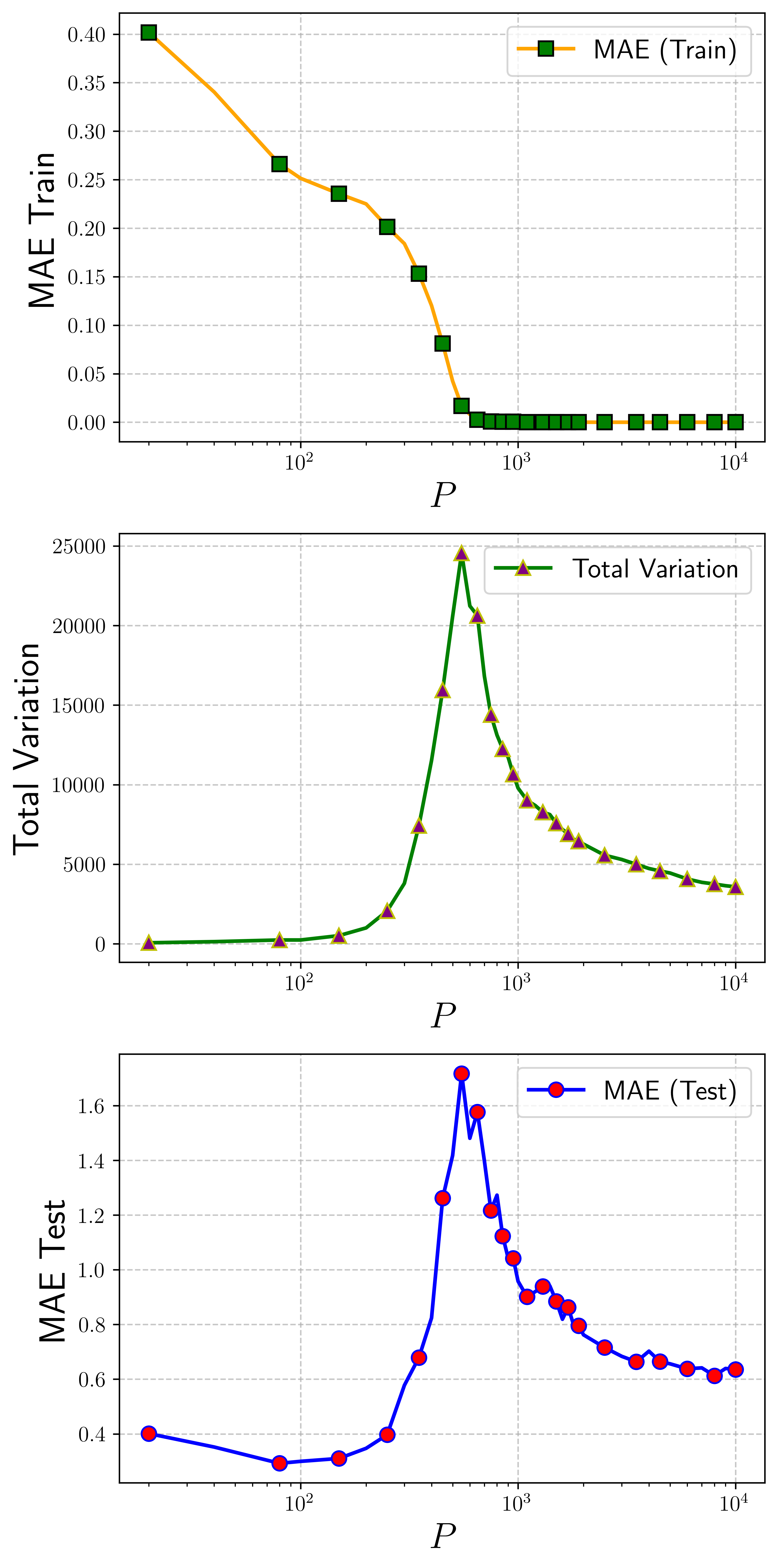}
        \caption{Noise std = 0.4.}
    \end{subfigure}
    \caption{Curves of MAE$_{\text{train}}(P)$, TV$(P)$, and MAE$_{\text{test}}(P)$ for the datasets (1000 points) in Figure \ref{fig_data}, corresponding to three noise levels: small (left), moderate (middle), and large (right). Each scenario consists of 500 points for training and 500 points for testing ($N=N'=500$).}
    \label{fig:Double_descent}
\end{figure}

\section{Technical proofs}\label{sec:proof}
\begin{proof}[Proof of Theorem \ref{thm:NN_exists}]
     Fix any dataset $\{(x_i, y_i) \in \R^{d+1}\}_{i=1}^N$ such that $x_i\neq x_j$ for $i\neq j$.
     By the formula of shallow NNs \eqref{eq:shallow}, it suffices to prove the conclusion for $P=N$ (if $P>N$, we can take $\omega_j=0$ for $j>N$). 
     Let us prove it by induction. 

\medskip
   \noindent\textbf{Step 1} (Base step).
   For $N=1$, since $\Omega$ contains a ball centered at $0$ in $\R^{d+1}$, there exists $(a_1,b_1)\in \Omega$ such that $\langle a_1, x_1 \rangle + b_1 >0$. As a consequence, $\sigma(\langle a_1, x_1 \rangle + b_1 )\neq 0$. Taking
   \begin{equation*}
       \omega_1 = \frac{y_1}{\sigma(\langle a_1, x_1 \rangle + b_1 )},
   \end{equation*}
   it is easy to verify that $ (\omega_1,a_1,b_1)$
   is a solution for the base step.

\medskip
\noindent\textbf{Step 2} (Induction step).
   Let us assume that the conclusion holds for some $N \geq 1$. Then, consider a new data point $(x_{N+1},y_{N+1})\in \R^{d+1}$ such that $x_{N+1}\neq x_i$ for $i=1,\ldots, N$. 
   By Milman's Theorem \cite[Thm.\@ 3.25]{rudin91functional}, we have that $ \text{Ext} (\conv (\{x_1,\ldots, x_{N+1}\}))\subseteq  \{x_1,\ldots,x_{N+1}\}$. Let $x_{i^{*}} \in  \text{Ext} (\conv (\{x_1,\ldots, x_{N+1}\})) $ with $i^{*}\in \{1,\ldots, N+1\}$. It follows that 
      \begin{equation}\label{eq:conv}
          x_{i^{*}} \notin \conv (\{x_1, \ldots, x_{i^{*}-1},x_{i^{*}+1},\ldots, x_{N+1} \}).
      \end{equation}
   By assumption, there exists $(\omega_j,a_j,b_j)\in \R\times \Omega$ with $j\neq i^{*}$ such that 
    \begin{equation}\label{eq:control_N}
        \sum_{j\neq i^{*}} \omega_{j} \sigma(\langle  a_j, x_i \rangle + b_j) = y_i, \quad \text{for }i\neq i^{*}.
    \end{equation}
    Since $\Omega$ contains a ball centered at $0$ in $\R^{d+1}$, by the Hahn-Banach theorem, we deduce from \eqref{eq:conv} that there exists $(a_{i^{*}},b_{i^*})\in \Omega$ such that
   \begin{align*}
       \langle a_{i^*} , x_i \rangle + b_{i^*} & < 0, \quad i \neq i^{*};\\
      \langle a_{i^*} , x_{i^{*}} \rangle + b_{i^*} & > 0.
   \end{align*}
   It follows from Assumption \ref{ass:sigma}(1) that
   \begin{align}
         \sigma (\langle a_{i^*} , x_i \rangle -b_{i^*} ) &=0 , \quad i \neq i^{*}; \label{eq:Hahn-Banach1} \\
         \sigma (\langle a_{i^*} , x_{i^{*}} \rangle + b_{i^*}) &> 0. \label{eq:Hahn-Banach2}
   \end{align}
   Next, we take 
   \begin{equation}\label{eq:omega_new}
       \omega_{i^{*}} = \frac{y_{i^{*}} -  \sum_{j\neq i^{*}} \omega_{j} \sigma(\langle  a_j, x_{i^{*}} \rangle + b_j)}{\sigma (\langle a_{i^*} , x_{i^{*}} \rangle + b_{i^*})}.
   \end{equation}
   We deduce from \eqref{eq:control_N}-\eqref{eq:omega_new} that
   \begin{equation*}
        \sum_{j=1}^{N+1} \omega_{j} \sigma(\langle  a_j, x_i \rangle + b_j) = y_i, \quad \text{for }i = 1,\ldots, N+1.
    \end{equation*}
   The conclusion follows by induction.
\end{proof}

\begin{proof}[Proof of Theorem \ref{thm:generalization}]
\textbf{Step 1} (Decomposition).
     By the triangle inequality, we have
    \begin{equation*}
        d_{\textnormal{KR}}(m_{\text{test}}, m_{\text{pred}}) \leq d_{\textnormal{KR}}(m_{\text{test}}, m_{\text{train}}) + d_{\textnormal{KR}}(m_{\text{train}}, m_{\text{auxi}}) + d_{\textnormal{KR}}(m_{\text{auxi}}, m_{\text{pred}}),
    \end{equation*}
    where $m_{\text{auxi}}$ is an auxiliary distribution defined as:
    \begin{equation*}
        m_{\text{auxi}} = \frac{1}{N}\sum_{i=1}^N \delta_{(x_i,f_{\text{shallow}}(x_i, \Theta))}.
    \end{equation*}
    The term $d_{\textnormal{KR}}(m_{\text{test}}, m_{\text{train}})$ is independent of $\Theta$. Let us estimate the other two terms in the following steps. 
    
    \medskip
    \noindent\textbf{Step 2} (Estimate on $d_{\textnormal{KR}}(m_{\text{train}}, m_{\text{auxi}})$).
    By the definition of $d_{\textnormal{KR}}$, we have 
    \begin{equation*}
        d_{\textnormal{KR}}(m_{\text{train}}, m_{\text{auxi}}) = \sup_{F\in \text{Lip}_1(\R^{d+1}) } \frac{1}{N} \sum_{i=1}^N \left( F(x_i,y_i) - F(x_i, f_{\text{shallow}} (x_i, \Theta)) \right),
    \end{equation*}
    Since $F$ is $1$-Lipschitz, it follows that
   \begin{equation*}
        d_{\textnormal{KR}}(m_{\text{train}}, m_{\text{auxi}}) \leq \frac{1}{N} \sum_{i=1}^N |y_i- f_{\text{shallow}} (x_i, \Theta)|.
    \end{equation*}

    \medskip
    \noindent\textbf{Step 3} (Estimate on $d_{\textnormal{KR}}(m_{\text{auxi}}, m_{\text{pred}})$).
     By their definitions, $ m_{\text{auxi}}$ and $ m_{\text{pred}}$ can be re-written as:
    \begin{equation*}
        m_{\text{auxi}} = (id,\, f_{\text{shallow}} (\cdot, \Theta))\# m_X, \qquad  m_{\text{pred}} = (id,\, f_{\text{shallow}} (\cdot, \Theta))\# m_{X'},
    \end{equation*}
    where $id$ is the identity mapping in $\R^d$ and $\#$ denotes the push-forward operator. For any function $F \in \text{Lip}_1(\R^{d+1})  $, we obtain that
    \begin{align*}
        \int_{\R^{d+1}} F \, d  (m_{\text{auxi}} - m_{\text{pred}}) & = 
        \int_{\R^d} F\circ  (id,\,f_{\text{shallow}} (\cdot, \Theta)) \, d  (m_X - m_{X'} ) 
        \\
        & \leq \text{Lip}\left(F\circ  (id,\,f_{\text{shallow}} (\cdot, \Theta)) \right)\, d_{\textnormal{KR}} (m_X, m_{X'}) \\
        & \leq (1+\text{Lip}(  f_{\text{shallow}} (\cdot, \Theta)))\, d_{\textnormal{KR}} (m_X, m_{X'}),
    \end{align*}
    where the notation $\operatorname{Lip}(h(\cdot))$ denotes the Lipschitz constant of the function $h(\cdot)$. Since $F$ is arbitary, we deduce that 
    \begin{equation*}
        d_{\textnormal{KR}}(m_{\text{auxi}}, m_{\text{pred}}) \leq (1+\text{Lip}(  f_{\text{shallow}} (\cdot, \Theta)))\, d_{\textnormal{KR}} (m_X, m_{X'}).
    \end{equation*}
Recall the formula of $f_{\text{shallow}} (\cdot, \Theta)$ from \eqref{eq:shallow}. We obtain that for any $x_1,x_2\in \R^d$,
\begin{equation*}
\begin{split}
     |f_{\text{shallow}} (x_1, \Theta) - f_{\text{shallow}} (x_2, \Theta)| & \leq \sum_{j=1}^P | \omega_j  | \, |\sigma(\langle a_j,x_1\rangle + b_j) - \sigma(\langle a_j,x_2\rangle + b_j)| \\ 
    & \leq  L \|x_1-x_2\| \sum_{j=1}^P | \omega_j  | \|  a_j \|,
\end{split}
\end{equation*}
where the second line is from the Lipschitz continuity of $\sigma$. The conclusion follows.
\end{proof}

\begin{proof}[Proof of Lemma \ref{lm:non-empty}]
    It suffices to prove that the admissible set of problem \eqref{pb:NN_M_eq} is non-empty. By Theorem \ref{thm:NN_exists_0}, the following problem,
    \begin{equation}\label{pb:NN_epsilon_inter}
     \inf_{\mu\in \mathcal{M}(\Omega)} \|\mu\|_{\text{TV}}, \quad \text{s.t. }  \|\phi \, \mu - Y\|_{\ell^{\infty}} \leq \epsilon - V_0 LD_X h(\Omega_M),
\end{equation}
    has a solution in the form $\mu^{*} = \sum_{j=1}^N \omega^{*}_j \delta_{(a^{*}_j,b^*_j)}$. By the definition of $h(\Omega_M)$, there exists $\{(\bar{a}_j,\bar{b}_j) \in \Omega_{M}\}_{j=1}^N$ such that 
    \begin{equation*}
        \| (\bar{a}_j,\bar{b}_j) - (a^{*}_j,b^*_j)\| \leq h(\Omega_M), \quad \text{for } j=1,\ldots, N.
    \end{equation*}
As a consequence,
    \begin{equation*}
       | \sigma(\langle \bar{a}_j, x_i\rangle + \bar{b}_j) - \sigma(\langle {a}_j^*, x_i\rangle + {b}_j^*) | \leq L D_X h(\Omega_M), \quad \text{for }j=1,\ldots, N;\; i= 1,\ldots, N.
    \end{equation*}
    Let us define $ \bar{\mu} = \sum_{j=1}^N \omega^{*}_{j}\delta_{( \bar{a}_j,\bar{b}_j)}$. Then, for any $i =1,\ldots, N$,
\begin{equation*}
    |\phi_i\, ( \bar{\mu} - \mu^{*})|  = \left| \sum_{j=1}^N \omega_j^{*} ( \sigma(\langle \bar{a}_j, x_i\rangle + \bar{b}_j) - \sigma(\langle a_j^*, x_i\rangle + b_j^*)) \right| \leq L D_X h(\Omega_M)\,  \text{val}\eqref{pb:NN_epsilon_inter} \leq  V_0 L D_X h(\Omega_M). 
    \end{equation*}
     Since $\text{supp}(\bar{\mu}) \subseteq \Omega_M$, recalling the definition of A from \eqref{eq:A}, there exists $\bar{\omega}\in \R^M$ such that $A\, \bar{\omega} = \phi\, \bar{\mu}$. Applying the triangle inequality, we obtain that
    \begin{equation*}
        \|A \, \bar{\omega} - Y \|_{\ell^{\infty}} = \| \phi\, \bar{\mu} - Y \|_{\ell^{\infty}} \leq \|\phi\, \mu^{*} - Y \|_{\ell^{\infty}} + \|\phi\, (\bar{\mu} - \mu^{*}  ) \|_{\ell^{\infty}} \leq \epsilon.
    \end{equation*}
    The conclusion follows.
\end{proof}

\begin{proof}[Proof of Theorem \ref{thm:discrete}]
\textbf{Step 1} (Proof of \eqref{eq:regression_dis}). The left side of \eqref{eq:regression_dis} is obvious by the definition of \eqref{pb:NN_reg_M_eq}. By Theorem \ref{thm:NN_exists_0}, let the empirical measure $\mu^{*} = \sum_{j=1}^N \omega^{*}_j \delta_{(a^{*}_j,b^*_j)}$  be a solution of \eqref{pb:NN_reg_rel}. Take $\{(\bar{a}_j,\bar{b}_j) \in \Omega_{M}\}_{j=1}^N$ such that $ \| (\bar{a}_j,\bar{b}_j) - (a^{*}_j,b^*_j)\| \leq h(\Omega_M)$ for $j=1,\ldots, N$, and let $ \bar{\mu} = \sum_{j=1}^N \omega^{*}_{j}\delta_{( \bar{a}_j,\bar{b}_j)}$. Recalling the objective function of \eqref{pb:NN_reg_rel}, by the triangle inequality, we have
\begin{equation*}
\begin{split}
     \|\bar{\mu}\|_{\text{TV}} + \frac{\lambda}{N} \left\| \phi\, \bar{\mu} - Y\right\|_{\ell^1} & \leq \|\bar{\mu}\|_{\text{TV}} + \frac{\lambda}{N} \left\| \phi\, \mu^{*} -Y  \right\|_{\ell^1} + \frac{\lambda}{N} \left\| \phi\, (\bar{\mu}-\mu^{*}) \right\|_{\ell^1}\\
     & = \text{val}\eqref{pb:NN_reg_rel} + \frac{\lambda}{N} \left\| \phi\, (\bar{\mu}-\mu^{*}) \right\|_{\ell^1},
\end{split}
\end{equation*}
where the second line is by definitions of $\bar{\mu}$ and $\mu^{*}$. Moreover, we have
\begin{equation*}
\begin{split}
     \left\| \phi\, (\bar{\mu}-\mu^{*}) \right\|_{\ell^1} &= \sum_{i=1}^N \left| \sum_{j=1}^N \omega_j^{*} (\sigma(\langle \bar{a}_j, x_i\rangle + \bar{b}_j) - \sigma(\langle {a}_j^*, x_i\rangle + {b}_j^*)) \right| \leq N  L D_X h(\Omega_M) \|\omega^{*}\|_{\ell^1}.
\end{split}
\end{equation*}
    Combining with the fact that $ \|\omega^{*}\|_{\ell^1} = \|\mu^{*}\|_{\text{TV}} \leq \text{val}\eqref{pb:NN_reg_rel}$, we obtain that
    \begin{equation*}
          \|\bar{\mu}\|_{\text{TV}} + \frac{\lambda}{N} \left\| \phi\, \bar{\mu} - Y \right\|_{\ell^1} \leq (1+ \lambda L D_X h(\Omega_M) ) \text{val}\eqref{pb:NN_reg_rel}.
    \end{equation*}
    Since $\text{supp}(\bar{\mu}) \subseteq \Omega_M$, the right side of \eqref{eq:regression_dis} follows.

\medskip
\noindent\textbf{Step 2} (Proof of \eqref{eq:epsilon_dis}).
The left side of \eqref{eq:epsilon_dis} is obvious. 
By the proof of Lemma \ref{lm:non-empty}, we have
\begin{equation*}
    \text{val}\eqref{pb:NN_M_eq} \leq \text{val}\eqref{pb:NN_epsilon_inter}.
\end{equation*}
Let $V\colon \R_{+} \to \R$, $\epsilon \mapsto \textnormal{val}\eqref{intro_pb:NN_epsilon}$. By Lemma \ref{lem:danskin}, $V$ is convex. Therefore,
\begin{equation*}
    \text{val}\eqref{pb:NN_epsilon_inter} = V(\epsilon- V_0 LD_X h(\Omega_M)) \leq V(\epsilon) - V_0 LD_X h(\Omega_M) \, p,\quad \text{with }\, p \in \partial V(\epsilon - V_0 LD_X h(\Omega_M)).
\end{equation*}
By \eqref{eq:right_derivative}, we have $V'(0^{+}) = - c_0$. The convexity of $V$ implies that $p\geq -c_0$. Substituting it to the inequality above, we obtain the conclusion.
\end{proof}

\begin{proof}[Proof of Lemma \ref{lm:lp}]
    Points (1) and (2) can be deduced from the definition of linear programming problems \eqref{pb:lp} and \eqref{pb:lp2} by a contradiction argument. 
 
 Let us prove \eqref{eq:extreme_lp}.
     Suppose that \((u^+, u^-)\) is a basic feasible solution of \eqref{pb:lp} and \(\| u^+ - u^- \|_{\ell^0} > N\). Let \(\omega = u^+ - u^-\). We first prove by contradiction that 
\begin{equation}\label{eq:u_w}
    u^+ = (\omega)_{+}, \quad u^- = (\omega)_{-}.
\end{equation}
Assume the previous equality does not hold. Then, there exists some \(j \in \{1, \ldots, M\}\) such that \(u^+_j > 0\) and \(u^-_j > 0\). In this case, we obtain a strictly better solution with respect to \(u^+, u^-\) by replacing \(u^+_j\) and \(u^-_j\) with \(u^+_j - \min\{u^+_j, u^-_j\}\) and \(u^-_j - \min\{u^+_j, u^-_j\}\), respectively. Thus, \eqref{eq:u_w} holds.

Since \(\|\omega\|_{\ell^0} > N\), by the proof of Lemma \ref{lem:monotone}, there exists \(0 \neq \eta \in \mathbb{R}^M\) such that \(A \, \eta = 0\) and \(|\eta_j| \leq |\omega_j|\) for all \(j = 1, \ldots, M\). Consequently, \(((\omega + \eta)_+, (\omega + \eta)_-)\) and \(((\omega - \eta)_+, (\omega - \eta)_-)\) are in the feasible set of \eqref{pb:lp}. Moreover, \(\omega + \eta\) and \(\omega - \eta\) have the same sign as \(\omega\) for each coordinate. Applying \eqref{eq:u_w}, we have
\[
    u^+ = (\omega)_+ = \frac{(\omega + \eta)_+ + (\omega - \eta)_+}{2}, \quad 
    u^- = (\omega)_- = \frac{(\omega + \eta)_- + (\omega - \eta)_-}{2}.
\]
Since \(\eta \neq 0\), this contradicts the assumption that \((u^+, u^-)\) is a basic feasible solution. Hence, \(\|u^+ - u^-\|_{\ell^0} \leq N\). The inequality \(\|v^+ - v^-\|_{\ell^0} \leq N\) follows from the same proof.
\end{proof}

\begin{proof}[Proof of Lemma \ref{lem:monotone}]
Suppose that $\Theta^k = (\omega^k,a^k,b^k)$ is not the output of Algorithm \ref{alg1}.
   By \eqref{eq:alg_sparse_1}-\eqref{eq:alg_sparse_2} in Algorithm \ref{alg1}, we obtain $\hat{\omega}^k$ and $\bar{\omega}^k$, and it follows that 
   \begin{equation*}
       A^k \hat{\omega}^k = A^k \bar{\omega}^k = A^k {\omega}^k.
   \end{equation*}
   Updating $\Theta^{k+1}$ by \eqref{eq:alg_sparse_3}, we have
   \begin{equation*}
        f_{\textnormal{shallow}}(x_i, \Theta^{k+1}) = f_{\textnormal{shallow}}(x_i, \Theta^k),\quad \text{for }i=1,\ldots, N; \quad \|\omega^{k+1}\|_{\ell^1} \leq \|\omega^{k}\|_{\ell^1}.
   \end{equation*}
   Recalling $\alpha$ and $\beta$ from \eqref{eq:alg_sparse_1}, it follows that
   \begin{equation*}
       1\pm \beta \alpha_j \geq 0,\quad \text{for } j=1,\ldots, p^k.
   \end{equation*}
   This implies that
    \begin{align*}
      \|\hat{\omega}^{k+1}\|_{\ell^1} = \sum_{j=1}^{p_k} |\hat{\omega}^{k+1}_j| = \sum_{j=1}^{p_k} (1+\beta \alpha_j)|\omega^{k}_j| = \|\omega^k\|_{\ell^1} + \beta \sum_{j=1}^{p_k} \alpha_j |\omega_j^k|.\\
       \|\bar{\omega}^{k+1}\|_{\ell^1} = \sum_{j=1}^{p_k}|\bar{\omega}^{k+1}_j| = \sum_{j=1}^{p_k} (1-\beta \alpha_j)|\omega^{k}_j| = \|\omega^k\|_{\ell^1} - \beta \sum_{j=1}^{p_k} \alpha_j |\omega_j^k|.
  \end{align*}
  It follows that $\|\omega^k\|_{\ell^1} = (\|\hat{\omega}^k\|_{\ell^1}+\|\bar{\omega}^k\|_{\ell^1})/2$.
  If $\|\omega^{k+1}\|_{\ell^1} = \|\omega^k\|_{\ell^1}$, by \eqref{eq:alg_sparse_3}, we have $\|\omega^k\|_{\ell^1} = \|\hat{\omega}^k\|_{\ell^1}= \|\bar{\omega}^k\|_{\ell^1}$. Again, from \eqref{eq:alg_sparse_3}, it follows that $\omega^{k+1} = \hat{\omega}^k$. Given the construction of $\hat{\omega}^k$, it is evident that $\|\hat{\omega}^k\|_{\ell^0} \leq \|\omega^k\|_{\ell^0} - 1$. Thus, the conclusion is drawn.
\end{proof}

\begin{proof}[Proof of Proposition \ref{prop:L2}]
\textbf{Step 1} (Linear independence of \(\{\phi_i\}_{i=1}^N\)).  Suppose, for contradiction, that $\{\phi_i\}_{i=1}^N$ are linearly dependent. Then there exists $(\alpha_1, \ldots, \alpha_N) \neq 0$ such that
\[
\sum_{i=1}^N \alpha_i \, \phi_i(a,b) = 0 \quad \text{for all } (a,b) \in \Omega.
\]
For any $(\omega_j, a_j, b_j) \in \mathbb{R} \times \Omega$, $j=1, \ldots, N$, it follows that
\[
\sum_{j=1}^N \omega_j \sum_{i=1}^N \alpha_i \, \phi_i(a_j, b_j) = \sum_{i=1}^N \alpha_i \sum_{j=1}^N \omega_j \, \phi_i(a_j, b_j) = 0.
\]
By Theorem \ref{thm:NN_exists}, setting $y_i = \alpha_i$ for all $i$, there exists $(\omega_j, a_j, b_j)$ such that
\[
\sum_{j=1}^N \omega_j \, \phi_i(a_j, b_j) = \alpha_i \quad \text{for all } i = 1, \ldots, N.
\]
Substituting into the previous equation yields $\|\alpha\|^2 = 0$, contradicting $\alpha \neq 0$. Thus, $\{\phi_i\}_{i=1}^N$ are linearly independent. As a consequence, the Gram matrix $Q$ of $\{\phi_i\}_{i=1}^N$ is strictly positive definite.

\medskip
\noindent\textbf{Step 2} (Proof of \eqref{eq:sol_L2_reg}). The first-order optimality condition for \eqref{pb:L2} is given by
\[
\mu = -\frac{\lambda}{N} \sum_{i=1}^N \left( \langle \phi_i, \mu \rangle_{L^2} - y_i \right) \phi_i.
\]
Assuming \(\mu = \sum_{i=1}^N \omega_i \phi_i\) and substituting into the above equation, we obtain
\[
\omega_i = -\frac{\lambda}{N} \sum_{j=1}^N \left( \langle \phi_i, \phi_j \rangle_{L^2} \, \omega_j - y_i \right), \quad \text{for all } i,
\]
where we have used the linear independence of \(\{\phi_i\}_{i=1}^N\). Recalling the definition of the Gram matrix \(Q\), equation \eqref{eq:sol_L2_reg} follows immediately.

\medskip
\noindent\textbf{Step 3} (Proof of \eqref{eq:sol_L2_exact}). Since $Q$ is invertible (proved in step 1), the convergence of $\mu^*_{\lambda}$ to $\mu^*$ is straightforward. The Lagrangian associated with the exact representation counterpart of problem \eqref{pb:L2} is given by:
for \( (\mu, p) \in L^2(\Omega) \times \mathbb{R}^N \),
\begin{equation*}
    \mathcal{L}(\mu, p) = \|\mu\|_{L^2(\Omega)}^2 + \sum_{i=1}^N p_i \big( \langle \phi_i, \mu \rangle_{L^2(\Omega)} - y_i \big).
\end{equation*}
The optimality condition yields:
\begin{equation*}
    \begin{cases}
        \frac{\partial \mathcal{L}}{\partial \mu}(\mu, p) = 0, \\[0.6em]
        \frac{\partial \mathcal{L}}{\partial p}(\mu, p) = 0,
    \end{cases} \qquad \Rightarrow \qquad  \begin{cases}
        \mu = -\frac{1}{2} \sum_{i=1}^N p_i \phi_i, \\[0.6em]
        \langle \phi_i, \mu \rangle_{L^2(\Omega)} = y_i, \quad \text{for } i = 1, \dots, N.
    \end{cases}
\end{equation*}
Substituting the expression for \( \mu \) into the second equation gives:
\begin{equation*}
    -\frac{1}{2} Q p = Y.
\end{equation*}
Since \( Q \) is invertible, we solve for \( p \):
\begin{equation*}
    p = -2 Q^{-1} Y.
\end{equation*}
Substituting \( p \) back into the expression for \( \mu\), we deduce that $\mu = \mu^*$ satisfies the optimality condition. By the convexity, $\mu^*$ is the solution of the exact representation problem. 
\end{proof}

\begin{proof}[Proof of Proposition \ref{prop:double_descent}]
Since \(\{x_i\}_{i=1}^{N}\) are disjoint, we can prove by induction that there exist \(\{(a_i^*, b_i^*) \in \Omega\}_{i=1}^{N}\) and a permutation \(s\) of \(\{1, \ldots, N\}\) such that
\[
\langle a_{s(j)}^*, x_{s(i)} \rangle + b_{s(j)}^* < 0 \quad \text{if } j < i, \quad \text{and} \quad \langle a_{s(j)}^*, x_{s(i)} \rangle + b_{s(j)}^* > 0 \quad \text{if } j \geq i.
\]
Moreover, for any \(\{(a_i, b_i) \in \Omega\}_{i=1}^{N}\) satisfies the previous inequalities (possibly with different $s$), the functions \( \{\sigma(\langle a_i, x \rangle + b_i)\}_{i=1}^N \) can exactly interpolate the training data by appropriately choosing the weight \( \omega \).  

Since \( \langle a, x \rangle + b \) is continuous with respect to \( (a, b) \), it follows that a sufficiently small perturbation of \( (a, b) \) does not change the sign of \(\langle  a, x \rangle + b\) if it is non-zero. In this context, there exists $\delta>0$ such that if
\begin{equation}\label{eq:P0}
    \{(\bar{a}_j, \bar{b}_j)\}_{j=1}^{P} \cap B((a_i^*, b_i^*), \delta) \neq \emptyset \quad \text{for all } i = 1, \ldots, N,
\end{equation}
then \(\{\sigma(\langle \bar{a}_j,x\rangle +\bar{b}_j)\}_{j=1}^P\) can exactly interpolate the data points \((x_i, y_i)\) for all \(i\).

For any $i=1,\ldots,N$ and any $j\geq 1$, let \(A_{j}^i\) denote the following event:
\[
A_{j}^i = \left\{ (\bar{a}_j, \bar{b}_j) \in B((a_i^*, b_i^*), \delta) \right\}.
\]
Since the lower density at the point $(a_i^*, b_i^*)$ is strictly positive, then there exists $c>0$ such that the probability of $A_j^i$ satisfies \(\mathbb{P}(A_j^i) \geq c\, \delta^d\). Consequently,
\[
\sum_{j \geq 1} \mathbb{P}(A_j^i) = \infty.
\]
Since the events \(\{A_j^i\}_{j \geq 1}\) are independent, the second Borel-Cantelli Lemma implies that, almost surely, there exists some \(j\) such that \(A_j^i\) occurs. As \(i\) ranges over the finite set \(\{1, \ldots, N\}\), we conclude that there exists, almost surely, a threshold \(P_0\) such that condition \eqref{eq:P0} holds for all \(P \geq P_0\). Point (1) follows.

Point (2) is straightforward since the set \( \Omega_P \) is increasing, and we have an exact representation of the training set when \( P \geq P_0 \).

For point (3), when \( P \geq P_0 \), problem \eqref{pb:random_feature} is equivalent to \eqref{pb:NN_exact_rel}, with \( \Omega \) replaced by \( \Omega_P \). As \( \Omega_P \) grows with \( P \), the objective value decreases accordingly but remains greater than \( \text{val}(\eqref{pb:NN_exact_rel}) \) since \( \Omega_P \subseteq \Omega \).
On the other hand, consider a solution of \eqref{pb:NN_exact_rel} in the form 
\(
\sum_{j=1}^N \tilde{\omega}_j \delta_{(\tilde{a}_j, \tilde{b}_j)}.
\)
By an argument similar to the proof of point (1), we can show that each \( (\tilde{a}_j, \tilde{b}_j) \) can be approximated almost surely by points in \( \Omega_P \) as \( P \to \infty \).
The convergence of TV$(P)$ (val\eqref{pb:random_feature}) follows.
\end{proof}

\section{Conclusion and perspectives}\label{sec:conclusion}
\textit{Conclusion}.
In this paper, we have presented and analyzed  three large-scale non-convex optimization problems related to shallow neural networks: \eqref{intro_pb:NN_exact}, \eqref{intro_pb:NN_epsilon}, and \eqref{intro_pb:NN_reg}. By applying the mean-field relaxation technique, we successfully convexified them in  an infinite-dimensional measure setting. Leveraging a representer theorem, we demonstrated the absence of relaxation gaps when the number of neurons exceeds the number of training data. Additionally, we explored a generalization bound, providing a qualitative analysis of the optimal hyperparameters, which offers significant guidance in their selection.

On the computational side, we  introduced a discretization of the relaxed problems, obtaining error estimates. We established that the discretized problems are equivalent to linear programming problems, which can be efficiently addressed using the simplex method,  benefiting the primal problems. However, when the feature dimensionality is high, this discretization suffers from the curse of dimensionality. To mitigate this, we proposed a sparsification algorithm, which,  combined with gradient descent-type algorithms for overparameterized shallow NNs, allows to obtain effective solutions for the primal problems. Our numerical experiments validate the efficacy of our analytical framework and computational algorithms.

\medskip
\textit{Perspectives}. The first perspective pertains to Remark \ref{rem:exact-regression2}. The case where $P < N$ is particularly relevant in real-world applications, as it represents scenarios where the dataset is large, but the number of parameters is limited due to model constraints. In this situation, we focus exclusively on the regression problem, since the existence of solutions to the representation problem is not guaranteed. Additionally, there is no equivalence between the primal problem and its relaxation. Nonetheless, by employing the techniques of proof from \cite{bonnans2023large} and \cite{ma2022barron}, we can derive the following results:
\[
    0 \leq \text{val}\eqref{intro_pb:NN_reg} - \text{val}\eqref{pb:NN_reg_rel} \leq \frac{C}{P}, \quad \text{ for } P < N,
\]
where $C$ is a constant independent of $P$ and $N$. Together with \eqref{eq:value_eq} for the case $P \geq N$, we obtain:
\[
     0 \leq \text{val}\eqref{intro_pb:NN_reg} - \text{val}\eqref{pb:NN_reg_rel} \leq \frac{C \, \mathbf{I}_{[1,N)}(P)}{P}, \quad \text{ for } P \in \mathbb{N}_{+},
\]
where $\mathbf{I}_{[1,N)}$ is the indicator function of the set $[1, N)$. The above inequality improves the result in \cite[Prop.\@ 1]{mei2018mean}, where the authors demonstrate only a gap of $\mathcal{O}(1/P)$, without addressing the fact that the gap vanishes when $P$ approaches  $N$. 
For the case $P<N$, numerical algorithms face significant challenges, as there is no equivalence between the primal and relaxed problems.

\medskip
The second perspective refers to the asymptotic behavior of the problems studied in this article as \( N \) approaches infinity. This extension is closely connected to the exact representation of a continuous function by an infinitely wide shallow NN \cite{barron1993universal,ma2022barron}. Specifically, given some \( f \in \mathcal{C}(X) \), where \( X \) is a compact set, the following question arises: does there exist a measure \( \mu \in \mathcal{M}(\R^{d+1}) \) such that
\begin{equation}\label{eq:representation_continuous}
    f(x) = \int_{\R^{d+1}} \sigma(\langle a, x\rangle + b)\, d\mu(a,b), \quad \text{for all } x \in X\, ?
\end{equation}
The existence of such a measure \( \mu \) depends on the regularity of \( f \), the form of the activation function \( \sigma \), and the geometry of the compact set \( X \). To the best of our knowledge, the strongest result is given in \cite{klusowski2018approximation}, where \( X \) is a hypercube, \( \sigma \) is the ReLU function, and the Fourier transform of an extension of \( f \) lies in \( L^1(\mathbb{R}^d, \|\xi\|^2 \, d\xi) \). A sufficient condition for \( f \) to satisfy this requirement is that \( f \in H^{k}(X) \) for some \( k > d/2 + 2 \), which is a very strong regularity assumption (stronger than \( \mathcal{C}^2(X) \)). Whether \eqref{eq:representation_continuous} holds for more general activation functions and for functions \( f \) with less regularity remains an open problem.

We may be able to make progress on this open problem by building on the work presented in this article. To explore this, suppose that the set \( \{ x_i \}_{i \geq 1} \) is dense in \( X \) and that \( y_i \) represents the observation of \( f(x_i) \). Thanks to Theorem \ref{thm:NN_exists_0}, there exists a measure \( \mu_N \in \mathcal{M}_N(\Omega) \), which is a solution to \eqref{pb:NN_exact_rel}, such that the equality in \eqref{eq:representation_continuous} holds for \( x = x_1, \ldots, x_N \).
We conjecture that the value of \eqref{pb:NN_exact_rel}, given by \( \|\mu_N\|_{\text{TV}} \), is controlled by $\|f\|_{L^{\infty}(X)}$ and the Lipschitz constant of \( f \), and is independent of $N$. If this holds, we demonstrate that the sequence \( \{\|\mu_N\|_{\text{TV}}\}_{N \geq 1} \) remains bounded when \( f \) is Lipschitz. Consequently, the sequence \( \{\mu_N\}_{N \geq 1} \) has a weak-* limit, which serves as a solution to \eqref{eq:representation_continuous}. To prove the previous conjecture, we need to conduct a more detailed dual analysis of the relaxed problem \eqref{pb:NN_exact_rel}, as outlined in Appendix \ref{app:dual}.


\medskip
The third perspective is to extend the relaxation method and analysis  in this article to more complex neural network architectures, such as deep multilayer perceptrons, ResNets, and Neural ODEs. The mean-field relaxations of deep NNs pose significant challenges, as discussed in \cite{fernandez2022continuous,nguyen2023rigorous}.
We are particularly interested in investigating the following two problems: 
\begin{itemize}
    \item Representation property of multilayer NNs. A multilayer NNs can be expressed as a composition of shallow networks:
\begin{equation*}
    X(k) = g\big(X(k-1), f_{\text{shallow}} (X(k-1), \Theta(k-1))\big), \quad \text{for } k=1, \ldots, K.
\end{equation*}
If each layer has the same width \(P\), then the \(N\)-sample representation property has been studied in several specific architectures (see \cite{alvarez2024interplay} for the residual structure and \cite{hernandez2024deep} for the case \(g= Id\) and \(P=2\)). In these studies, a critical condition required is:
\begin{equation*}
    KP \geq \alpha  N, \quad \text{with } \alpha \geq 1.
\end{equation*}
Theorem \ref{thm:NN_exists} also falls within this framework, as the total number of neurons must exceed \(N\). 
However, numerous simulations in machine learning suggest that deeper networks exhibit greater expressive power \cite{eldan2016power}. This raises the intriguing possibility of identifying a new criterion such as
\begin{equation*}
    KP \geq \alpha(K)N,
\end{equation*}
where \(\alpha(K) < 1\) for some specific values of \(K\), which would provide significant insights into designing NN architectures.
\item Mean-field relaxation of semi-autonomous Neural ODE. A direct mean-field relaxation of the classical Neural ODE \cite{ruiz2023neural} is challenging, as it requires a relaxation at each time \( t \), resulting in a stochastic process \( (\mu_t)_{t \geq 0} \). In \cite{li2024universal}, the authors introduce the semi-autonomous Neural ODE (SA-NODE) and demonstrate its UAP for dynamic systems. The SA-NODE has a natural formulation for mean-field relaxation as follows:
\begin{equation}\label{eq:SA-NODE}
    \dot{X} = \int_{\mathbb{R}^3} \sigma( a^1 X + a^2 t + b ) \, d\mu(a^1, a^2, b), \quad \text{for } t \in [0, T],
\end{equation}
where we consider only the state $X$ in \( \mathbb{R} \) for simplicity of notation. The mean-field technique provides the advantage that \(\mu\) functions as an affine control in \eqref{eq:SA-NODE}.
The representation property and the UAP, whether for the final targets or the entire trajectory, of the SA-NODE can be characterized as the simultaneous controllability of \eqref{eq:SA-NODE} and the optimal control problems governed by its associated continuity equation, similar to the approach in \cite{ruiz2023control}.
\end{itemize}


\appendix

\section{Dual analysis of representation problems}\label{app:dual}
Recall the definition of $\mathcal{U}(\epsilon)$ from Proposition \ref{prop:lambda} and $C_{X,X'} = d_{\text{KR}}(m_X,m_{X'})$. We are interested in the following optimization problem on the hyperparameter $\epsilon$:
\begin{equation}\label{pb:parameter}\tag{UB}
\inf_{\epsilon\geq 0} \, \mathcal{U}(\epsilon) = \epsilon  + C_{X,X'}\,\text{val}\eqref{intro_pb:NN_epsilon}.
    \end{equation}
To investigate the existence of solutions and the first-order optimality condition of \eqref{pb:parameter}, we need to perform a dual analysis for problems \eqref{intro_pb:NN_epsilon} and \eqref{intro_pb:NN_exact}. Recall the subgradient of the total variation norm \cite[Eq.\@ 9]{duval2015exact}: for any $\mu \in \mathcal{M}(\Omega)$,
\begin{equation*}\label{eq:subgradient}
    \partial \|\mu\|_{\text{TV}} = \left\{ v \in \mathcal{C}(\Omega)\, \Big| \, \int_{\Omega} v (\theta) d\mu (\theta) = \|\mu\|_{\text{TV}} \text{ and } \|v\|_{\mathcal{C}(\Omega)} \leq 1  \right\}.
\end{equation*}

\begin{lem}\label{lm:dual}
Let Assumption \ref{ass:sigma} hold true. Let $\epsilon>0$. The dual problem (in the sense of  Fenchel) of \eqref{pb:NN_epsilon_rel} is
\begin{equation}\label{pb:NN_epsilon_dual}\tag{DR$_\epsilon$}
    \sup_{\|\phi^{*} \, p\|_{\mathcal{C}(\Omega)} \leq 1}  \langle Y, p \rangle - \epsilon \| p\|_{\ell^{1}},
\end{equation}
where $\phi^{*} \colon \R^N \to \mathcal{C}(\Omega)$, $ p\mapsto \sum_{i=1}^N p_i\sigma(\langle a, x_i \rangle + b)$, is the adjoint operator of $\phi$.
Then, the strong duality holds: 
\begin{equation*}
    \textnormal{val}\eqref{pb:NN_epsilon_rel} = \textnormal{val}\eqref{pb:NN_epsilon_dual}.
\end{equation*}
The solution set $S\eqref{pb:NN_epsilon_dual}$ is non-empty, convex, and compact. 
Moreover,  $(\mu^{*}_{\epsilon}, p^{*}_{\epsilon})$ is a couple of solutions of \eqref{pb:NN_epsilon_rel} and  \eqref{pb:NN_epsilon_dual} if and only if
\begin{equation}\label{eq:primal_dual}
    \begin{cases}
        \phi^{*}  \, p^{*}_{\epsilon} \in \partial \|\mu^{*}_{\epsilon}\|_{\textnormal{TV}},\\
        \phi_i\, \mu^{*}_{\epsilon} = Y_i -\epsilon \cdot \textnormal{sign}((p^{*}_{\epsilon})_i ) ,\quad \textnormal{for all }i \text{ such that }(p^{*}_{\epsilon})_i \neq 0.
    \end{cases}
\end{equation}
\end{lem}
\begin{proof}
\textbf{Step 1} (Strong duality). 
    Let us rewrite problem \eqref{pb:NN_epsilon_rel} as
    \begin{equation}\label{pb:primal_epsilon}
        \inf_{\mu\in \mathcal{M}(\Omega)} J(\mu) + \chi_{B_{\infty}(Y,\epsilon)} (\phi \, \mu),
    \end{equation}
    where $J\colon \mathcal{M}(\Omega) \to \R$, $\mu\mapsto \|\mu\|_{\text{TV}}$, $\chi_{B_{\infty}(Y,\epsilon)}$ is the indicator function of $B_{\infty}(Y,\epsilon)$, and $B_{\infty}(Y,\epsilon) = \{z\in \R^N \mid \|z-Y\|_{\ell^{\infty}}\leq \epsilon\} $. Let us consider $\mathcal{M}(\Omega)$ equipped with the weak-$*$ topology, then its dual space is $C(\Omega)$ equipped with its weak topology \cite[Prop.\@ 3.14]{brezis2011functional}. In this sense, we calculate the Fenchel conjugate of $J$:
    \begin{equation*}
        J^{*}(v) = \begin{cases}
            0,  \quad & \text{if } \|v\|_{\mathcal{C}(\Omega)} \leq 1 ,\\
            \infty, & \text{otherwise}.
        \end{cases}
    \end{equation*}
    Following \cite{rockafellar1967duality}, the dual problem of \eqref{pb:primal_epsilon} is
    \begin{equation}\label{pb:dual_epsilon}
    \begin{split}
        \sup_{\|\phi^{*} \, p\|_{\mathcal{C}(\Omega)} \leq 1} \inf_{z\in B_{\infty}(Y,\epsilon)} \langle z , p\rangle & = \sup_{\|\phi^{*} \, p\|_{\mathcal{C}(\Omega)} \leq 1} \left(\langle Y, p\rangle +
        \inf_{z\in B_{\infty}(0,\epsilon)} \langle z , p\rangle \right)\\
       & = \sup_{\|\phi^{*} \, p\|_{\mathcal{C}(\Omega)} \leq 1} \langle Y, p\rangle - \epsilon \|p\|_{\ell^1}.
    \end{split}
    \end{equation}
    By Theorem \ref{thm:NN_exists}, there exists $\mu_0$ such that $ \phi \, \mu_0 = Y$. Furthermore, $\|\mu_0\|_{\text{TV}}$ is finite and $ \chi_{B_{\infty}(Y,\epsilon)} $ is continuous at $  \phi \, \mu_0$. It follows the qualification condition for problem \eqref{pb:primal_epsilon}. Combining with the convexity and l.s.c.\@ property of problem \eqref{pb:primal_epsilon}, we obtain the strong duality by the Fenchel-Rockafellar Theorem \cite{rockafellar1967duality}. This implies $\text{val}\eqref{pb:primal_epsilon} = \text{val}\eqref{pb:dual_epsilon}$ and the existence of the dual solution. Therefore, $S\eqref{pb:NN_epsilon_dual}$ is non-empty. The convexity and compactness of $S\eqref{pb:NN_epsilon_dual}$ is straightforward by the concavity and the coercivity of problem \eqref{pb:NN_epsilon_dual}.

    \medskip
    \noindent\textbf{Step 2} (Sufficiency of \eqref{eq:primal_dual}). By the strong duality, $(\mu^{*}_{\epsilon}, p^{*}_{\epsilon}) \in \mathcal{M}(\Omega)\times \R^N$ is a couple of solutions of \eqref{pb:NN_epsilon_rel} and  \eqref{pb:NN_epsilon_dual} if and only if 
    \begin{equation}\label{eq:strong_duality}
       J(\mu_{\epsilon}^{*}) + \chi_{B_{\infty}(Y,\epsilon)} (\phi \,\mu_{\epsilon}^{*}) = - J^{*}( \phi^{*}\,p^{*}_{\epsilon}  ) - \chi^{*}_{B_{\infty}(Y,\epsilon)} (-p^{*}_{\epsilon}).
    \end{equation}
    
    Now, assume that $(\mu^{*}_{\epsilon}, p^{*}_{\epsilon})$ satisfies \eqref{eq:primal_dual}. By the first line of \eqref{eq:primal_dual} and the Fenchel's relation, we deduce that
    \begin{equation*}
         J(\mu_{\epsilon}^{*}) + J^{*}( \phi^{*}\,p^{*}_{\epsilon}  ) = \langle \phi^{*}\,p^{*}_{\epsilon}, \mu_{\epsilon}^{*}\rangle = \langle p^{*}_{\epsilon}, \phi\, \mu_{\epsilon}^{*}\rangle = \langle Y, p^{*}_{\epsilon} \rangle - \epsilon \|p^{*}_{\epsilon}\|_{\ell^{1}},
    \end{equation*}
    where the last equality is by the second line of \eqref{eq:primal_dual}. We also obtain that
    \begin{equation*}
        \chi_{B_{\infty}(Y,\epsilon)} (\phi \,\mu_{\epsilon}^{*}) = 0,\quad   - \chi^{*}_{B_{\infty}(Y,\epsilon)} (-p^{*}_{\epsilon})=  \inf_{z\in B_{\infty}(Y,\epsilon)} \langle z , p^{*}_{\epsilon}\rangle =  \langle Y, p^{*}_{\epsilon} \rangle - \epsilon \|p^{*}_{\epsilon}\|_{\ell^{1}}.
    \end{equation*}
    Therefore, \eqref{eq:strong_duality} holds true for $(\mu^{*}_{\epsilon}, p^{*}_{\epsilon})$.

 \medskip
    \noindent\textbf{Step 3} (Necessity of \eqref{eq:primal_dual}). 
    Assume that \eqref{eq:strong_duality} is true. This implies that
    \begin{equation*}
          J(\mu_{\epsilon}^{*}) + J^{*}( \phi^{*}\,p^{*}_{\epsilon}  ) = -\chi_{B_{\infty}(Y,\epsilon)} (\phi \,\mu_{\epsilon}^{*}) -\chi^{*}_{B_{\infty}(Y,\epsilon)} (-p^{*}_{\epsilon}) = \langle Y, p^{*}_{\epsilon} \rangle - \epsilon \|p^{*}_{\epsilon}\|_{\ell^{1}}.
    \end{equation*}
    By the definition of the Fenchel conjugate,
    \begin{equation*}
        J(\mu_{\epsilon}^{*}) + J^{*}( \phi^{*}\,p^{*}_{\epsilon}  ) \geq  \langle \phi^{*}\,p^{*}_{\epsilon}, \mu_{\epsilon}^{*}\rangle = \langle p^{*}_{\epsilon}, \phi\, \mu_{\epsilon}^{*}\rangle.
    \end{equation*}
Comparing the previous two inequalities, we obtain
    \begin{equation*}
         \epsilon \|p^{*}_{\epsilon}\|_{\ell^{1}} \leq \langle p^{*}_{\epsilon},  Y - \phi\, \mu_{\epsilon}^{*}\rangle.
    \end{equation*}
    Since $ \mu_{\epsilon}^{*}$ is a solution of \eqref{pb:NN_epsilon_rel}, we have $ \phi\, \mu_{\epsilon}^{*} \in B_{\infty}(Y,\epsilon)$. Combining this with the inequality above, we obtain the second line of \eqref{eq:primal_dual} along with the equality $ \epsilon \|p^{*}_{\epsilon}\|_{\ell^{1}} = \langle p^{*}_{\epsilon}, Y - \phi\, \mu_{\epsilon}^{*}\rangle$, which implies that
\[
    J(\mu_{\epsilon}^{*}) + J^{*}( \phi^{*}\,p^{*}_{\epsilon} ) = \langle \phi^{*}\,p^{*}_{\epsilon}, \mu_{\epsilon}^{*}\rangle.
\]
The first line of \eqref{eq:primal_dual} follows from Fenchel’s relation.
\end{proof}

For the case $\epsilon=0$, we have the following similar result to Lemma \ref{lm:dual} from \cite[Prop.\@ 13]{duval2015exact}. 
\begin{lem}\label{lm:dual_0}
    Let Assumption \ref{ass:sigma} hold true. The dual problem  of \eqref{pb:NN_exact_rel} is
    \begin{equation}\label{pb:NN_exact_dual}\tag{DR$_0$}
    \sup_{\|\phi^{*} \, p\|_{\mathcal{C}(\Omega)} \leq 1}  \langle Y, p \rangle.
\end{equation}
Then, the strong duality holds:
\begin{equation*}
\textnormal{val}\eqref{pb:NN_exact_rel} = \textnormal{val}\eqref{pb:NN_exact_dual}.
\end{equation*}
The solution set $S\eqref{pb:NN_exact_dual}$ is non-empty, convex, and compact.
\end{lem}
\begin{proof}
    The strong duality is from \cite[Prop.\@ 13]{duval2015exact}.
    It is easy to see that the feasible set of \eqref{pb:NN_exact_dual} is non-empty, convex, and closed. Let us prove that the feasible set is bounded. For any $\tilde{Y}\in \R^N $, by Theorem \ref{thm:NN_exists}, there exists $\mu\in \mathcal{M}(\Omega)$ such that $\phi\, \mu = \tilde{Y}$. By the arbitrariness of $\tilde{Y}$, we obtain that $\phi\colon \mathcal{M}(\Omega) \to \R^N$ is surjective. We deduce from the open mapping theorem that the set $\{p\in \R^N \,\mid\, \|\phi^{*} p\|_{\mathcal{C}(\Omega)} \leq 1\} $ is bounded (see \cite[Cor.\@ 1.30]{bonnans2019convex}), thus compact. Combining with the concavity and the continuity of the objective function, the conclusion follows.
\end{proof}

By Lemmas \ref{lm:dual} and \ref{lm:dual_0}, for any $\epsilon \geq 0$, the set
$ \mathcal{S}\eqref{pb:NN_epsilon_dual}$ is non-empty, compact, and convex. To simplify notations in the following results, we introduce the following variables, give any $\epsilon \geq 0$:
\begin{equation}\label{eq:c_epsilon}
    c_{\epsilon} = \min_{p} \left\{ \, \|p\|_{\ell^1}\, \mid\, p\in \mathcal{S}\eqref{pb:NN_epsilon_dual} \, \right\}, \quad C_{\epsilon} = \max_{p} \left\{ \, \|p\|_{\ell^1}\, \mid\, p\in \mathcal{S}\eqref{pb:NN_epsilon_dual} \, \right\}.
\end{equation}

\begin{lem}\label{lem:danskin}
    Let Assumption \ref{ass:sigma} hold true. Let $V\colon \R_{+} \to \R$, $\epsilon \mapsto \textnormal{val}\eqref{intro_pb:NN_epsilon}$.
    Then, the function $V$ is convex and l.s.c. Moreover,
    \begin{equation}\label{eq:subgradient_V}
        \partial V(\epsilon) = [-C_{\epsilon}, -c_{\epsilon}],\quad \text{for }\epsilon >0,
    \end{equation}
    where $\partial V$ is the subgradient of $V$, and $C_{\epsilon}$ and $c_{\epsilon}$ are defined in \eqref{eq:c_epsilon}. The right  derivative of $V$ at the point $0$ is \begin{equation}\label{eq:right_derivative}
        V'(0^{+}) = - c_0.
    \end{equation}
    where $c_0$ is defined in \eqref{eq:c_epsilon} with $\epsilon=0$.
\end{lem}
\begin{proof}
  Let $\mathcal{Z}\coloneqq \{p\in \R^N \,\mid\, \|\phi^{*} p\|_{\mathcal{C}(\Omega)} \leq 1\}$. We know that $\mathcal{Z}$ is non-empty, convex, and compact by the proof of Lemma \ref{lm:dual_0}.  Consider the function $\bar{V}\colon \R \to \R,\, \epsilon \mapsto \sup_{p\in \mathcal{Z}} \langle Y, p \rangle -\epsilon \|p\|_{\ell^1}$.
    By the strong duality in Lemmas \ref{lm:dual} and \ref{lm:dual_0}, we have that $V(\epsilon) = \bar{V}(\epsilon)$ for $\epsilon \geq 0$.
    The assumptions outlined in Danskin's Theorem \cite[Prop.\@ A.3.2]{bertsekas2009convex} hold true for $\bar{V}$.
     We then deduce that $\bar{V}$ is convex and l.s.c., and 
     \begin{equation*}
         \partial \bar{V}(\epsilon)= \conv\{-\|p\|_{\ell^1}\, \mid\, p\in S\eqref{pb:NN_epsilon_dual}\} = [-C_{\epsilon},-c_{\epsilon}], \quad \text{for } \epsilon >0,
     \end{equation*} 
     where the second equality comes from the fact that $S\eqref{pb:NN_epsilon_dual}$ is connected, which is a result of its convexity.
     The same conclusion holds true for $V$ by its equivalence to $\bar{V}$. From the formula for the directional derivative in Danskin's Theorem, we obtain \eqref{eq:right_derivative}.
\end{proof}

The following theorem gives the optimality condition for problem \eqref{pb:parameter}. Recall that $C_{X,X'} = d_{\text{KR}}(m_X,m_{X'})$ and the definition of $c_{\epsilon}$ and $C_{\epsilon}$ from \eqref{eq:c_epsilon}.

\begin{thm}\label{thm:epsilon}
    Let Assumption \ref{ass:sigma} hold true. Then problem \eqref{pb:parameter} is convex and has solutions. 
    Moreover, the following holds: \begin{enumerate}
        \item If $C_{X,X'} < c_0^{-1}$,
        then $ S\eqref{pb:parameter}=\{0\}$.
        \item If $C_{X,X'} \geq c_0^{-1}$, then $ \epsilon \in  S\eqref{pb:parameter}$ if and only if
        \begin{equation}\label{eq:optimality}
            C_{X,X'}^{-1}\in [c_{\epsilon}, C_{\epsilon}].
        \end{equation}
    \end{enumerate}
\end{thm}

\begin{proof}
It is easy to observe that \eqref{intro_pb:NN_epsilon} has the unique solution $0$ for $\epsilon \geq \|Y\|_{\ell^{\infty}}$, implying that the objective function $\mathcal{U}$  is coercive. By Lemma \ref{lem:danskin}, $\mathcal{U}$ is convex and l.s.c. The existence of solutions follows.

From \eqref{eq:right_derivative}, we obtain that $\mathcal{U}'(0^{+}) > 0$ when \(C_{X,X'} < c_0^{-1}\). This implies that $0$ is the unique minimizer by the convexity of $\mathcal{U}$. Thus, point (1) holds.
    
Let us prove statement (2). Assume that \( C_{X,X'} \geq c_0^{-1} \). It follows that \(\mathcal{U}'(0^{+}) \leq 0\).
Then, any point \(\epsilon > 0\) is a minimizer of \eqref{pb:parameter} if and only if \(0 \in \partial \mathcal{U}(\epsilon)\), which is equivalent to the inclusion relation in point (2) by \eqref{eq:subgradient_V}.
Additionally, \(0 \in S\eqref{pb:parameter}\) if and only if \(C_{X,X'}  = d_{\textnormal{KR}}(m_X, m_{X'}) = c_0^{-1}\). This is covered by point (2) by taking \(\epsilon = 0\).
\end{proof}

\end{document}